\ifdefined\useorstyle

\documentclass[opre,nonblindrev]{informs3}

\usepackage{packages-or}
\usepackage{editing-macros}
\usepackage{statistics-macros-or}
\usepackage{formatting}
\usepackage{wrapfig}
\usepackage[textsize=tiny]{todonotes}
\newcommand{\hntodo}[1]{\todo{Hong: #1}}

\DoubleSpacedXI 

\usepackage{endnotes}
\let\footnote=\endnote

\usepackage{natbib}
 \bibpunct[, ]{(}{)}{,}{a}{}{,}%
 \def\bibsep{\smallskipamount}%
\TheoremsNumberedThrough     
\ECRepeatTheorems

\EquationsNumberedThrough    

\usepackage{hyperref}
\usepackage{./statistics-macros-or}

\usepackage{pgfplotstable}
\usepackage{graphicx}
\usepackage{subcaption}
\usepackage{float}
 
\usepackage{algorithm}
\usepackage{algpseudocode}
\usepackage{tabularx}

\usepackage{overpic}
\usepackage{tikz}
\usepackage{rotating}
\usepackage{psfrag}
\usepackage{enumitem}
  
\begin{document}



\RUNAUTHOR{Namkoong, Ma, and Glynn}
\RUNTITLE{Minimax Optimal Estimation of Stability Under Distribution Shift}

\TITLE{Minimax Optimal Estimation of Stability Under Distribution Shift}

\ARTICLEAUTHORS{%

\AUTHOR{Hongseok Namkoong}
\AFF{Decision, Risk, and Operations Division, Columbia Business School, New York, NY 10027, \EMAIL{namkoong@gsb.columbia.edu}} 

\AUTHOR{Yuanzhe Ma}
\AFF{Department of Industrial Engineering and Operations Research, Columbia University, New York, NY 10027, \EMAIL{ym2865@columbia.edu}} 

\AUTHOR{Peter Glynn}
\AFF{Department of Management Science and Engineering, Stanford University, Stanford, CA 94305, \EMAIL{glynn@stanford.edu}} 

} 

\ABSTRACT{The performance of decision policies and prediction models often deteriorates
when applied to environments different from the ones seen during training. To
ensure reliable operation, we analyze the \emph{stability} of a
system under distribution shift, which is defined as the smallest change in
the underlying environment that causes the system's performance to deteriorate
beyond a permissible threshold. In contrast to standard tail risk measures and
distributionally robust losses that require the specification of a plausible
magnitude of distribution shift, the stability measure is defined in terms of
a more intuitive quantity: the level of acceptable performance degradation.
We develop a minimax optimal estimator of stability and analyze its
convergence rate, which exhibits a fundamental phase shift behavior. Our
characterization of the minimax convergence rate shows that evaluating
stability against large performance degradation incurs a statistical cost.
Empirically, we demonstrate the practical utility of our stability framework
by using it to compare system designs on problems where robustness to
distribution shift is critical.

}



\maketitle

%


\else

\documentclass[11pt]{article}
\usepackage[numbers]{natbib}
\usepackage{packages}
\usepackage{editing-macros}
\usepackage{formatting}
\usepackage{./statistics-macros}
\usepackage{wrapfig}
\usepackage[textsize=tiny]{todonotes}
\newcommand{\hntodo}[1]{\todo{Hong: #1}}


\begin{document}

\abovedisplayskip=8pt plus0pt minus3pt
\belowdisplayskip=8pt plus0pt minus3pt


\begin{center}
  {\LARGE Minimax Optimal Estimation of Stability Under Distribution Shift} \\
  \vspace{.5cm}
  {\Large Hongseok Namkoong$^{a}$  ~~~ Yuanzhe Ma$^{b}$\footnote{HN and YM equally contributed to this work.} ~~~
    Peter W. Glynn$^{c}$} \\
  \vspace{.2cm}
  Decision, Risk, and Operations Division$^a$,
  Department of Industrial Engineering and Operations Research$^b$,
  and Management Science and Engineering$^c$ \\
  \vspace{.2cm}
    {\large Columbia University$^{a, b}$ \qquad \qquad Stanford University$^{c}$} \\
    \vspace{.2cm}
  \texttt{namkoong@gsb.columbia.edu, ym2865@columbia.edu, glynn@stanford.edu}
\end{center}


\begin{abstract}%
  The performance of decision policies and prediction models often deteriorates
when applied to environments different from the ones seen during training. To
ensure reliable operation, we analyze the \emph{stability} of a
system under distribution shift, which is defined as the smallest change in
the underlying environment that causes the system's performance to deteriorate
beyond a permissible threshold. In contrast to standard tail risk measures and
distributionally robust losses that require the specification of a plausible
magnitude of distribution shift, the stability measure is defined in terms of
a more intuitive quantity: the level of acceptable performance degradation.
We develop a minimax optimal estimator of stability and analyze its
convergence rate, which exhibits a fundamental phase shift behavior. Our
characterization of the minimax convergence rate shows that evaluating
stability against large performance degradation incurs a statistical cost.
Empirically, we demonstrate the practical utility of our stability framework
by using it to compare system designs on problems where robustness to
distribution shift is critical.


\end{abstract}

\fi

\section{Introduction}
\label{section:introduction}

Distribution shift is a universal challenge in data-driven decision-making.
Distribution shift can occur because data is often collected from a limited
number of sources~\cite{Manski13}, marginalized groups are
underrepresented~\cite{ChenLaDaPaKe14}, and decision policies operate in
non-stationary environments~\cite{Koopman72}. On distributions different from
that of the training data, the performance of decision policies and prediction
models have been observed to deteriorate across a wide range of applications.
Economic and healthcare policies must remain valid over space and time, but
their external validity is often called into question~\cite{Accord10,
  Sprint15, BasuSuHa17, RosenzweigUd20, DehejiaPoSa21}.  Despite tremendous
recent progress in machine learning systems, prediction models perform poorly
under distribution shifts in healthcare~\cite{LeekEtAl10, BandiEtAl18,
  ZechBaLiCoTiOe18, ChenPiRoJoFeGh20, WongEtAl21}, loan
approval~\cite{Hand06}, wildlife conservation~\cite{BeeryCoGj20}, and
education~\cite{AmorimCaVe18}. For example, EPIC's proprietary sepsis risk
assessment model---currently deployed at hundreds of hospitals across the
nation---has recently been found to perform ``substantially worse'' than the
vendor's claims~\cite{WongEtAl21}.

\begin{wrapfigure}{r}{.45\columnwidth}
  \vspace{-0.4cm} 
  \centering
\includegraphics[scale=.17]{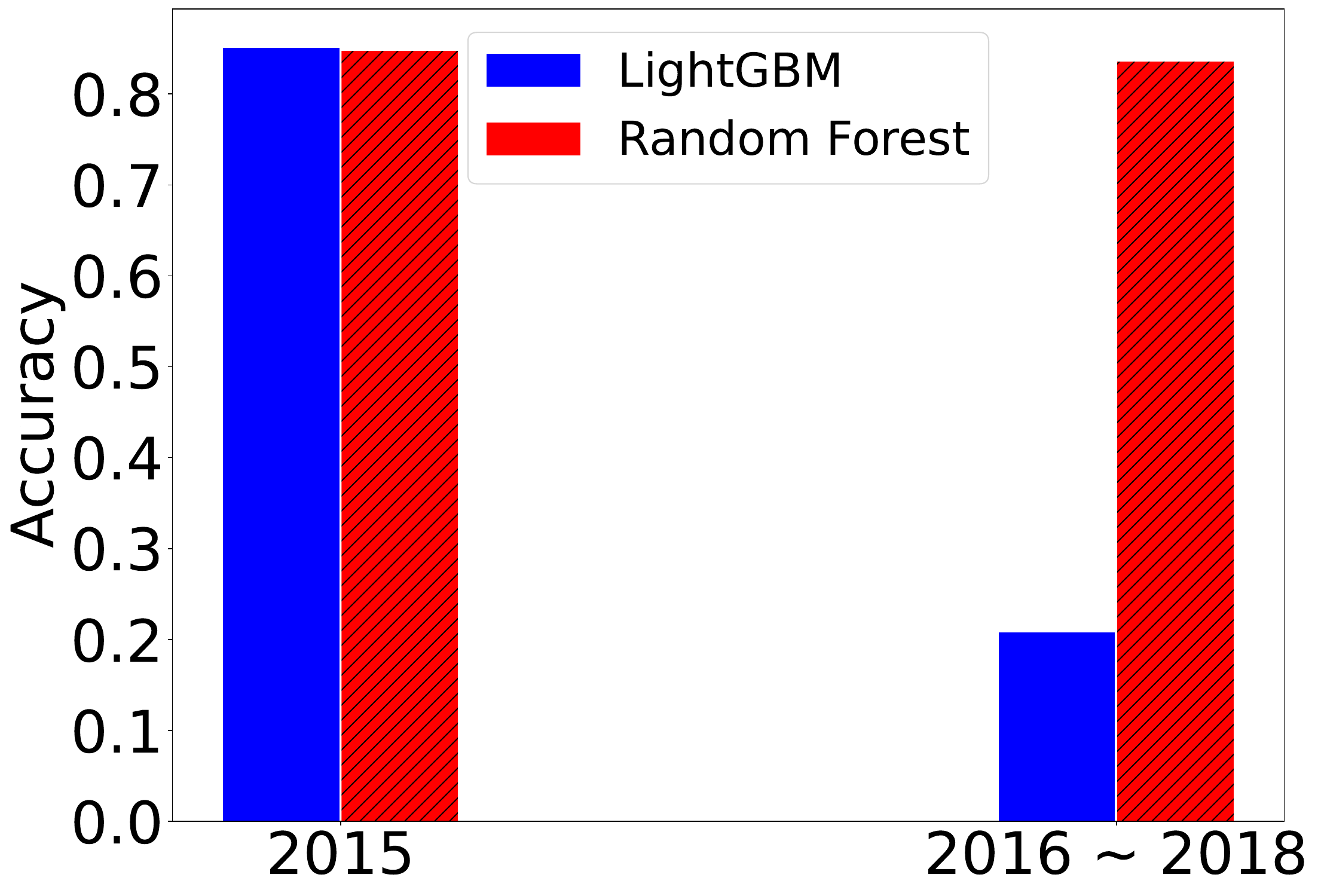}
  \vspace{-0.2cm}
  \caption{Models with initially near-identical
    average-case performance exhibit different accuracy over time.}
\label{fig:NHIS-bar-plot}
  \vspace{-2cm}
\end{wrapfigure}
Standard methods that evaluate average-case performance cannot ensure stable
performance over unanticipated distribution shifts. As an example, consider a
binary classification problem where the goal is to predict whether an
individual is utilizing healthcare resources based on a set of rich
individual-level covariates. We discuss the problem in detail in
Section~\ref{section:nhis} and discuss how such classifiers can inform
policy-making for low-income individuals.  By simulating an analyst \emph{in
  2015}, we use popular machine learning packages---gradient-boosted trees
from $\mathsf{LightGBM}$ and random forests from $\mathsf{scikitlearn}$---to
train two classification models. While the two models achieve near-identical
average-case performance on out-of-sample data from 2015
(Figure~\ref{fig:NHIS-bar-plot}), their performance diverges dramatically in
subsequent years.~\citet{DAmourEtAl20} recently observed a similar phenomenon
across a wide range of prediction scenarios employed at Google.

To ensure the effectiveness of decision policies and prediction models, they
must be rigorously tested before deployment. Toward this goal, we analyze the
\emph{stability} of a system under distribution shifts. For a given random
cost function denoted by $\rv$, we consider an analyst who has access to
i.i.d. scenarios with associated costs $\rv_1, \ldots, \rv_n \simiid P$.  For
a chosen level of permissible performance degradation $y$, we define the
\emph{stability} of the system as the smallest distribution shift in the
underlying environment that deteriorates performance above the threshold $y$.
In stochastic control or reinforcement learning, the cost function $\rv$ is
often defined as the cumulative cost over a sample path; see
Section~\ref{section:queueing} for a queueing example.  In supervised learning
contexts, the cost function $\rv$ may be defined as the expected prediction
error given a set of variables whose distributions may change over space and
time; see Section~\ref{section:nhis} for an illustration on the healthcare
example depicted in Figure~\ref{fig:NHIS-bar-plot}.

We use the Kullback-Leibler (KL) divergence to quantify the notion of
closeness between probability distributions. For a $\sigma$-finite measure
$\mu$ such that $Q, P \ll \mu$, the KL divergence is
\begin{equation*}
  \dkl{Q}{P} \defeq \E_Q\left[ \log \frac{dQ/d\mu}{dP/d\mu}\right].
\end{equation*}
If $Q$ is not absolutely continuous with respect to $P$, we have
$\dkl{Q}{P} = \infty$.  Formally, the stability $I_y(P)$ at threshold
$y \ge \E_{P}[\rv]$ is defined as
\begin{equation}
  \label{eqn:primal}
  I_y(P) \defeq  \inf_{Q} \left\{\dkl{Q}{P}: \E_Q [\rv] \geq y\right\}.
\end{equation}
For simplicity, we omit the dependence on the threshold $y$ and often write
$I(P)$ or $I$.  As we detail in Section~\ref{section:approach}, a
\highlight{well-known dual reformulation of the minimization
  problem~\eqref{eqn:primal}~\cite[Theorem 5.2]{DonskerVa76}} shows the
optimal solution $Q\opt$---the smallest adversarial distribution shift---is
given by an exponential tilt
$\frac{dQ\opt}{dP} \propto \exp(\lambda\opt \rv)$, where $\lambda\opt$ is the
optimal dual variable that grows with $y$. Thus, the stability $I_y(P)$
measures the tail performance of the data-generating system $P$ around the
threshold $y$, and \highlight{was proposed by~\citet{LemaitreSeArBoGaIo15, BachocGaHaLoRi23}.}

Comparisons of tail risk across multiple system designs are meaningful insofar
as the risk criterion is interpretable. There is a significant body of work on
tail risk measures, e.g., coherent risk measures~\cite{ArtznerDeEbHe99,
  RockafellarUr00, Delbaen02, Rockafellar07} and other axiomatic definitions
of tail risk~\cite{RuszczynskiSh06, RockafellarUrZa06}, and distributionally
robust losses~\cite{BlanchetMu19, KuhnEsNgSh19, RahimianMe19, LiNaXi21,
  VanParysEsKu21, GaoKl22} that consider worst-case
performance over a pre-defined set of distributions. The two notions are closely related through
duality~\cite{Rockafellar07, ShapiroDeRu09}.  Since the evaluation of
tail risk requires the modeler to postulate a magnitude of anticipated
distribution shift, a plausible choice remains one of the biggest challenges
in operationalizing notions of tail risk.

The stability measure~\eqref{eqn:primal} provides a complementary approach by
comparing the stability of the system for a chosen threshold $y$ in the
\emph{cost scale}. This approach is more intuitive and interpretable than
setting a specific magnitude of distribution shift, such as
$\dkl{Q}{P} \le \rho$. It is often easier for analysts and engineers to decide
on a tolerable amount of monetary loss or prediction error based on past data
and domain knowledge when using the cost scale.  \highlight{While it is evident that no
class of criteria is uniformly dominant, we highlight interpretability as a
key advantage of the stability measure.}
A similar high-level approach is common in causal inference research where
through sensitivity analysis, researchers analyze the amount of unobserved
confounding required to challenge the conclusion of an observational
study~\cite{Rosenbaum10, Rosenbaum11, YadlowskyNaBaDuTi22}. For example, the
finding that smoking causes lung cancer was considered credible because it
would require an unrealistic amount of unobserved confounding (e.g.,
unrecorded hormone) to call the findings of the study into
question~\cite{CornfieldHaHaLiShWy59}.  Conceptually, our approach draws a
parallel by analyzing how much distribution shift is required to deteriorate
system performance to an intolerable degree.

\paragraph{Main contributions} To build intuition, we begin with a dual
reformulation of the infimum problem~\eqref{eqn:primal} in
Section~\ref{section:approach}. We make concrete how the stability
measure~\eqref{eqn:primal} is related to the tail probability of the cost
$\rv$.  Because committing to a plausible magnitude of distribution shift can
be challenging, we study statistical estimation of stability $I_y(P)$ using
i.i.d. observations $\rv_i \simiid P$ from the data-generating distribution
$P$. Our main theoretical results characterize the finite sample minimax rate
of convergence, which quantifies the fundamental rate at which $I_y(P)$ can be
estimated based on $n$ observations. Our finite sample results complement
prior asymptotic convergence results for $\what{I_n}$ that assume a fixed
data-generating distribution $P$~\cite{FeuervergerMu77, CsorgoTe90, DuffyMe05,
  DuffyWi15, RohwerAnTo15}.

\highlight{
Finite sample minimax rates are often considered overly conservative as it
characterize the worst-case performance over a class of data-generating
distributions
\begin{equation*}
  \mathfrak{M}_n \defeq \inf_{\what{I}_n} \sup_{P \in \mc{P}}
  \E_P  \left| \what{I}_n - I_y(P) \right|,
\end{equation*}
where the outer infimum is over all estimators---measurable functions of the
data $\rv_1, \ldots, \rv_n$---and the inner supremum is over data-generating
distributions. (This set is not to be confused with possible distribution
shifts we consider in our estimand~\eqref{eqn:primal}.)  To overcome this
limitation and provide sharp insights, our analysis considers an highly
restricted set of data-generating distributions. This allows us to
characterize how the tail behavior of $\rv$ impacts estimability of stability.

We crystallize our mathematical insight by studying random variables $\rv$
with varying tail behavior. To maximize clarity in our theory, we characterize
the minimax rate $\mathfrak{M}_n$ for random variables with Gamma-like tails
with $\P(\rv \ge t) \approx e^{-\absc t}$ up to polynomial terms, where
$\absc = \inf \{\lambda: \E[e^{\lambda \rv}] = \infty\}$ is the abscissa of
convergence.}  Our main results establish
\begin{equation}
  \label{eqn:minimax}
 \mathfrak{M}_n = \inf_{\what{I}_n} \sup_{P \in \mc{P}}
  \E_P  \left| \what{I}_n - I_y(P) \right|
  \asymp  n^{- \paran{\half \wedge \frac{\mcpgamma}{\absc y}}},
\end{equation}
where $\asymp$ notation hides constants that depend on $\absc$ and $y$, as
well as polylogarithmic factors in $n$. In the formal results we present in
Sections~\ref{section:convergence} and~\ref{section:hardness}, we provide
stronger results on the probability of error rather than expected estimation
error.  Our focus on Gamma-like tails allows a sharp characterization of how
statistical difficulty depends on different tail behaviors, providing almost
an instance-specific understanding.

\begin{wrapfigure}{r}{.44\columnwidth}
  \vspace{-0.7cm} 
  \centering
\includegraphics[scale=.23]{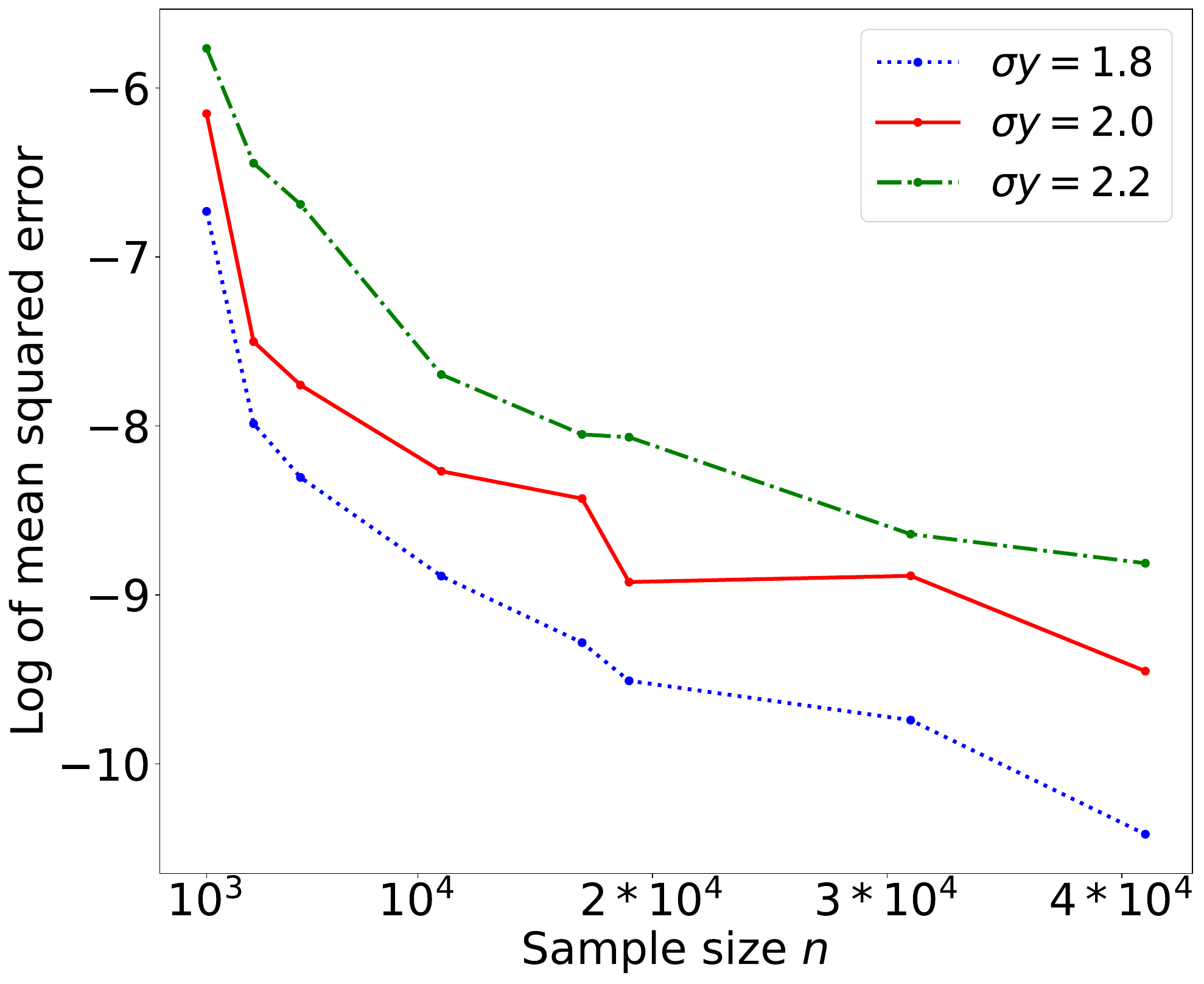}
  \vspace{-0.5cm}
  \caption{Mean squared error of $\what{I}_n$ (over 40 runs) with
  $P = \mathsf{Exp}(\absc), \absc \in \{0.9,1,1.1\}$ and a fixed threshold
  $y=2$.}
\label{fig:simulation-convergence-plot}
\end{wrapfigure}
To upper bound the minimax rate $\mathfrak{M}_n$, we focus on a simple plug-in
estimator $\what{I}_n$ of the dual representation of $I_y(P)$ (see
definition~\eqref{eqn:dual-plug-in} to come). In
Section~\ref{section:convergence}, we prove that $\what{I}_n$ converges at a
rate that depends on both the optimal dual variable and the abscissa of
convergence $\absc = \inf \{\lambda: \E[e^{\lambda \rv}] = \infty\}$.  Our
results show a phase transition in the convergence rate~\eqref{eqn:minimax},
which simultaneously depends on the threshold $y$ and the abscissa of
convergence
\begin{equation}
  \label{eqn:abscissa}
  \absc = \inf \{\lambda: \E[e^{\lambda \rv}] = \infty\}.
\end{equation}
For large values of $y$ relative to $\absc$, estimation becomes difficult as
the stability measure increasingly focuses on the tails of $\rv \sim P$. We
illustrate the degradation in convergence in
Figure~\ref{fig:simulation-convergence-plot}. 

\highlight{Our fundamental hardness results in Section~\ref{section:hardness} show that
the pessimistic convergence rate is unavoidable for large values of $y$. This
is surprising because we focus on a small set $\mc{P}$ of data-generating
distributions in the minimax risk~\eqref{eqn:minimax} that exhibit Gamma-like
tail behavior with a \emph{fixed abscissa of convergence $\absc$} (see
Definition~\ref{def:pclass}).
When we consider distributions with different
abscissas of convergence, it can be shown that the convergence rate further
degrades to $\Omega_p(1/\log n)$. To our knowledge, our bounds are the first
theoretical characterization of the fundamental difficulty in estimating
$I_y(P)$, showing that \emph{no statistical estimator} can avoid the
degradation depicted in Figure~\ref{fig:simulation-convergence-plot}.
The rate~\eqref{eqn:minimax} also makes clear how our focus on
  Gamma-like tails sharply characterizes the relationship between statistical
  efficiency and the tail behavior of $\rv$.  In particular, for light-tailed
  random variables ($\absc = \infty$), we do not observe a shape transition and
  recover simply the usual parametric rate. (See
  Corollaries~\ref{cor:gaussian-ub} and~\ref{cor:gaussian-lb} for formal
  statements.) }

In Section~\ref{section:experiment}, we empirically demonstrate how stability
against a cost threshold $y$ can provide practical assessments of system
performance under potential distribution shift. On two problems where it is
critical to maintain performance under distribution shift, we illustrate how
the stability measure~\eqref{eqn:primal} differentiates between reliable
vs. fragile designs.  We then evaluate the performance of different system
designs under both real and simulated distribution shifts, and show that the
designs identified as the most stable according to our
measure~\eqref{eqn:primal} are able to maintain good performance across a
range of scenarios.

\paragraph{Related work} 
Stability and robustness are central topics of interest for multiple
scientific communities. We give a necessarily abridged review in the context
of our work. In stochastic simulation, sensitivity analysis methods evaluate
how performance measures (``outputs'') are affected by changes in the input
variables~\cite{SongNePe14,Lam16b}.  \emph{Local} sensitivity analysis aims to
quantify changes in the output induced by small perturbations in the input
parameters~\cite{JiangNeHo19}.  See~\citet[Ch. 7]{Glasserman04}
and~\citet[Ch. 7]{AsmussenGl07} for a primer on derivative estimation in
simulation settings.  In finance, estimation of such local sensitivity is used
to hedge against changes in the environment (e.g., price
uncertainty). \emph{Global} sensitivity analysis studies how input variability
propagates to output variability. Many authors have used conditional variances
to attribute the part of output variance due to variability in different
subsets of the
input~\cite{HommaSa96,SaltelliRaAnCaCaGaSaTa08,Owen14,SongNeSt16,
  LemaitreSeArBoGaIo15}.  \highlight{Relatedly  in robust
  satisficing~\cite{LongSiZh23},
  one allows the objective
  $\E_{Z \sim P}f(x,Z)$ to deviate from a given target when the underlying
  environment $P$ changes, which is different to our setting.  The fragility
  is defined as the worst expected degradation of the objective normalized by
  the change in $P$.  In contrast, our notion of stability directly measures
  the least amount of perturbation required to make the loss intolerable.}

Another related stream of work analyzes the local sensitivity of an output
against \emph{distributional uncertainty and model
  misspecification}~\cite{GlassermanXu13,Lam16a,Lam17, LamQi17,
  GlassermanYa18,GoevaLaQiZh19}. These works typically analyze the worst-case
performance under infinitesimally small distribution shifts. By analyzing
distributional variations induced by statistical uncertainty, a number of
authors have recently proposed confidence intervals based on worst-case local
sensitivity~\cite{LamZh17, GhoshLa19, BlanchetKaMu19, DuchiGlNa21}.  In the
context of the stability measure~\eqref{eqn:primal}, analyzing the impact of
small distribution shifts is related to the limit $y \downarrow \E_P[\rv]$.
In contrast, coherent risk measures and distributionally robust objectives
measure the worst-case performance under a fixed ambiguity set of distribution
shifts~\cite{Rockafellar07, ShapiroDeRu09, JeongNa22, LiNaXi21}.  As we are
similarly concerned with stability against changes in the underlying
environment, our main theoretical analysis studies a phase shift behavior for
large values of $y$.

Reproducibility and stability is a critical concern in
statistics~\cite{LiBrHuBi11,Yu13,Stodden15,MurdochSiKuAbYu19} and more
generally, scientific research~\cite{Ioannidis05}. Classical robust statistics
approaches~\cite{Huber81,MaronnaMaYo06,HuberRo09} develop techniques (e.g.,
leverage score, breakdown point, influence function) to understand stability
against sample outliers \highlight{or sample removal~\citep{BroderickGiMe23, FreundHo23, MoitraRo23}}---as opposed to distribution shifts as in our work. In a
recent work,~\citet{GuptaRo21} introduced the $s$-value, a notion of stability
similar to ours in the context of statistical inference. They propose a new
paradigm of statistical inference and hypothesis testing centered around
stability under distribution shifts. 
On the other hand, this work focuses on operational
scenarios involving automated decision policies and prediction models and
theoretically characterizes the finite sample minimax
rate~\eqref{eqn:minimax}.

Although we take the data-generating distribution $P$ as given, our approach
is most effective when $P$ contains diverse observations.  Our framework is
thus complementary to design-based perspectives in statistics and machine
learning that aim to maximize heterogeneity in data. The burgeoning direction
includes efforts to benchmark machine learning models across distribution
shifts~\cite{SaenkoKuFrDa10, TaoriDaShCaReSc19, KohSaEtAl20} and multi-site
data collection in experimental design~\cite{CrucesGa07, BanerjeeKaZi15,
  GertlerShAlCaMaPa15, DehejiaPoSa21, TiptonPe17, TiptonRo18}.


\section{Stability}
\label{section:approach}

The stability measure~\eqref{eqn:primal} assesses the smallest distribution
shift that causes a meaningful degradation in system performance, measured by
the threshold value $y$. In this work, we use the Kullback-Leibler divergence
to measure the magnitude of distribution shift, a notion of distance that
enjoys invariance to affine transformations. The KL divergence is widely used
across multiple fields, including robust control~\cite{HansenSa01}, maximum
likelihood in exponential families~\cite{Brown86}, Bayesian statistical
analysis~\cite{Berk66}, and hypothesis testing~\cite{Csiszar84}.  However, its
main limitation is that $\dkl{Q}{P} = \infty$ when $Q$ is not absolutely
continuous with respect to $P$; the stability measure~\eqref{eqn:primal} is
thus only defined over distribution shifts $Q$ whose support is contained in
that of the data-generating distribution $P$.  In practice, this restriction
does not greatly limit the applicability of the stability measure, as we
primarily focus on analyzing the stability of real-valued performance
measures.  In many cases, it is natural to model the data-generating
distribution $P$ as having a long, continuous support in $\R$ as we illustrate
in Section~\ref{section:experiment}.

\paragraph{Dual reformulation} We begin our discussion by understanding the
stability measure~\eqref{eqn:primal} through its dual reformulation. Rewrite
the primal problem in terms of likelihood ratios $L \defeq \frac{dQ}{dP}$
\begin{align}
  \label{eqn:likelihood}
  I_y(P) \defeq \inf_{Q} \left\{ \dkl{Q}{P}:  \E_Q [\rv] \geq y\right\}
  = \inf_{L \ge 0}  \left\{ \E_P[L \log L]: \E_P [L \rv] \geq y, \E_P[L] = 1 \right\},
\end{align}
where the final infimum is taken over measurable functions.  Taking the dual
over this final problem, we arrive at the following classical reformulation.
\begin{lemma}[{\citet[Theorem 5.2]{DonskerVa76}}]
  \label{lemma:duality-result}
  Let $\rv \sim P$ be a real-valued random variable.  For every
  $y$ with $\E_P[\rv] \le y < \mathrm{ess~sup}~\rv$,
  \begin{equation}
    \label{eqn:dual}
    I_y(P) 
    =  \sup_{\lambda \in \R} \set{\lambda  y - \log \E_P[e^{\lambda \rv}]}. 
  \end{equation}
  Furthermore, if the supremum on the right-hand side is attained at
  $\lambda\opt$, then  $\lambda\opt \ge 0$ and the distribution $Q\opt$ that achieves the infimum on the left-hand side satisfies
  $dQ\opt(x) = \frac{e^{\lambda\opt x} dP(x)}{\E_P[e^{\lambda\opt \rv}]},
  \forall x \in \R$.
\end{lemma}
\noindent Although this result is well-known, we give a proof based on modern
duality theory in Section~\ref{section:proof-of-duality} for completeness.

Both the primal and dual formulations provide insights into how $I_y(P)$ grows
with $y$. In the primal~\eqref{eqn:primal}, the constraint set grows smaller
as the threshold $y$ increases, whereas $y$ appears directly in the dual
objective~\eqref{eqn:dual} via the Lagrangian. From first-order optimality
conditions, the dual optimum $\lambda\opt$ satisfies
\begin{equation*}
  \frac{\E_P[\rv e^{\lambda\opt \rv}]}{\E_P[e^{\lambda\opt \rv}]} = y,
\end{equation*}
where the left-hand side is the derivative of the cumulant generating function
$\lambda \mapsto \log \E_P[e^{\lambda \rv}]$ evaluated at $\lambda\opt$.
Lemma~\ref{lemma:duality-result} states that the corresponding primal optimum is
given by the exponential tilt $\frac{dQ\opt}{dP} \propto e^{\lambda\opt \rv}$,
which upweights tail instances of $\rv$. From the convexity of the dual
objective, it is clear that $\lambda\opt$ grows with $y$; equivalently, the
primal optimum approaches the essential supremum of $\rv$ as $y$ grows.

Concretely, the dual reformulation~\eqref{eqn:dual} allows us to build
intuition by yielding analytic expressions for the stability measure $I$ for
specific distributions.



\highlight{
\begin{example}[Gaussian distribution]
  If $\rv \sim P = \mathsf{N}(\mu,s^2)$ and $y \ge \mu$, then
  $\lambda\opt(P) =  \frac{y-\mu}{s^2} $,
  $I_y(P)=  \frac{(y-\mu)^2}{2s^2} $, and
  $Q\opt(P) = \mathsf{N}(y, s^2)$.
  \label{example:Gaussian}
\end{example}

\begin{example}[Gamma distribution]
  If $\rv \sim P = \mathsf{Gamma}(\mcpgamma,\absc)$ and $y \ge \frac{\mcpgamma}{\absc}$, then
  $\lambda\opt(P) = \absc - \frac{\mcpgamma}{y}$,
  $I_y(P)= s y - 1 -\mcpgamma \log(\frac{\mcpgamma}{s y})$, and
  $Q\opt(P) = \mathsf{Gamma}\left(\mcpgamma,\frac{\mcpgamma}{y} \right)$.
  \label{example:Gamma}
\end{example}
}
\vspace{-10pt}

\paragraph{Statistical estimation} We use the dual
formulation~\eqref{eqn:dual} to estimate $I_y(P)$ without parametric
assumptions. Parametric assumptions on $P$ are unrealistic and restrictive in
general, but we view them to be particularly problematic since they require
considering a limited set of distribution shifts $Q$. Denoting by $\emp$ the
empirical distribution over the data $\rv_1, \ldots, \rv_n$, consider the
empirical plug-in estimator 
\begin{equation}
  \label{eqn:dual-plug-in}
  \what{I}_n \defeq \sup_{\lambda \in \R}
  \left\{ \lambda y - \log \E_{\emp}[e^{\lambda \rv}] \right\}.
\end{equation}
Compared to the primal~\eqref{eqn:primal} counterpart that involves
high-dimensional optimization over probability measures, the dual plug-in is
efficient to compute as it only requires a binary search in a single
dimension.  Our main theoretical results to come demonstrate the optimality of
the dual plug-in estimator~\eqref{eqn:dual-plug-in}, and exactly characterize
how estimation becomes more challenging as the threshold $y$ grows.

\highlight{
To illustrate that the simple plug-in estimator $\what{I}_n$ is nearly optimal, we compare it with another natural estimator 
$\Tilde{I}_n$~\eqref{eqn:dual-kde} that is based on
  a kernel density
  estimator $\Tilde{P}_n$ for the density of $\rv$ with the Gaussian
  kernel~\cite{Tsybakov09}:
\begin{equation} 
 \label{eqn:dual-kde}
  \Tilde{I}_n \defeq \sup_{\lambda \in \R}
  \left\{ \lambda y - \log \E_{\Tilde{P}_n}[e^{\lambda \rv}] \right\}.
\end{equation} 
In Figure~\ref{fig:simulation-convergence-plot-with-kde}, we note that 
$\Tilde{I}_n$ converges at a much slower rate compared to $\what{I}_n$ .
\begin{figure}
\centering 
\begin{minipage}{.55\textwidth}
  \vspace{-.5cm}
  \centering
\includegraphics[height=0.23\textheight]{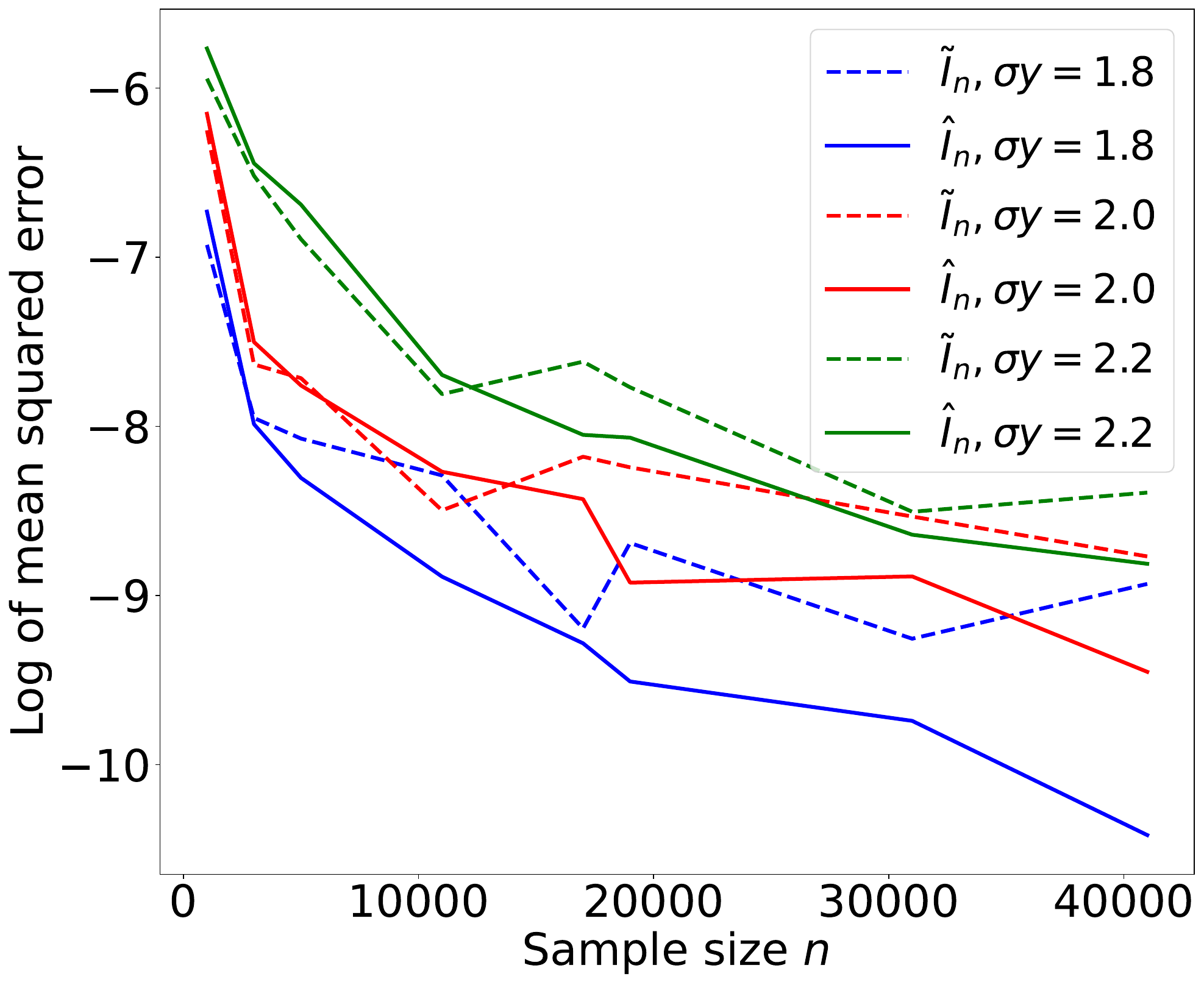}
\caption{Mean squared error of $\what{I}_n$ and $\Tilde{I}_n$ (over 40 runs) with
  $P = \mathsf{Exp}(\absc), \absc \in \{0.9,1,1.1\}$ and a fixed threshold
  $y=2$, here for $\Tilde{I}_n$ we choose the Gaussian kernel function with bandwith $h=0.1$}
\label{fig:simulation-convergence-plot-with-kde}
\end{minipage} 
\end{figure} 
}

\paragraph{Connections to large deviations theory} The dual
reformulation~\eqref{eqn:dual} suggests our stability measure has deep
connections to large deviations theory~\cite{DemboZe98, DeuschelSt89,
  Ellis07}, which has a wide range of applications, including queueing
theory, insurance mathematics, finance, and statistical
mechanics~\cite{Weiss95, Asmussen00, Ellis07, Touchette09}.  To illustrate the
connection concretely, let $\rv_1, \ldots, \rv_n \sim P$ be an i.i.d. sequence
of random variables. Consider estimating the deviations probability of a
random walk $S_m \defeq \sum_{i=1}^m \rv_i$
\begin{equation}
  \label{eqn:rare-prob}
  p_m \defeq \P\left(S_m / m \in A\right)
  ~~~\mbox{where}~~ A \subset \R ~~\mbox{s.t.}~~ \E[\rv] \notin A.
\end{equation}
Large deviations theory~\cite{DemboZe98} shows the approximation
\begin{equation*}
  p_m \approx e^{-m \inf_{y \in {\rm cl}(A)} I_y(P)},
\end{equation*}
where the stability measure $I_y(P)$ is alternatively called the large
deviations rate function.  Our choice of notation is intentional to make this
connection explicit.

Since $y \mapsto I_{y}(P)$ is nondecreasing on $[\E_P[\rv], \infty)$, there is no
real loss of generality in restricting attention to one-sided intervals
$A = [y, \infty)$ for a fixed $y \ge \E_P[\rv]$. In particular, Cramer's
theorem~\cite[Ch VI. Lemma 2.8]{AsmussenGl07} shows
\begin{equation}
  \label{eqn:cramer}
  p_m = \P (S_m \ge m y) = \exp\left( - m I_y(P) + \Theta(\sqrt{m}) \right).
\end{equation}
As a product of this connection, the main statistical results we give in the
sequel can be adapted to provide minimax convergence rates on the deviations
probability~\eqref{eqn:rare-prob}. While these results appear new to the
literature, we omit them for brevity.

 \highlight{
   \paragraph{Extensions}
   While a natural and classical starting point
   (e.g.,~\cite{LemaitreSeArBoGaIo15}), our focus on the KL divergence is
   admittedly somewhat arbitrary. We provide duality results for alternative
   metrics in Appendix~\ref{section:proof-of-duality} and leave as future work
   a proper comparison \emph{between} different stability measures.  In
   contrast to distributionally robust formulations that consider the KL
   divergence across the \emph{joint distribution on the entire data}, the
   notions of stability we consider only concern the single-dimensional random
   variable $\rv$. As a result, we conjecture that the choice of the notion of
   distance in the stability measure is less important than in DRO.  }


\section{Convergence results}
\label{section:convergence}

In the following two sections, we present our main theoretical contributions,
which characterize the minimax rate of convergence for estimating the
stability measure~\eqref{eqn:primal}. We take a two-pronged approach, where we
first establish finite sample convergence guarantees for the empirical plug-in
estimator for the dual formulation~\eqref{eqn:dual}
\begin{equation*}
\what{I}_n \defeq \sup_{\lambda \in \R}
\left\{ \lambda y - \log \E_{\emp}[e^{\lambda \rv}] \right\}.
\end{equation*}
Without imposing parametric assumptions on the distribution of $\rv$, we
characterize how the convergence rate depends on the threshold $y$ and the
abscissa of convergence
$\absc = \inf \{\lambda: \E[e^{\lambda \rv}] = \infty\}$.  In combination with
hardness bounds we give in the next section, our results show that the dual
plug-in estimator $\what{I}_n$ enjoys the best achievable finite sample
convergence rate uniformly over a class of distributions. In contrast, direct
plug-in approaches based on nonparametric estimators of $P$ converge at a
slow nonparametric rate that depends exponentially on the smoothness of the
density of $\rv$~\cite{Wasserman06}.

Our main result in this section quantifies the phase shift behavior in the
convergence rate. As our goal is to obtain convergence guarantees uniform over
the underlying distribution $P$, we consider a class of distributions such
that the optimal dual solution is bounded by some specified upper bound
$\bar{\lambda}$
\begin{equation*}
  \opttilt(P) = \argmax_{\lambda} \left\{ \lambda y - \log \E_{P}[e^{\lambda \rv}] \right\}
  \le \bar{\lambda}.
\end{equation*}
Our proof uses chaining techniques~\cite[Ch 2.2]{VanDerVaartWe96} to uniformly
bound the empirical process
$\lambda \mapsto \E_{\emp}[e^{\lambda \rv}] - \E_{P}[e^{\lambda \rv}]$ over
the interval $[0, \bar{\lambda}]$, which in turn bounds the estimation error
$\what{I}_n - I$.

\highlight{
Our main finite sample result makes explicit the distribution-dependent
constants that govern convergence behavior and thus allow proving uniform
convergence results. Our  upper bound results
(Theorems~\ref{theorem:upper-ld-rate},~\ref{theorem:upper-md-rate}) follow as
corollaries of the following theorem.
\begin{theorem}
  \label{theorem:convergence}
  Let $\bar{\lambda} > 0$ be a fixed constant with
  $\opttilt(P) \le \bar{\lambda}$ and let $\ubletter$ satisfy
  $1 < \ubletter \le \min\{2, \frac{\absc}{\bar{\lambda}}\}$
  where $\absc$ is the abscissa of convergence~\eqref{eqn:abscissa}.
  There is a universal
  constant $K > 0$ such that
  \begin{equation}
    \label{eqn:upper-general}
    \P\left(\left|\what{I}_n - I(P)\right| \ge
    K \left(\frac{1}{\delta}\right)^{\frac{1}{\ubletter}}
     \bar{\lambda}
    \left(
      \E_P[|\rv|^\ubletter e^{\bar{\lambda} \ubletter \rv}]
      \right)^{\frac{1}{\ubletter}}
      n^{\frac{1}{\ubletter} - 1}
      \right) 
      \le \delta.
  \end{equation}
\end{theorem}
\noindent See Section~\ref{section:proof-of-upper} for the proof.
}

Recalling the definition of the abscissa of convergence~\eqref{eqn:abscissa},
$e^{\bar{\lambda} \rv}$ has at most $(\absc / \bar{\lambda})$ number of
moments. This observation uncovers a natural phase shift phenomenon for the
dual plug-in estimator, as we empirically observed in
Figure~\ref{fig:simulation-convergence-plot}. When
$\absc / \bar{\lambda} > 2$, we show the existence of second moments allows
the dual plug-in to achieve the parametric rate $O_p(n^{-1/2})$. On the other
hand, consider the case when $\absc / \bar{\lambda} \le 2$. Since the dual
solution $\opttilt$ grows with $y$---this is easy to verify using first-order
conditions---$\absc / \bar{\lambda} \le 2$ implies the threshold $y$ is large
relative to the abscissa of convergence $\absc$. Since $e^{\opttilt \rv}$ may
not have a second moment, in this case, we prove the convergence rate
deteriorates to $O_p(n^{\bar{\lambda}/\absc - 1})$, which becomes worse as $y$
becomes larger.  Treating these two cases separately, we refer to
$\absc / \bar{\lambda} \le 2$ and $\absc / \bar{\lambda} > 2$ as the large and
moderate deviations regime respectively.

Concretely, we study the minimax rate of convergence for estimating $I_y(P)$
over a natural class of distributions that exhibit Gamma-like tail behavior.
Since Theorem~\ref{theorem:convergence} gives convergence results that
explicitly depend on the underlying data-generating distribution $P$, we take
worst-case rates over this class of distributions with Gamma-like tails.
\begin{definition} For $\absc,y,\mcpgamma$ with $\absc y >1$ and $\mcpgamma \in (0,1)$,
  \label{def:pclass}
  $\mc{P}_{\absc, y,\mcpgamma}$ is the class of distributions satisfying the
  following conditions.
\begin{enumerate}
\item \label{item:nonneg-heavy-tail} $\rv \ge 0$ and $\E_{P}[e^{\absc\rv}] = \infty$. 
\item \label{item:mgf-bound}
    $\E_P[e^{\lambda \rv}] \le \frac{\absc}{\absc- \lambda}$ for
  $0 \le \lambda < \absc$ and $\E_P[\rv] \le \frac{1}{\absc}$.
\item \label{item:argmax-bound}
  $\opttilt(P) = \argmax_{\lambda} \left\{ \lambda y - \log \E_P[e^{\lambda
      \rv}] \right\} \le \sigma - \frac{\mcpgamma}{y} \eqdef \bar{\lambda}$.
\end{enumerate}
\end{definition}
\noindent The condition $\rv \ge 0$ has no real significance other than simplifying the
calculations in the upper bound; the condition merely provides a range on
which $\rv$ is anchored in. As we showcase in our empirical demonstrations,
the translation-invariant property of the stability measure allows translating
the cost $\rv$ and the threshold $y$ by a constant to ensure nonnegativity. We
assume $\absc y > 1$ since $\opttilt(P) = 0$ and
$I(P) = 0$ for $P = \mathsf{Exp}(\absc) \in \mc{P}_{\absc,y,\mcpgamma}$ with
$\absc y \le 1$. The condition  $\E_P[\rv] \le \frac{1}{\absc}$ ensures
$\E_P[\rv] \le \frac{1}{\absc} < y$.

Conditions~\ref{item:nonneg-heavy-tail},~\ref{item:mgf-bound} restrict
$\mc{P}_{\absc,y,\mcpgamma}$ to have Gamma-like tails so that $\P(\rv \ge t)$ decays
like $e^{-\absc t}$ up to polynomial terms. To contextualize
Condition~\ref{item:argmax-bound}, consider
$\rv \sim P_0 = \mathsf{Gamma}(\mcpgamma,\absc)$, which has
$\lambda\opt(P_0) = \absc - \mcpgamma/y$. The final condition thus
considers distributions with the optimal exponential tilting factor
$\opttilt(P)$ bounded by that of the Gamma distribution with shape $\mcpgamma$
and rate $\absc$.  To unpack the condition further, denote by $\kappa_P$ the
cumulant generating function,
$\kappa_P(\lambda) = \log \E_P[e^{\lambda \rv}]$, and recall $\opttilt(P)$
satisfies $\kappa_P'(\opttilt(P)) = y$ from first-order optimality
conditions. Since $\kappa_P(\cdot)$ is convex,
$\kappa_P'(\lambda) = \E_P[\rv e^{\lambda \rv}] / \E_P[e^{\lambda \rv}]$ is
non-decreasing in $\lambda$. Hence, Condition~\ref{item:argmax-bound} is
equivalent to
\begin{equation}
  \E_P[\rv e^{(\absc - \mcpgamma/y) \rv}] \ge y \E_P[e^{(\absc - \mcpgamma/y) \rv}], \label{eqn:first-order-optimality}
\end{equation} 
which roughly implies $\rv$ has heavy enough tails. When
$\rv \sim P_0= \mathsf{Gamma}(\mcpgamma,\absc)$, equality holds in the above
display~\eqref{eqn:first-order-optimality}.

We first consider the large deviations regime where $\absc y \ge
2\mcpgamma$. Our convergence result explicitly quantifies how the statistical
error of the dual plug-in $\what{I}_n$ degrades with $y$.
\begin{theorem}
  \label{theorem:upper-ld-rate}
  Fix $\mcpgamma  \in (0,1)$.
  Let $\absc y \ge 2\mcpgamma$ and $\absc y > 1$. There is a universal constant $K > 0$ such that for any
  $\delta \in (0, 1)$ and $n \ge 8$, 
  \begin{equation}
    \label{eqn:upper-ld-rate}
    \sup_{P \in \mc{P}_{\absc,y,\mcpgamma}}
    \P\left(\left|\what{I}_n - I(P)\right| \ge
      K 
      \left(
 1- \frac{\mcpgamma}{\absc y}
    \right) 
    \left(\frac{1}{\delta}\right)^{2(1 - \frac{\mcpgamma}{\absc y})}
    \left( 4  \left(1 - \frac{\mcpgamma}{\absc y} \right) \log n  \right)^{ 1+ 2(1 - \frac{\mcpgamma}{\absc y})}
      n^{- \frac{\mcpgamma}{\absc y} }
    \right) 
    \le \delta.
  \end{equation}
\end{theorem}
\noindent See Section~\ref{section:proof-of-upper-ld-rate} for the
proof.

On the other hand, in the moderate deviations regime where $\absc y < 2\mcpgamma$, the
dual plug-in estimator $\what{I}_n$ achieves uniform parametric rates.  This
is in contrast to the primal plug-in estimator, which suffers nonparametric
rates even in this regime due to the difficulty of estimating densities.
\begin{theorem}
\label{theorem:upper-md-rate}
  Fix $\mcpgamma \in (0,1)$.
Let $1 < \absc y < 2\mcpgamma$. There is a universal constant $K > 0$ such that for any
$\delta \in (0, 1)$,
\begin{equation}
\label{eqn:upper-md-rate}
\sup_{P \in \mc{P}_{\absc,y,\mcpgamma}}
\P\left(\left|\what{I}_n - I(P)\right| \ge
  K \left(\frac{1}{\delta}\right)^{\frac{1}{2}}
 \left(
 \frac{\absc y}{\mcpgamma-  \frac{\absc y}{2}} 
 \right)^{\half}
\frac{\absc y - \mcpgamma}{\mcpgamma-\frac{\absc y}{2}} 
  n^{-\half}
\right) 
\le \delta.
\end{equation}
\end{theorem}
\noindent See Section~\ref{section:proof-of-upper-md-rate} for the
proof. Using the identity $\E[X] = \int_{0}^\infty \P(X \ge x) dx$ for a
nonnegative random variable $X$, our probability bounds can be transformed to
give upper bounds on the expected absolute error as we initially outlined in
the display~\eqref{eqn:minimax}.

\highlight{
To consider a broader class of distributions, we  extend Definition~\ref{def:pclass} by relaxing Condition~\ref{item:mgf-bound}.
\begin{definition} For $m \ge 1, m \in \mathbb{N}_+$, $\absc,y,\mcpgamma$ with $\absc y >m$, $\mcpgamma \in (0,1)$,    $\mc{P}_{\absc, y,\mcpgamma,m}$ is the class of distributions 
satisfying the following conditions.
\begin{enumerate}
\item \label{item:nonneg-heavy-tail-relaxed} $\rv \ge 0$ and $\E_{P}[e^{\absc\rv}] = \infty$. 
\item \label{item:mgf-bound-relaxed}
$\E_P[e^{\lambda \rv}] \le \paran{\frac{\absc}{\absc- \lambda}}^m$ for
$0 \le \lambda < \absc$ and $\E_P[\rv] \le \frac{m}{\absc}$.
\item \label{item:argmax-bound-relaxed}
$\opttilt(P) = \argmax_{\lambda} \left\{ \lambda y - \log \E_P[e^{\lambda
  \rv}] \right\} \le \sigma - \frac{\mcpgamma}{y} \eqdef \bar{\lambda}$.
\end{enumerate}
\label{def:pclass-relaxed}
\end{definition}
\noindent The condition  $\E_P[\rv] \le \frac{m}{\absc}$ ensures
$\E_P[\rv] \le \frac{m}{\absc} < y$ for any $P \in \mc{P}_{\absc, y,\mcpgamma,m}$. Condition~\ref{item:mgf-bound-relaxed} 
restricts all distributions in $\mc{P}_{\absc, y,\mcpgamma,m}$ to have a lighter tail compared to $\mathsf{Gamma}(m,\absc)$, while
Condition~\ref{item:argmax-bound-relaxed} requires $P$ to have a heavier tail compared to  $\mathsf{Gamma}(\mcpgamma,\absc)$.
Thus,  $\mathsf{Gamma}(k,\absc) \in \mc{P}_{\absc, y,\mcpgamma,m}$ for all $k \in [\mcpgamma,m]$. 

Recalling we assume $\absc y >m$ in Definition~\ref{def:pclass-relaxed},  the inequality $\absc y > m > 2\mcpgamma$ always holds 
for $\mcpgamma < \half$
and therefore we only have the large deviations regime in this case.
We  have the upper bound result that is similar to Theorem~\ref{theorem:upper-ld-rate},
which we prove in Section~\ref{sec:proof-of-upper-ld-rate-relaxed}.
\begin{theorem}
  \label{theorem:upper-ld-rate-relaxed}
  Fix $\mcpgamma  \in (0,1)$.
  Let $\absc y > m > 2\mcpgamma$. There is a universal constant $K > 0$ such that for any
  $\delta \in (0, 1)$ and $n \ge 8$, 
  \begin{equation}
    \label{eqn:upper-ld-rate-relaxed}
    \sup_{P \in \mc{P}_{\absc,y,\mcpgamma,m}}
    \P\left(\left|\what{I}_n - I(P)\right| \ge
      K 
      \left(
 1- \frac{\mcpgamma}{\absc y}
    \right) 
    \left(\frac{1}{\delta}\right)^{2(1 - \frac{\mcpgamma}{\absc y})}
    \left( 4  \left(1 - \frac{\mcpgamma}{\absc y} \right) \log n  \right)^{ 1+ 2m(1 - \frac{\mcpgamma}{\absc y})}
      n^{- \frac{\mcpgamma}{\absc y} }
    \right) 
    \le \delta.
  \end{equation}
\end{theorem}


}

\highlight{
  We can also extend the results in this section to light-tailed random variables.
  Letting $\absc \to \infty$ in Theorem~\ref{theorem:convergence}, we obtain the following result.
  \begin{corollary}
    \label{cor:gaussian-ub}
Let $\bar{\lambda}$ be any fixed constant with
$\opttilt(P) \le \bar{\lambda}$ and assume $\E_P[e^{\lambda \rv}] < \infty$ for any $\lambda$.
There exists a universal
constant $K > 0$ such that
\begin{equation}
\label{eqn:upper-light-tail}
\P\left(\left|\what{I}_n - I(P)\right| \ge
K \left(\frac{1}{\delta}\right)^{\frac{1}{2}}
\bar{\lambda}
\left(
\E_P[|\rv|^2 e^{2\bar{\lambda} \rv}]
\right)^{\frac{1}{2}}
n^{-\frac{1}{2}}
\right) 
\le \delta.
\end{equation}
\end{corollary}
Concretely, we study the class of light-tailed distributions
$\mc{Q}_{\bar{\lambda},y}$   defined below.
\begin{definition} 
For $\bar{\lambda}> 0$ and $y$, 
  $\mc{Q}_{\bar{\lambda},y}$ is the class of distributions satisfying the
  following conditions.
\begin{enumerate}
\item \label{item:nonneg-light-tail}   $\E_{P}[e^{\lambda\rv}] < \infty$ for any $\lambda$.  
\item \label{item:argmax-bound-light-tail}
  $\opttilt_y(P) = \argmax_{\lambda} \left\{ \lambda y - \log \E_P[e^{\lambda
      \rv}] \right\} \le  \bar{\lambda}$.
\end{enumerate}
  \label{def:pclass-light-tail} 
\end{definition}

Corollary~\ref{cor:gaussian-ub} implies the standard $1/\sqrt{n}$ convergence rate
\begin{equation*}
 \sup_{P \in \mc{Q}_{\bar{\lambda},y}}
\P\left( \left|\what{I}_n(\rv_1^n) - I_y(P)\right| \ge 
K \left(\frac{1}{\delta}\right)^{\frac{1}{2}}
 \frac{1}{\sqrt{n}} 
\right)
\le \delta,
\end{equation*}
where $K$ is some constant that depends on $\bar{\lambda},y$.  For fixed
$\bar{\lambda}$ and $y$, we can show (see
Section~\ref{section:proof-property-Q-class}) that $\mc{Q}_{\bar{\lambda},y}$
contains the following sub-Gaussian distributions:
\begin{align*}
 \set{\mathsf{N}(\mu, s^2): \frac{y - \mu}{s^2} \le \bar{\lambda}}, \set{\mathsf{Ber}(p): 
p \ge \frac{y}{e^{\bar{\lambda}} - y   ( e^{\bar{\lambda}}- 1)}
}, \set{\mathsf{Uni}(a,b):  \frac{b-a}{b-y} \le \bar{\lambda}} \subset \mc{Q}_{\bar{\lambda},y}, 
\end{align*}
To further understand $\mc{Q}_{\bar{\lambda},y}$, consider any light-tailed
random variable $\rv \sim P$ with $y < \esssup \rv$, we can show that there
exists $\Tilde{\lambda}(y,P) < \infty$ such that
$P \in \mc{Q}_{\Tilde{\lambda}(y,P),y}$.  This implies for any finite set of
light-tailed distributions $\mc{P}$ with $y \le \esssup \rv$ for any
$\rv \sim P \in \mc{P}$, we can find $\Tilde{\lambda}(y,\mc{P})$ such that
$\mc{P} \subseteq \mc{Q}_{\Tilde{\lambda}(y,\mc{P}),y}$. 
}


\section{Fundamental hardness}
\label{section:hardness}

We now provide lower bounds on the minimax risk~\eqref{eqn:minimax} that
complement our convergence results in the previous section
(Theorems~\ref{theorem:upper-ld-rate} and~\ref{theorem:upper-md-rate}). Our
lower bounds quantify the fundamental hardness of estimation and show that the
plug-in estimator~\eqref{eqn:dual-plug-in} achieves the optimal rate of
convergence up to polylogarithmic factors. Unlike previous asymptotic results
that consider a fixed data-generating distribution $P$~\cite{Feuerverger89},
we consider a fixed finite sample size and quantify the difficulty of
estimation over a class of data-generating distributions. As minimax lower
bounds~\cite{Wainwright19} can be overly pessimistic when considering a large
set of data-generating distributions, we focus on the adapted class of
distributions $\mc{P}_{\absc, y, \mcpgamma}$ that exhibit a particular
Gamma-like tail behavior defined in Definition~\ref{def:pclass}.

As an example of pessimistic hardness results when considering a large class
of data-generating distributions, consider the case where we do not restrict
distributions in $\mc{P}_{\absc,y, \mcpgamma}$ to have the same abscissa of
convergence $\absc = \inf \{\lambda: \E[e^{\lambda \rv}] = \infty\}$.  So long
as $\mc{P}_{\absc,y, \mcpgamma}$ just contains two distributions with tails
decaying like $e^{-\absc x}$ and $e^{-\absc(1 - 1/\log n) x}$, we can show
that no estimator of stability $I_y(P)$ can have expected estimation error
better than $\Omega(1 / \log n)$. Intuitively, this is due to the inherent
difficulty in estimating $\absc$, the abscissa of convergence. Since $\absc$
is the knife-edge determining the existence of the MGF, it is difficult to
estimate: it is well known that the minimax rate for estimating $\absc$ is
$\Theta_p(1/\log n)$~\cite{HallTeVa92}.

With this theoretical motivation in mind, we focus on the admissible setting
where $\absc$ is fixed
(Conditions~\ref{item:nonneg-heavy-tail},~\ref{item:mgf-bound}). We
characterize a \emph{tailored} notion of hardness adapted to the class
$\mc{P}_{\absc, y, \mcpgamma}$, and show an unavoidable degradation from the
parametric rate even over the restricted set of distributions with a fixed
abscissa of convergence~$\absc$.  The phase shift behavior
$\widetilde{\Theta}_p(n^{- (\half \wedge \frac{\mcpgamma}{\absc y})})$ in our
bound~\eqref{eqn:minimax} is fundamental. Estimation becomes harder for large
threshold values $y$, as well as for light-tailed distributions with higher
$\absc$. Although a priori counterintuitive, the latter scaling is natural
since adverse events with large values of $R$ occur more rarely for
lighter-tailed distributions, and the stability measure is thus more difficult
to estimate.

To prove our minimax lower bound, we use Le Cam's reduction from estimation to
hypothesis testing~\cite{LeCam73, LeCam86, Yu97}.  Our information-theoretic
approach constructs two distributions close in total variation distance---thus
difficult to tell apart using finite samples---but sufficiently separated with
respect to the stability $I(P)$. Although we state probabilistic minimax
bounds for the estimation error, their expectation counterparts can also be
easily shown from our proofs. Following Section~\ref{section:convergence}, we
first focus on the large deviations regime, where we observe a degradation from
the standard parametric rate.
\begin{theorem}
\label{theorem:lower-ld-rate}
Fix $\mcpgamma \in (\half, 1)$ and   $\delta \in (0, \half)$.
Let $\absc y \ge 2 \mcpgamma > 1$ and define $c = \frac{1}{2(1-2\delta)^2}$. If
\begin{align*}
  \log cn \ge     \max\Bigg\{
    \frac{\absc y}{\mcpgamma} \left( 2 \log \frac{\log cn}{\absc} + \log \left( \frac{4 \absc \left( 
    \frac{\mcpgamma}{y} \vee 1 \right) }{1-\mcpgamma}  \right)\right),
  & ~\absc y \paran{\frac{2}{\mcpgamma^2} \frac{\absc y + 1}{\absc y - 1} \vee \frac{2}{\mcpgamma}
    \vee \frac{4}{\absc y -  1 }  }, \\
  &   \frac{2 \absc \log 3(1 \vee \absc^{-2})}{1 - \mcpgamma},  \absc+ 1,  \frac{2}{1-\mcpgamma} \Bigg\},
\end{align*}
\begin{equation*}
\mathfrak{M}_n \defeq
\inf_{\what{I}_n} \sup_{P \in \mc{P}_{\absc,y,\mcpgamma}}
\P\left( \left|\what{I}_n(\rv_1^n) - I(P)\right| \ge 
         \left(\frac{\absc y}{\log cn}\right)^{1-\mcpgamma}
\frac{\absc y - 1}{4 \Gamma(\mcpgamma) e^{\frac{1}{\absc y}} (\absc y + 1)(\log cn- 1)}
\left(  \frac{1}{cn} \right)^{\frac{\mcpgamma}{\absc y}}
\right)
\ge \delta,
\end{equation*}
where the infimum is taken over the set of $(\rv_1, \ldots, \rv_n)$-measurable
functions.
\end{theorem}

Deferring a formal proof of Theorem~\ref{theorem:lower-ld-rate} to
Section~\ref{section:proof-of-lower-ld-rate}, we outline the main ideas
below. We construct two distributions $P_1, P_2 \in \mc{P}_{\absc,y,\mcpgamma}$
that exhibit slightly different tail behaviors
\begin{align}
  \label{eqn:ld-construction}
  f_1(x) \propto  x^{\mcpgamma + \frac{1}{\absc x_0} - 1 } e^{-\absc x}\indic{x \ge 0}~~\mbox{and}~~
  f_2(x) \propto
  \begin{cases}
    x^{\mcpgamma  + \frac{1}{\absc x_0} - 1} e^{-\absc x} & ~~\mbox{if}~~ 0 \le x \le x_0 \\
    x^{-1} e^{-\absc x} & ~~\mbox{if}~~ x > x_0 
  \end{cases}.
\end{align}
By selecting $x_0$ appropriately, we can show that samples from the two
distributions are difficult to tell apart with probability at least $\delta$,
yet the stability measures $I(P_1)$ and $I(P_2)$ are at least
$\widetilde{\Omega}(n^{-\frac{\mcpgamma}{\absc y}})$ separated.

In the moderate deviations regime where $\absc y < 2\mcpgamma$, we can show
that the minimax lower bound reduces to the usual parametric rate.
\begin{theorem}
\label{theorem:lower-md-rate}
  Fix $\mcpgamma \in (\half,1)$ and   $\delta \in (0, \half)$.
  Let $1 < \absc y < 2\mcpgamma$. If
\begin{equation*}
\log n \ge \max\left\{ 2 \left( 
\log \left( \frac{1}{(1-\mcpgamma) \wedge 
(2 - \absc y)} \right)
+
\log 2(1-2\delta)
  + \log \absc y \right),
\absc y \log 2 + \log 2(1-2\delta) \right\},
\end{equation*}
\begin{equation*}
\mathfrak{M}_n \defeq
\inf_{\what{I}_n} \sup_{P \in \mc{P}_{\absc,y,\mcpgamma}}
\P\left( \left|\what{I}_n(\rv_1^n) - I(P)\right| \ge
\frac{(\absc y -1)(1-2\delta)}{4 \sqrt{n}} \right)
\ge \delta.
\end{equation*}
\end{theorem}
\noindent Compared to the construction~\eqref{eqn:ld-construction} for
Theorem~\ref{theorem:lower-ld-rate}, the ``hard examples'' $P_1$ and $P_2$
take on a different shape when $\absc y < 2\mcpgamma$. In this regime, we
consider a more standard example of two distributions that differ in their
average-case behavior (when $x \le x_0$) in our proof in
Section~\ref{section:proof-of-lower-md-rate}.

\highlight{ 
The lower bound results shown Theorems~\ref{theorem:lower-ld-rate} and~\ref{theorem:lower-md-rate} 
still hold when we replace $P \in \mc{P}_{\absc,y,\mcpgamma}$ with $P \in \mc{P}_{\absc,y,\mcpgamma,m}$ introduced in 
Definition~\ref{def:pclass-relaxed}
since $\mc{P}_{\absc,y,\mcpgamma} \subseteq \mc{P}_{\absc,y,\mcpgamma,m}$ for any $m \ge 1$.
Combining these with Theorem~\ref{theorem:upper-ld-rate-relaxed}, 
 we   extend our main result~\eqref{eqn:minimax} to the more general
class of distributions $\mc{P}_{\absc,y,\mcpgamma,m}$:
\begin{equation*} 
 \inf_{\what{I}_n} \sup_{P \in \mc{P}_{\absc,y,\mcpgamma,m}}
  \E_P  \left| \what{I}_n - I_y(P) \right|
  \asymp  n^{- \paran{\half \wedge \frac{\mcpgamma}{\absc y}}}.
\end{equation*} 
}

\highlight{
We can also extend the main results in this section to random variables with light-tailed random variables.
In this case, we recover the usual parametric rate. See Section~\ref{section:proof-gaussian-lb} for the proof.
\begin{corollary}
\label{cor:gaussian-lb}
Fix $\bar{\lambda}, y$ such that $\mc{Q}_{\bar{\lambda}, y} \ne \emptyset$ and   $\delta \in (0, \half)$.
Define $c = 2(1-2\delta)$. If $n \ge \paran{\frac{c}{\bar{\lambda}}}^2$, 
\begin{equation*}
\mathfrak{M}_n \defeq
\inf_{\what{I}_n} \sup_{P \in\mc{Q}_{\bar{\lambda}, y}}
\P\left( \left|\what{I}_n(\rv_1^n) - I_y(P)\right| \ge 
 \bar{\lambda}
 \frac{c}{2\sqrt{n}} 
\right)
\ge \delta,
\end{equation*}
where the infimum is taken over the set of $(\rv_1, \ldots, \rv_n)$-measurable
functions.
\end{corollary}
 }


\section{Experiments}
\label{section:experiment}

We empirically demonstrate our approach in two operational problems where it
is critical to evaluate system stability prior to deployment. First, we study
scheduling policies in queueing, where the modeler learns a scheduling policy
based on a simulator and wishes to ensure the deployed policy maintains robust
performance over multiple real-world environments. Second, we turn our
attention to the empirical evaluation of prediction models. We study models
trained to predict healthcare resource utilization based on past and present
data. Since demographic distributions shift over time, we wish to ensure that
predictive performance remains stable in the future.  We adapt our proposed
methodology to be less conservative by considering distribution shift only
over a subset of covariates instead of the usual joint distribution shift over
all randomness. This allows modeling more structured and realistic
distribution shifts that occur in predictive modeling.  In both scenarios, the
proposed stability measure differentiates between brittle vs. robust models,
in contrast to typical average-case performance metrics that do not account
for tail performance. \highlight{Our empirical study also highlights the
  limitation of the proposed approach and directions of future work.}

\subsection{Queueing systems}
\label{section:queueing}

Managing congestion under resource constraints is a core goal in operations
research~\cite{Williams16}. Scheduling policies must maintain robust
performance over varied distribution shifts that occur in real-world queueing
systems~\cite{Koopman72, GreenKo89,GreenKoSv91,EickMaWh93}. In contrast to
typical simplifying assumptions in queueing theory, arrival distributions are
often highly nonstationary~\cite{GlynnZh19}, and demand surges are
frequent~\cite{HuChDo21}. Simultaneously, supply-side disruptions due to
staffing difficulties~\cite{Cook21, GargNa22} are increasingly common, leading
to nonstationary service rates.

We use a G/G/1 queue with multi-class jobs as the primary vehicle for modeling
stochastic workloads, where we consider jobs that incur quadratic costs with
sojourn time~\cite{VanMieghem95}.  Solving for the optimal scheduling policy
is intractable even for a single-server queue with a fixed arrival and service
distribution due to the prohibitively large state/policy
spaces~\cite{PapadimitriouTs99}.  We compare the stability of two
approximation methods that take radically different approaches.  First, we
look to the classical heavy-traffic limit literature that designs effective
policies under highly congested environments~\cite{HarrisonWe89, VanMieghem95,
  Whitt02, HarrisonZe04, MandelbaumSt04}. Under the tractable heavy-traffic
diffusion limit, a simple index-based myopic policy---the generalized c-$\mu$
rule (Gc-$\mu$)---is optimal for convex cost functions~\cite{VanMieghem95}.
The Gc-$\mu$ rule is intuitive and does not require any arrival information:
it serves jobs with the highest index calculated by the product of
instantaneous cost and the service rate. It enjoys a natural adversarial
robustness guarantee under demand surges, although its optimality is only
guaranteed in the heavy-traffic limit.

As an alternative to the classical index-based policy, we study deep
reinforcement learning (DRL) approaches to queueing. DRL methods have recently
achieved remarkable success in games and
robotics~\cite{SilverScSiAnHuGuHuBaLaBo17, GuHoLiLe17}, and can heuristically
solve the scheduling dynamic program through black-box function
approximations. However, they are generally difficult to train reliably due to
unbounded state spaces in queueing systems~\cite{WaltonXu21, DaiGl21}.  They
empirically exhibit high sensitivity to hyperparameters, implementation
details, and even random seeds~\cite{HendersonIsBa18}. Although DRL methods
require significant engineering effort and a large number of simulated samples
to achieve good performance in practice, we demonstrate below that DRL
approaches can effectively optimize queueing performance under a fixed
system dynamics.

We consider a single-server multi-class queueing model operating in a finite
time interval $[0,100]$, which we implement using a custom discrete-event
simulator. We consider three job types $j \in \mathcal{J} = \{1, 2, 3\}$ with
independent interarrival times. Denote by $\lambda_j(t)$ and $\mu_j(t)$ the
arrival and service rates and $\rho_j(t) = \frac{\lambda_j(t)}{\mu_j(t)}$ the
traffic intensity at time $t$.  We consider cost functions $C_j(x)= w_j x^2$
with $w_1=1,w_2=3,w_3=6$ where $x$ is the time a job spent in the system.  We
focus on a scenario where the modeler has access to a calibrated simulator
with independent inter-service times following an exponential distribution
with mean $\frac{1}{\mu_j(t)}$ and independent interarrival distribution
\begin{equation*}
  \half~ \mathsf{Exp}(\lambda_j(t)) + \half~ \left|\normal\left(0,\lambda_j(t)^{-2} \frac{\pi}{2}\right)\right|.
\end{equation*}
Following standard practice, we fix the training distribution to have
stationary arrival and service rates. 

\begin{figure}[t]
  \centering
  \begin{subfigure}[t]{0.45\textwidth}
\centering
\includegraphics[height=0.21\textheight]{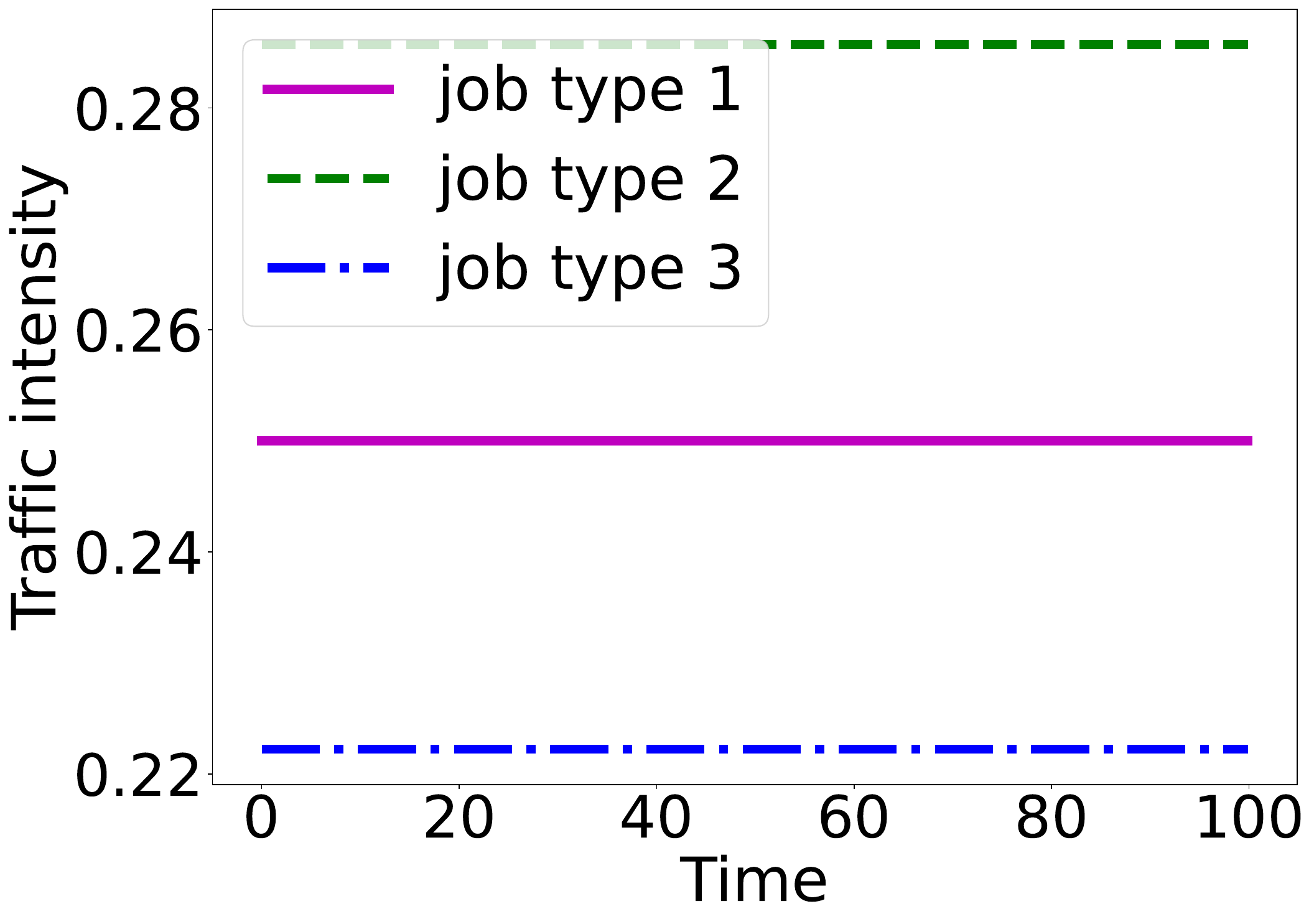}
\caption{Traffic intensities $\rho(t) = \frac{\lambda(t)}{\mu(t)}$
  representing arrival and service rates for each class.}
\end{subfigure}
\hfill
\begin{subfigure}[t]{0.45\textwidth}
\centering
\includegraphics[height=0.21\textheight]{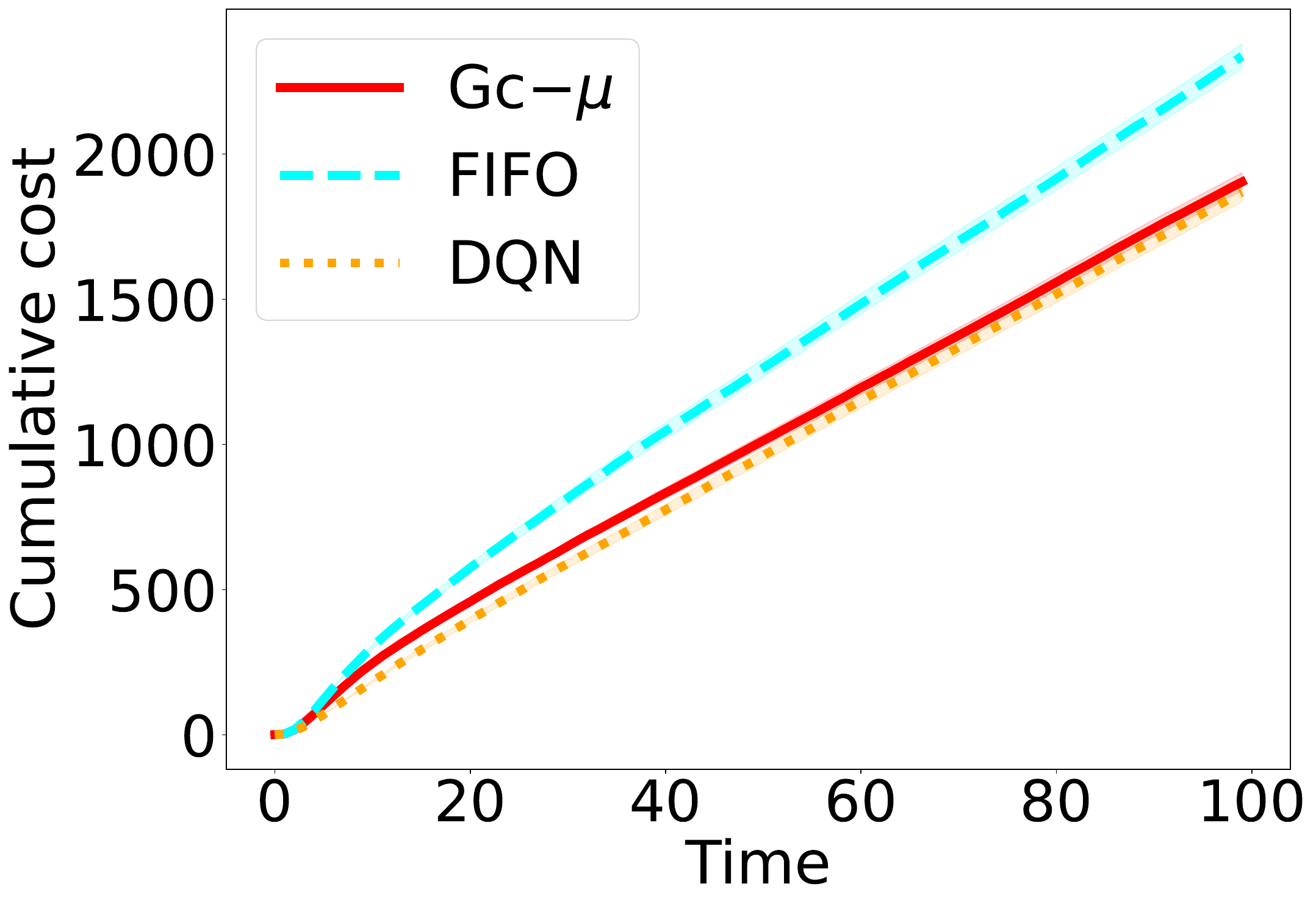}
\caption{Average cumulative costs. Shaded regions denote $\pm 1.96~\times$ standard deviations.}
\end{subfigure}
\caption{Performance of policies on i.i.d. test data}
\label{fig:queue-train}
\end{figure}

We compare three different scheduling policies.  As our first policy, letting
$a_j(t)$ be the age of the oldest type $j$ job at time $t$, the Gc-$\mu$ rule
serves jobs with the highest value of $C_j'(a_j(t))\mu_j(t)$ (breaking ties
arbitrarily). Since the Gc-$\mu$ rule only requires knowledge of service
rates, we assume $\mu_j(\cdot)$ is known as it can easily be estimated from
offline data. Second, we use a Q-learning method where we estimate the value
function using a feedforward neural network (DQN; Deep
Q-Networks~\cite{MnihKaSiGrAnWiRi13}). Deferring implementation details to
Section~\ref{section:detail-DQN}, we note that since the DQN policy is
finetuned to the particular simulation setting---itself often calibrated from
data---its robustness under distribution shift is of substantial
concern. Finally, we consider the first-in-first-out (FIFO) policy as a simple
benchmark.

On the training distribution with stationary arrival and service rates, we
first evaluate the average cumulative cost under each policy over $n=100,000$
sample paths.  Since the DQN model was trained to optimize performance on this
distribution, we observe it achieves the best average-case performance in
Figure~\ref{fig:queue-train}. As distribution shifts and nonstationarities are
common in practice~\cite{Koopman72,GreenKo89, GreenKoSv91,EickMaWh93}, we also
evaluate the stability of each policy using the cumulative cost at the end of
the horizon ($t = 100$).  \highlight{We first assess the validity of our
  estimation strategy. We  again compare our estimator $\what{I}_n$ based on the dual
  reformulation~\eqref{eqn:dual} with the one based on kernel density
  estimation of $\rv$~\eqref{eqn:dual-kde}. 
  Similar to
  Figure~\ref{fig:simulation-convergence-plot}, in
  Figure~\ref{fig:queue-convergence-plot}, we analyze the convergence rate of
  different estimators. For all three decision policies, we observe that our
  dual-based estimator outperforms nonparametric counterparts by multiple
  orders of magnitude.}

\begin{figure}[t]
  \centering
  \begin{minipage}{.55\textwidth} 
  \centering
\includegraphics[height=0.23\textheight]{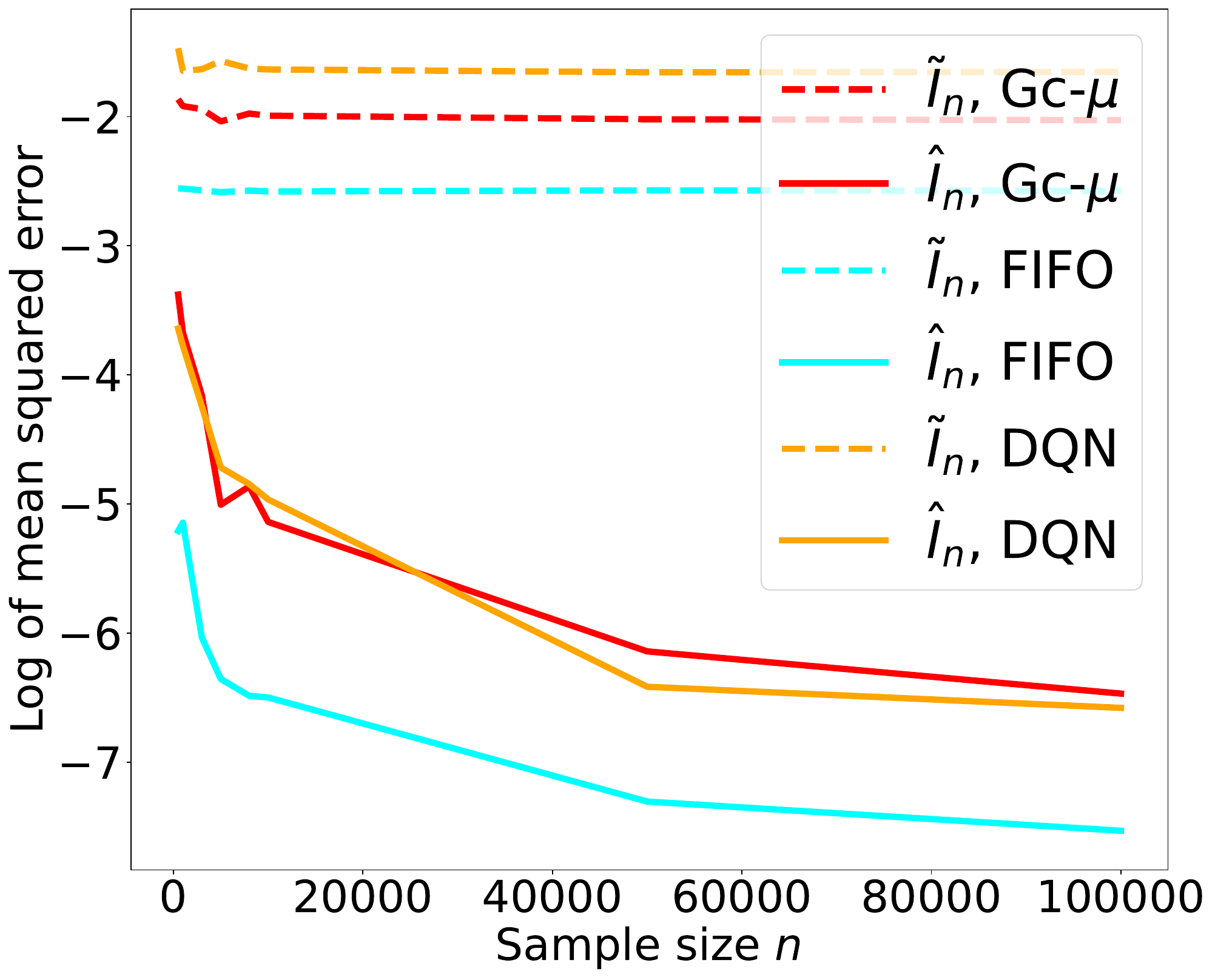}
\caption{\highlight{Mean squared error of $\what{I}_n$ (over 40 runs) for different policies with threshold $y= 3720.66$.
Here for the KDE estimator $\Tilde{I}_n$~\eqref{eqn:dual-kde} we choose the Gaussian kernel function with bandwith $h = 100$.}
}
\label{fig:queue-convergence-plot}   
\end{minipage} 
\begin{minipage}{.44\textwidth}
  \centering 
\includegraphics[height=0.23\textheight]{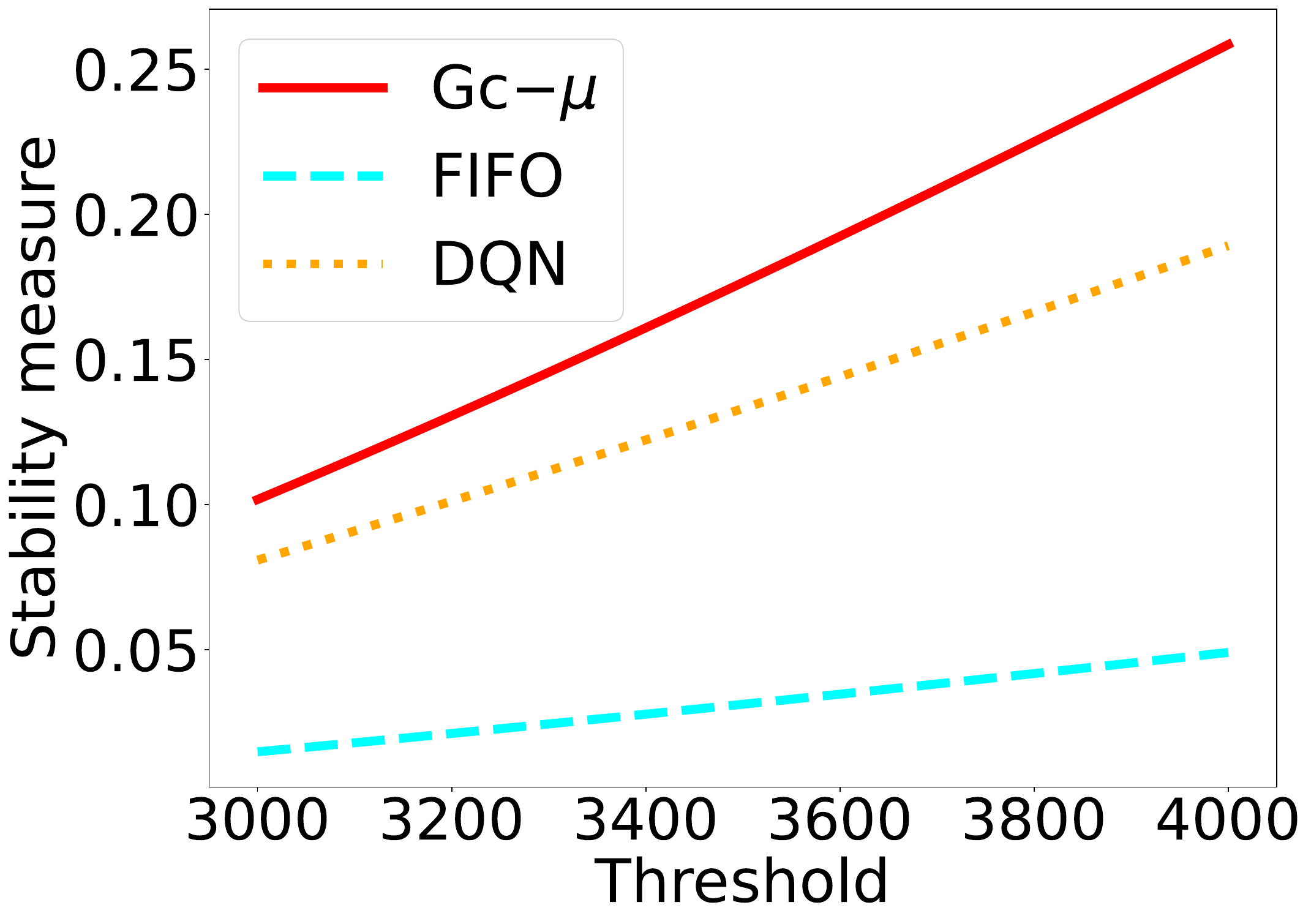}
\caption{Stability $\what{I}_n$ for different policies over a set of thresholds}  
\label{fig:queue-stability-plot}
\end{minipage}%
\end{figure}

Next, we use our dual-based estimator to compare the stability of different
decision policies. To set a common benchmark performance for all policies, we
choose the threshold value $y$ to be twice the mean of the average cumulative
cost under the DQN policy. Despite choosing a threshold value that unfairly
favors the DQN policy---as it had the lowest average cost---in
Table~\ref{table:queue} we observe that the Gc-$\mu$ rule exhibits the highest
stability.  To further examine whether Gc-$\mu$ rule is robust due to the
choice of $y$, we compute $\what{I}_n$ over a wide range of thresholds, and in
Figure~\ref{fig:queue-stability-plot}, we see the gap in stability growing as
$y$ increases.

\begin{table}[H] 
\centering
\begin{tabular}{|l|l|l|l|}
\hline
Policy &  Stability $\what{I}_n$ & Average Cost  \\ \hline
Gc-$\mu$ & 0.21194 & 1928.72    \\  
DQN    & 0.15741 & 1861.79 \\  
FIFO   &  0.03887 & 2345.53 \\ \hline
\end{tabular}
\captionof{table}{Stability with threshold $y= 3720.66$}
\label{table:queue}
\end{table}

The proposed stability measure simulates worst-case distribution shifts by
only using training data. To contextualize our findings, we simulate five
concrete distribution shifts modeling different changes to arrival and service
rates in the same family of distributions as above.
\begin{enumerate} 
\setlength{\itemsep}{1pt}
\setlength{\parskip}{0pt}
\setlength{\parsep}{0pt}
\item Arrival rates linearly increase to certain levels over time. \label{item:queue-shift-1}
\item There is a sudden increase in arrival rates.  \label{item:queue-shift-2}
\item Service rates linearly decrease to certain levels over time.  \label{item:queue-shift-3}
\item Arrival rates linearly increase, and service rates linearly decrease to certain levels over time.  \label{item:queue-shift-4}
\item Arrival rates are cyclical (sinusoidal).  \label{item:queue-shift-5}
\end{enumerate}

\begin{figure}[H]
\centering
\begin{subfigure}[t]{0.45\textwidth}
\includegraphics[height=0.21\textheight]{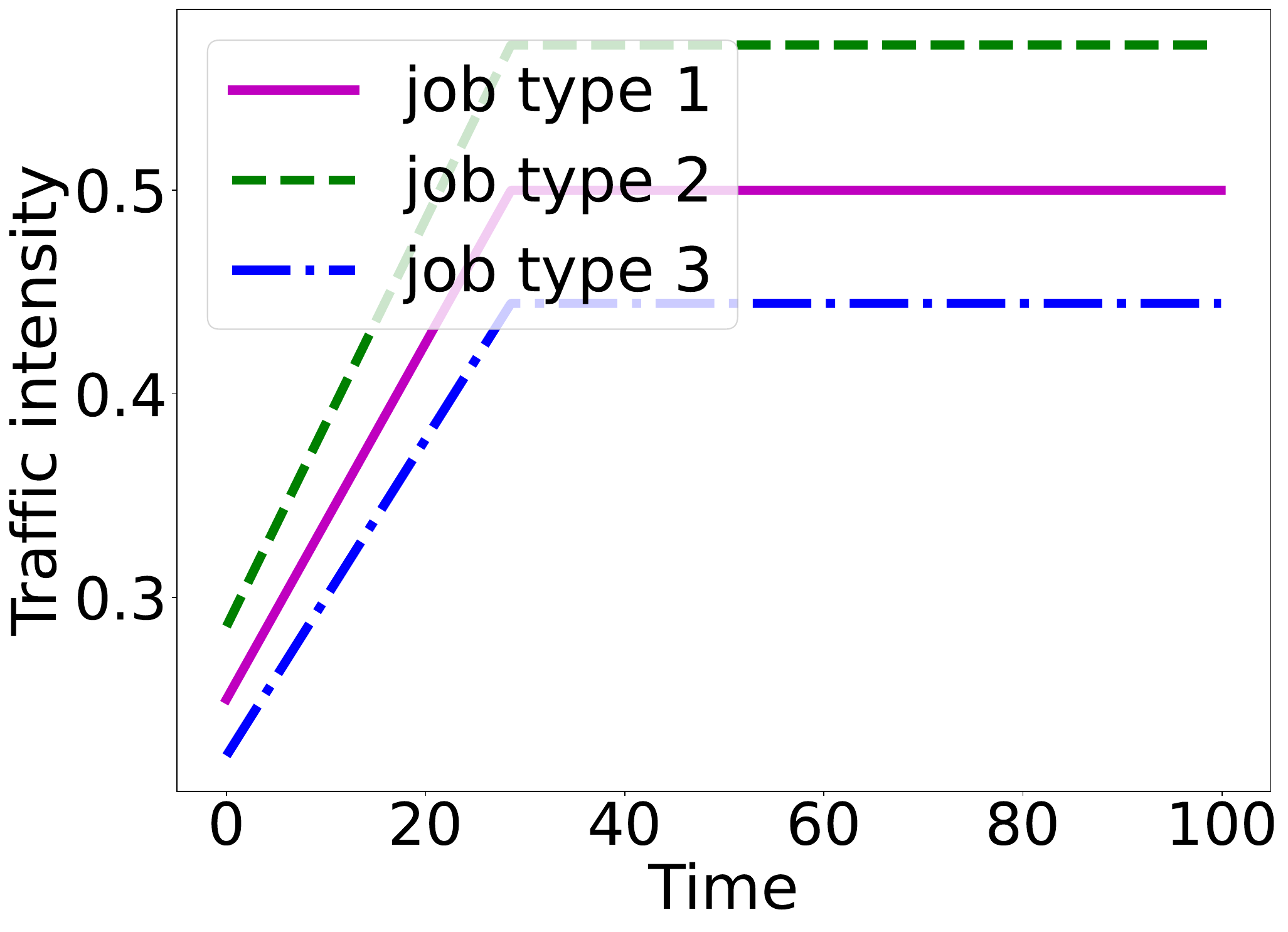}
\caption{Traffic intensities $\rho(t) = \frac{\lambda(t)}{\mu(t)}$   representing arrival and service rates for each class.} 
\end{subfigure}
\hfill
\begin{subfigure}[t]{0.45\textwidth}
\includegraphics[height=0.21\textheight]{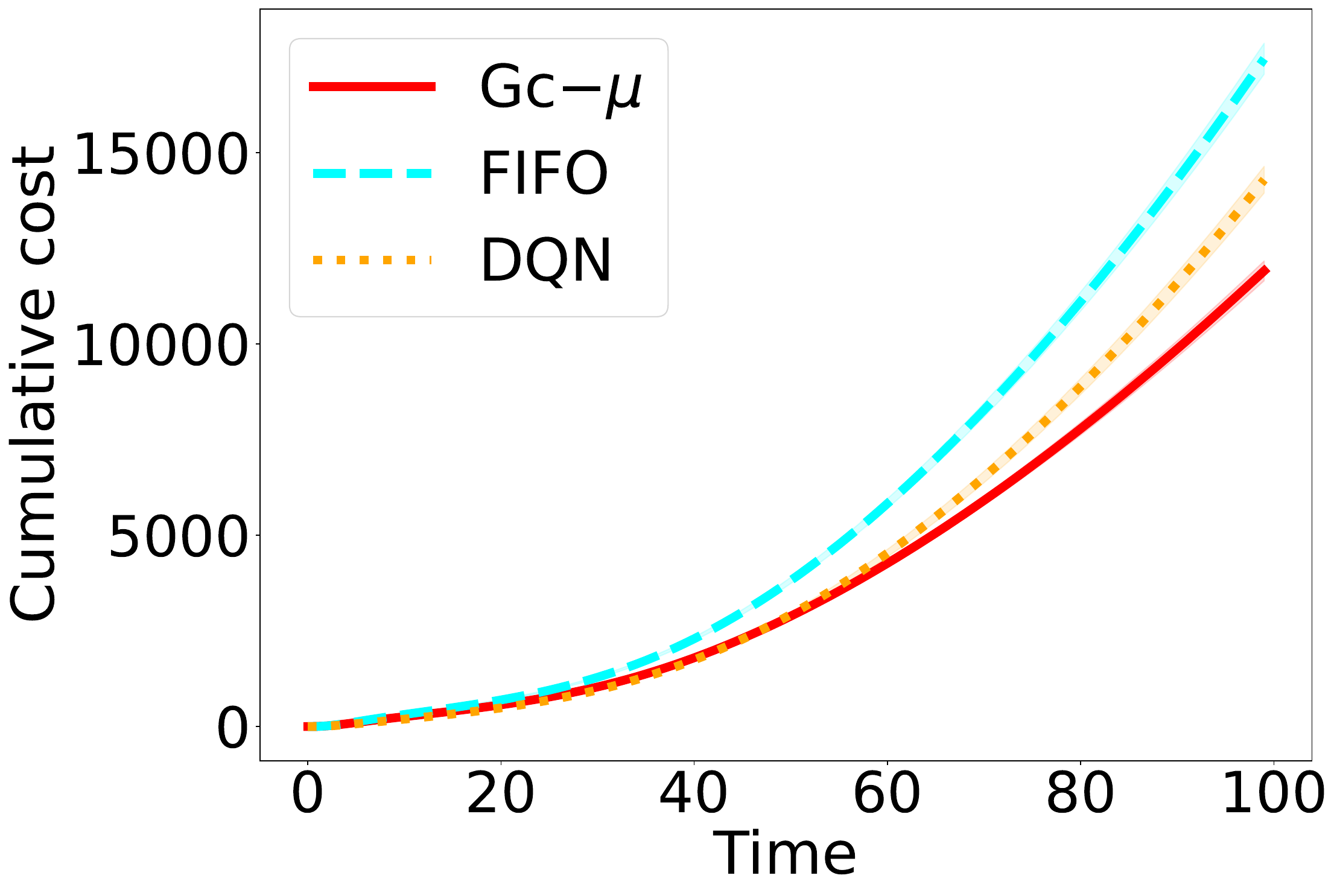}
\caption{Average cumulative costs. Shaded regions denote $\pm 1.96~\times$ standard deviations.} 
\end{subfigure}
\caption{Performance of  policies on data under distribution shift~\ref{item:queue-shift-1}}
\label{fig:queue-shift-1}
\end{figure}

\begin{figure}[H]
\centering
\begin{subfigure}[t]{0.442\textwidth}
\centering
\includegraphics[height=0.21\textheight]{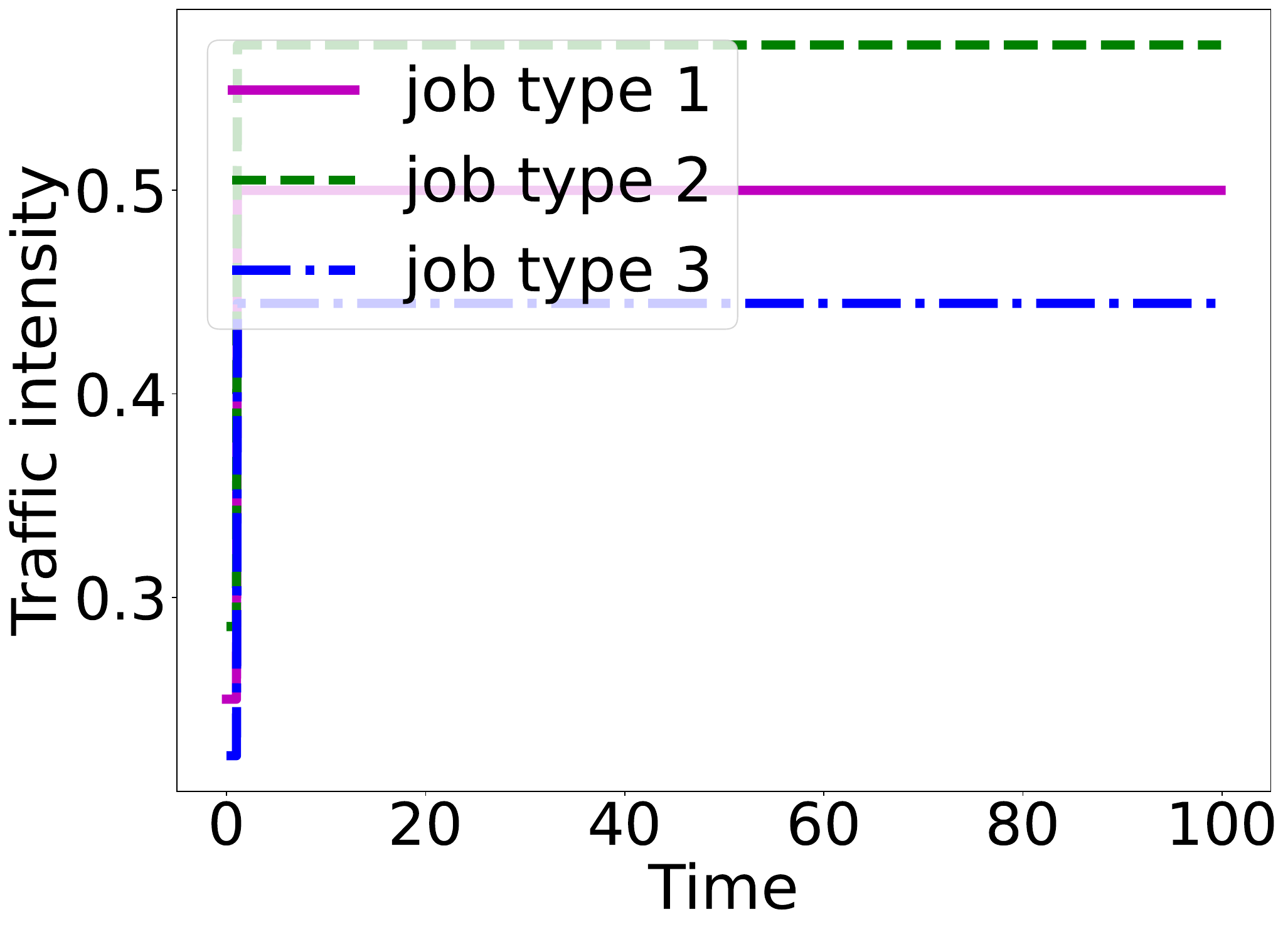}
\caption{Traffic intensities $\rho(t) = \frac{\lambda(t)}{\mu(t)}$   representing arrival and service rates for each class.} 
\end{subfigure}
\hfill
\begin{subfigure}[t]{0.45\textwidth}
\centering
\includegraphics[height=0.21\textheight]{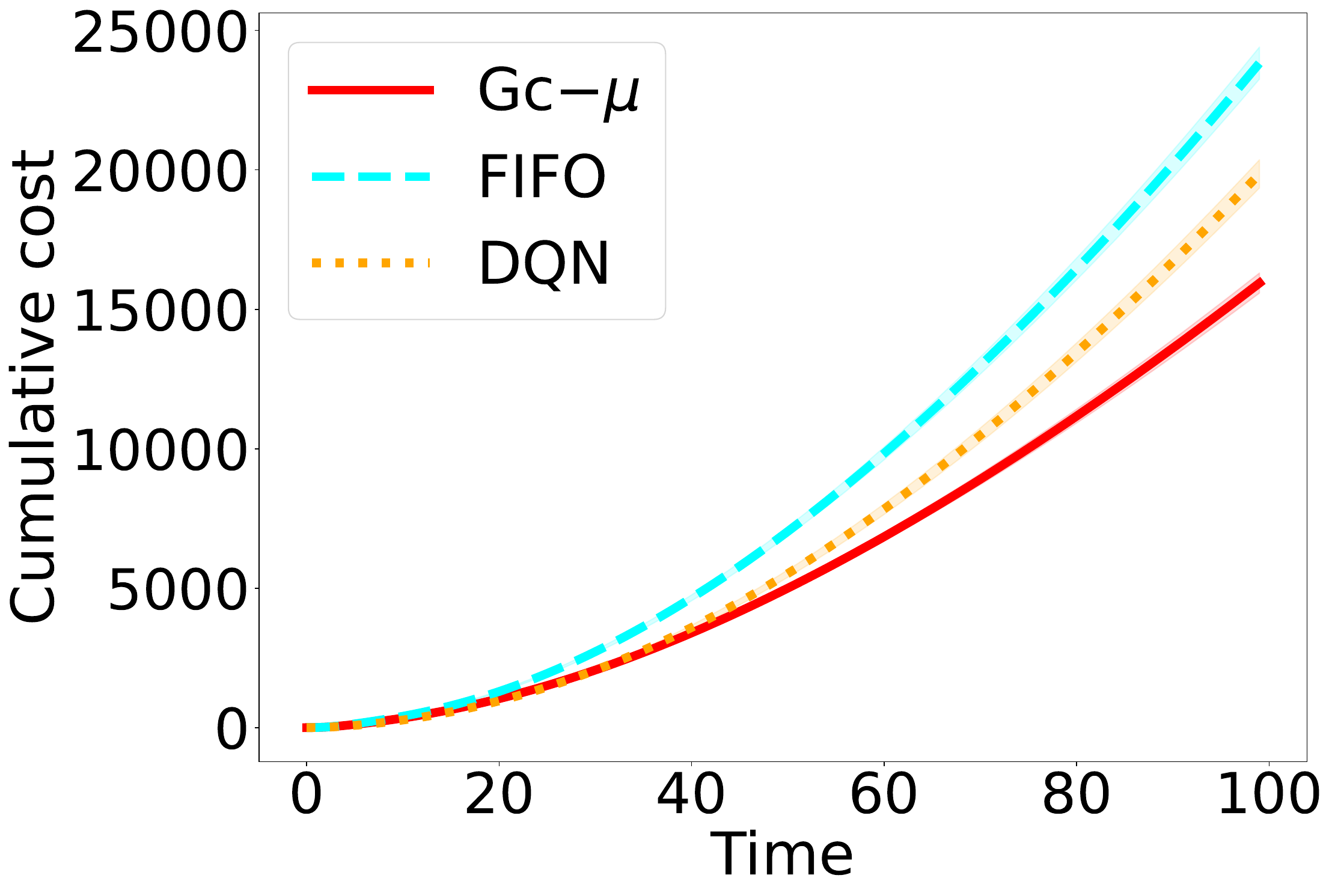}
\caption{Average cumulative costs. Shaded regions denote $\pm 1.96~\times$ standard deviations.} 
\end{subfigure}
\caption{Performance of  policies on data under distribution shift~\ref{item:queue-shift-2}}
\label{fig:queue-shift-2}
\end{figure}

\begin{figure}[H]
\centering
\begin{subfigure}[t]{0.442\textwidth}
\centering
\includegraphics[height=0.21\textheight]{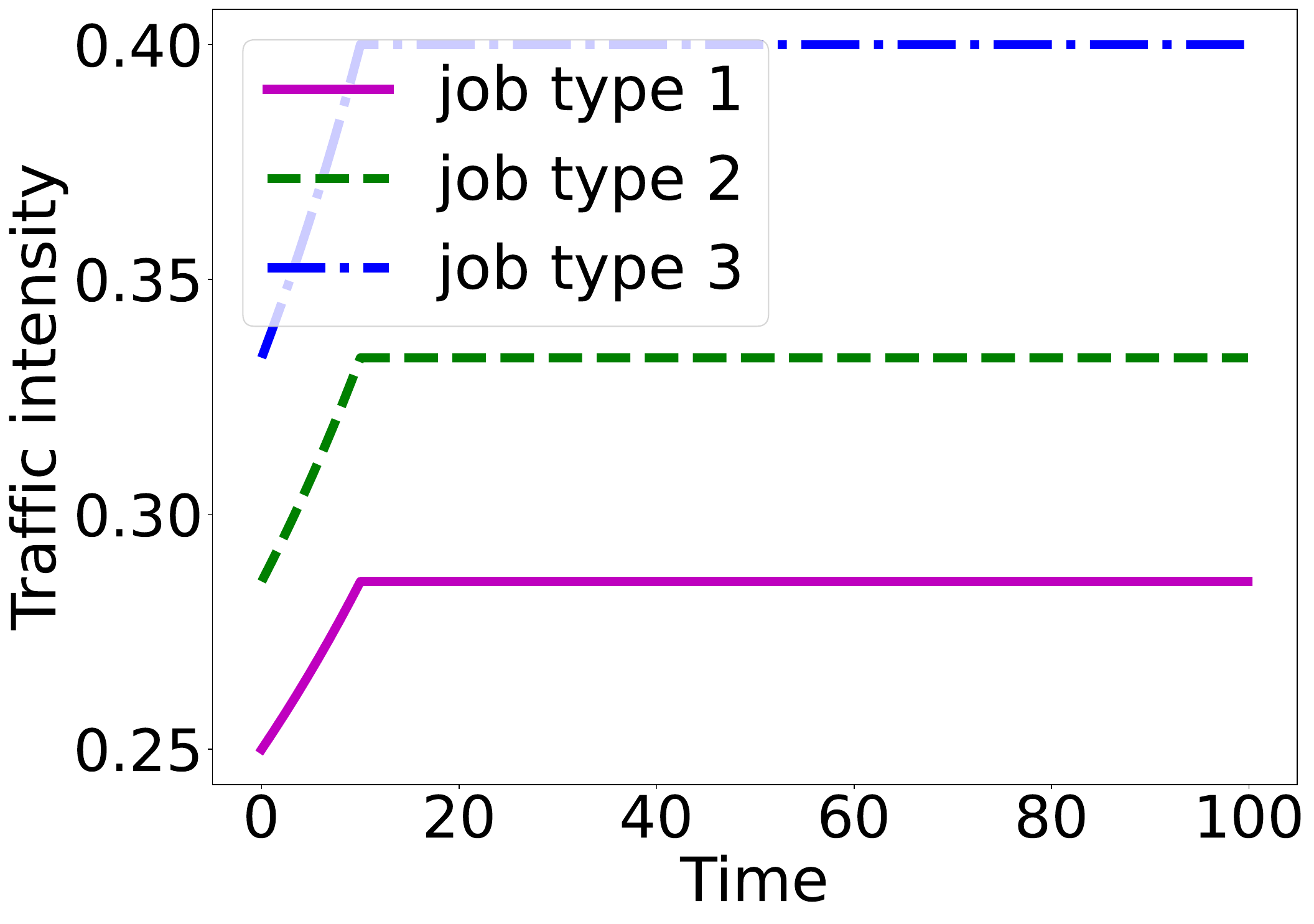}
\caption{Traffic intensities $\rho(t) = \frac{\lambda(t)}{\mu(t)}$   representing arrival and service rates for each class.} 
\end{subfigure}
\hfill
\begin{subfigure}[t]{0.45\textwidth}
\centering
\includegraphics[height=0.21\textheight]{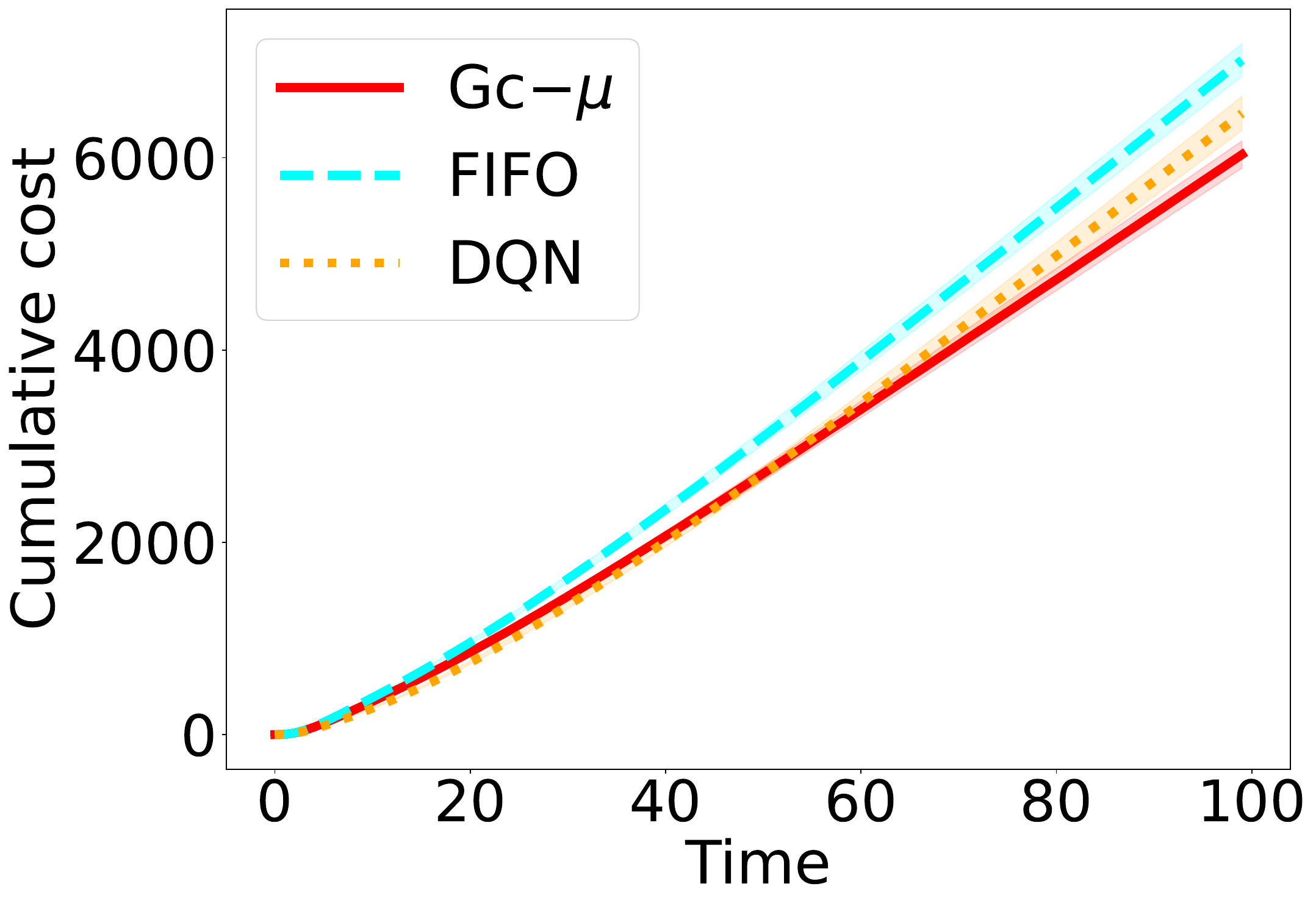}
\caption{Average cumulative costs. Shaded regions denote $\pm 1.96~\times$ standard deviations.} 
\end{subfigure}
\caption{Performance of  policies on data under distribution shift~\ref{item:queue-shift-3}}
\label{fig:queue-shift-3}
\end{figure}

\begin{figure}[H]
\centering
\begin{subfigure}[t]{0.442\textwidth}
\centering
\includegraphics[height=0.21\textheight]{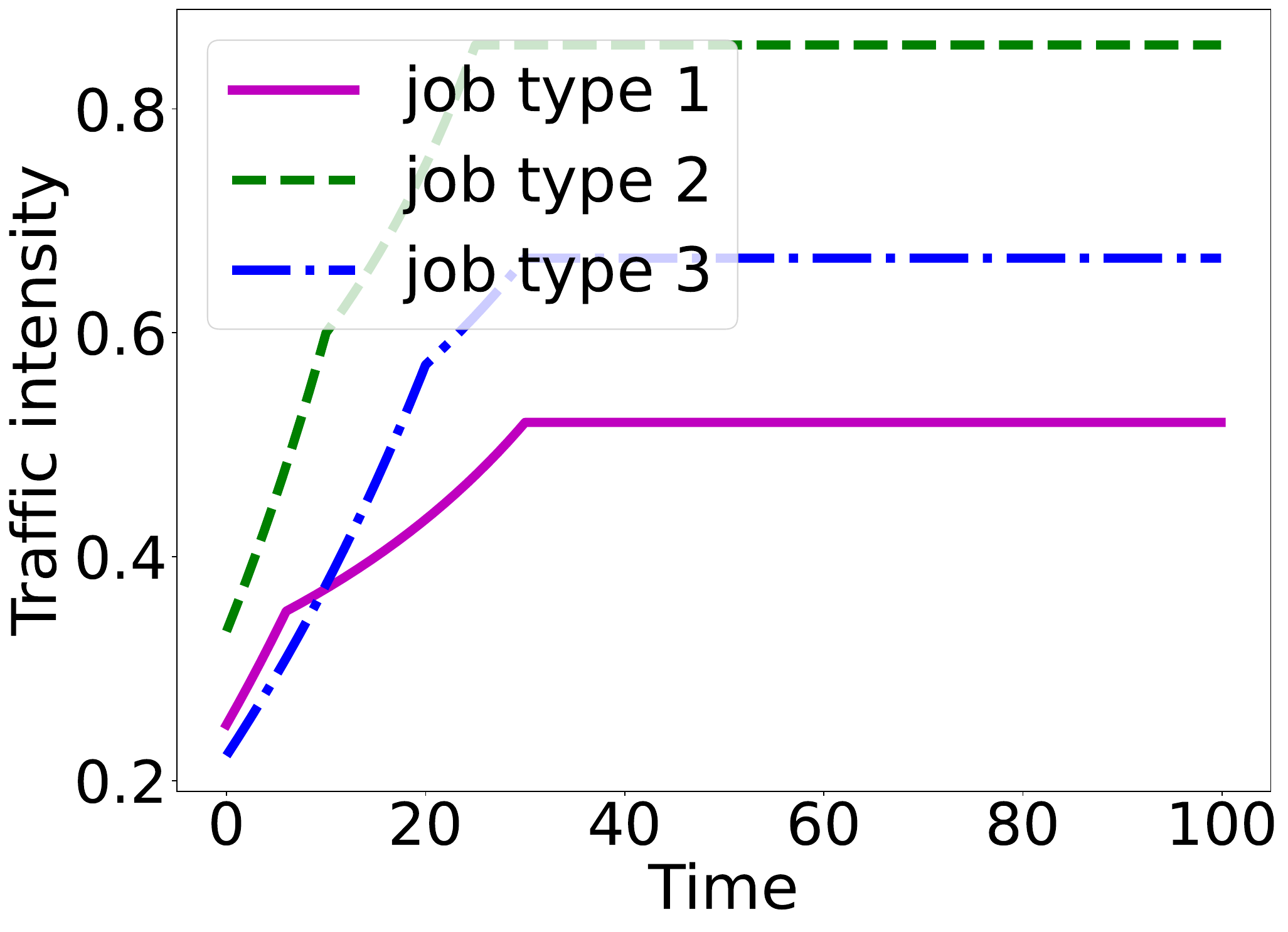}
\caption{Traffic intensities $\rho(t) = \frac{\lambda(t)}{\mu(t)}$   representing arrival and service rates for each class.} 
\end{subfigure}
\hfill
\begin{subfigure}[t]{0.45\textwidth}
\centering
\includegraphics[height=0.21\textheight]{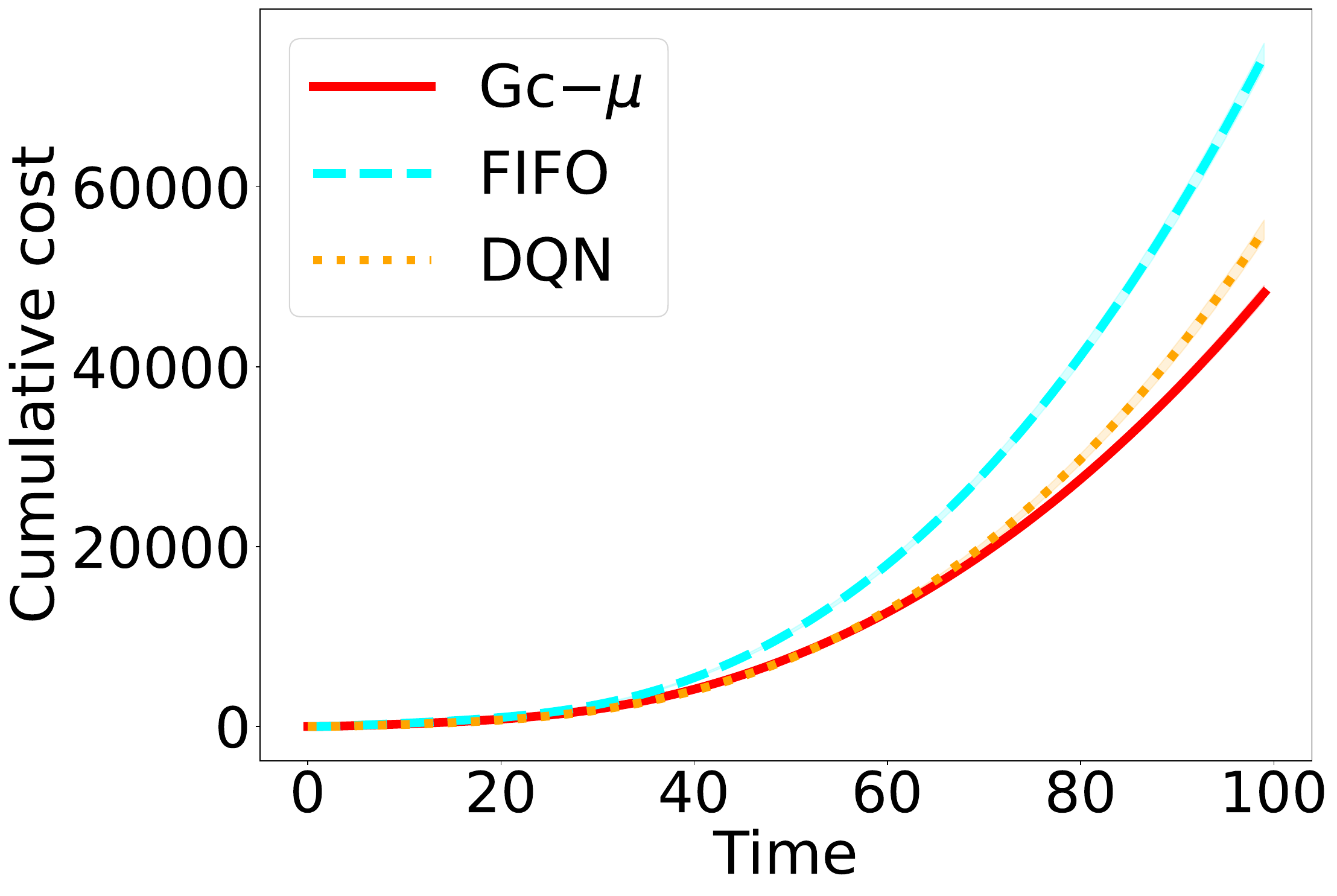}
\caption{Average cumulative costs. Shaded regions denote $\pm 1.96~\times$ standard deviations.} 
\end{subfigure}
\caption{Performance of  policies on data under distribution shift~\ref{item:queue-shift-4}}
\label{fig:queue-shift-4}
\end{figure}

\begin{figure}[H]
\centering
\begin{subfigure}[t]{0.442\textwidth}
\centering
\includegraphics[height=0.21\textheight]{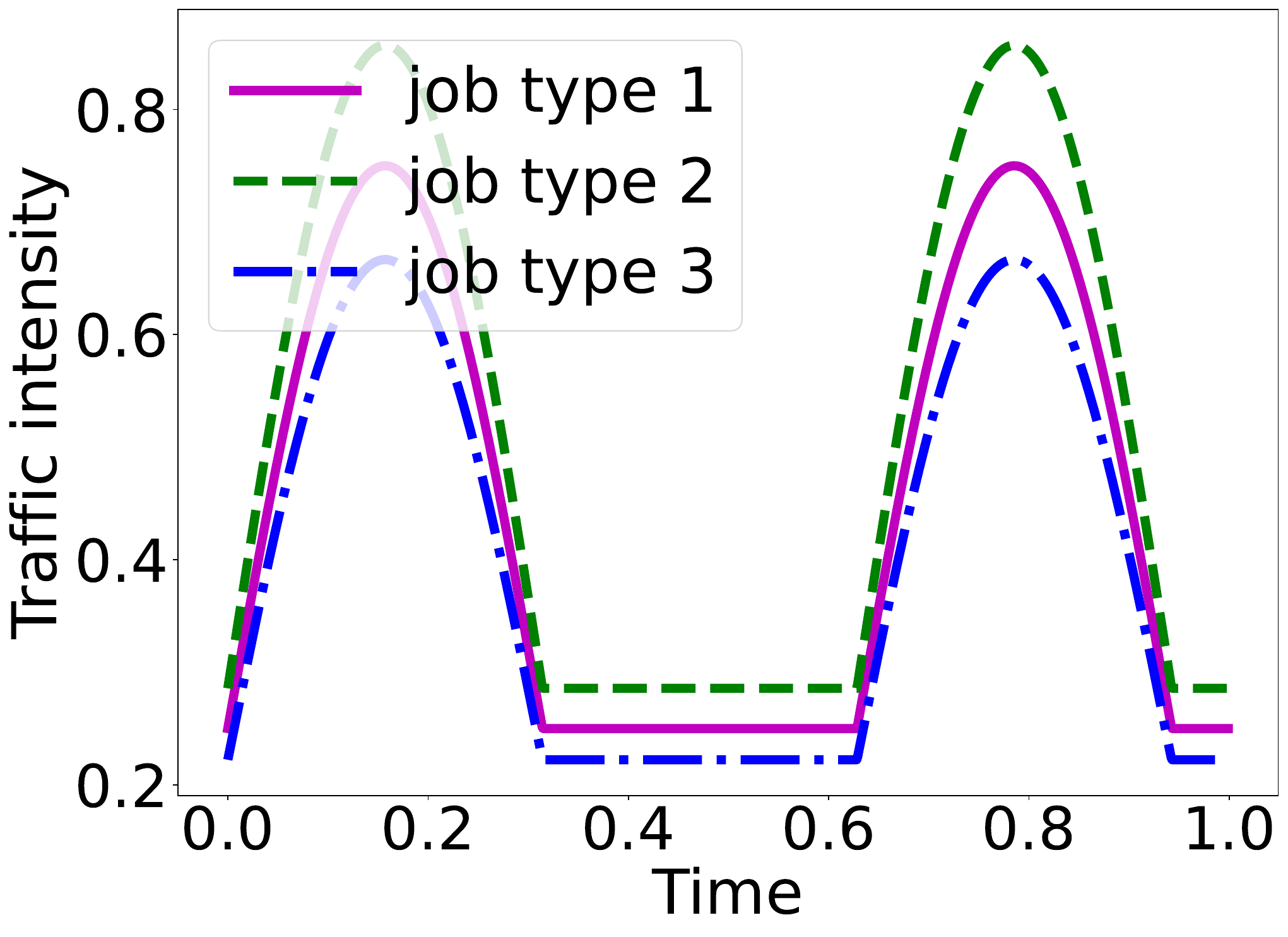}
\caption{Traffic intensities $\rho(t) = \frac{\lambda(t)}{\mu(t)}$   representing arrival and service rates for each class.} 
\end{subfigure}
\hfill
\begin{subfigure}[t]{0.45\textwidth}
\centering
\includegraphics[height=0.21\textheight]{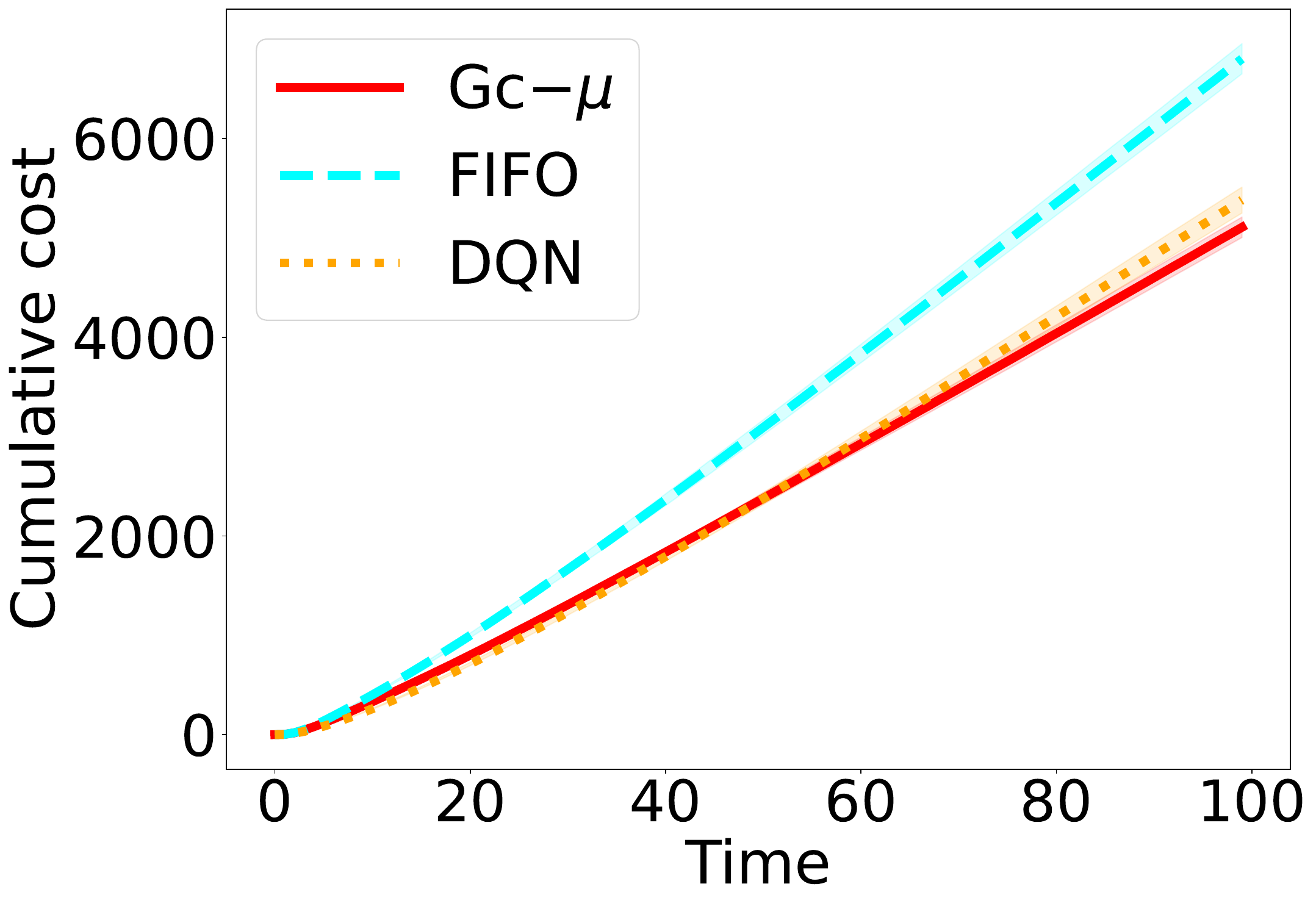}
\caption{Average cumulative costs. Shaded regions denote $\pm 1.96~\times$ standard deviations.} 
\end{subfigure}
\caption{Performance of  policies on data under distribution shift~\ref{item:queue-shift-5}}
{\footnotesize Since we consider cyclical arrival rates in distribution shift~\ref{item:queue-shift-5}, 
we only plot $\rho(t)$ for $t \in [0,1]$. \par}
\label{fig:queue-shift-5}
\end{figure}

To make a fair comparison, we assume that the Gc-$\mu$ rule does not know that
the underlying system has changed and still uses the (possibly incorrect)
original service rate to make decisions. In
Figures~\ref{fig:queue-shift-1}-\ref{fig:queue-shift-5}, we present
performance over 100,000 sample paths on each distribution shift. As suggested
by our stability metric, we observe that the simple index-based policy
exhibits significant robustness benefits over the engineered DQN approach.
\highlight{Our empirical analysis also shows that the stability metric is no
  panacea. When the real distribution shift in question is not aligned with
  the worst-case distribution in the stability measure~\eqref{eqn:primal}, we
  observe the stability measure may not be so informative in informing
  performance differentials.}

\highlight{
\begin{figure}[H]
\centering
\begin{minipage}{.44\textwidth}
  \centering 
\includegraphics[height=0.18\textheight]{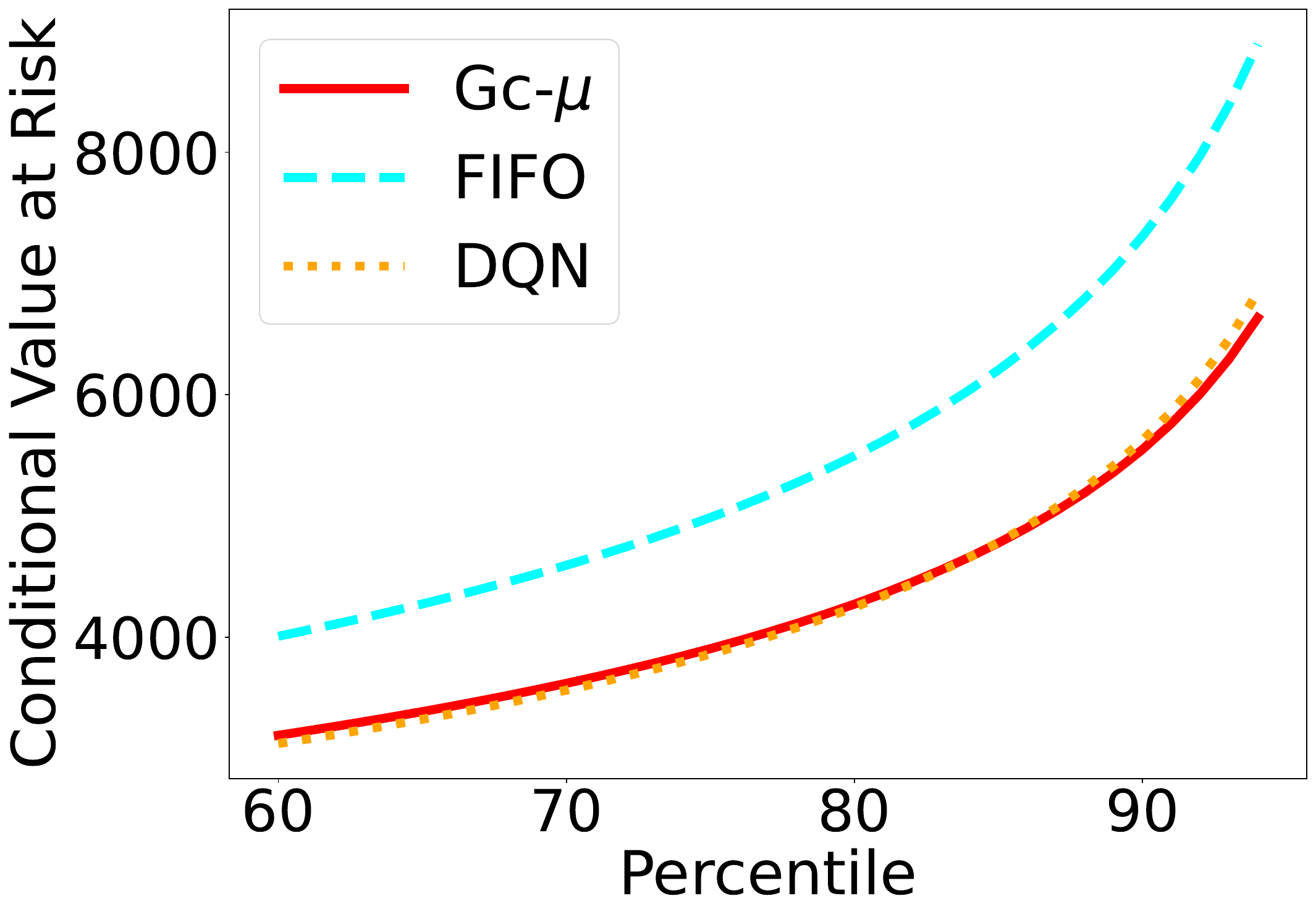}
\caption{CVaR with quantile level $\alpha$ for different policies}  
\label{fig:queue-cvar-plot}
\end{minipage}%
\begin{minipage}{.44\textwidth}
  \centering 
\includegraphics[height=0.18\textheight]{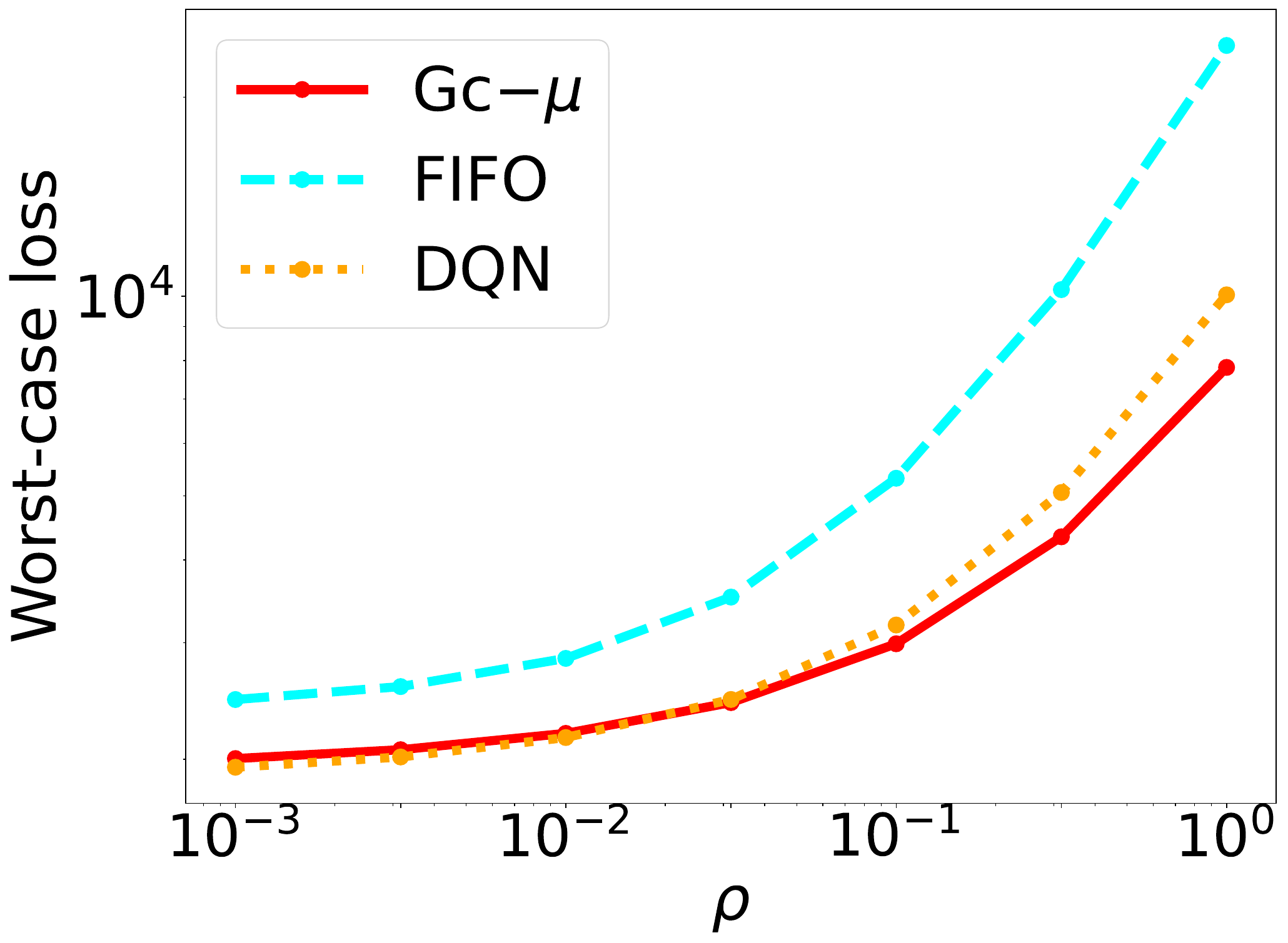}
\caption{Worst-case expectation over a KL neighborhood of radius $\rho$ for
  different policies}
\label{fig:queue-worst-case-expectation-plot}
\end{minipage}%
\end{figure}

Finally, we compare our stability measure against other common
coherent risk measures. We consider distributionally robust
losses---equivalently, coherent risk measures~\cite[Chapter
6]{ShapiroDeRu09}---over the ambiguity set $\mc{P}$
\begin{align*}
  \sup_{Q \in \mc{P}} \E_{Q}[\rv] ~~\mbox{where}~~
  & \mc{P}^{\rm kl}_{\rho} = \{Q: \dkl{Q}{P} \le \rho\}~\mbox{or}\\
  &  \mc{P}^{\rm cvar}_{\alpha} = \set{Q: P = a Q + (1-a) Q' ~~\mbox{for some probability}~Q'~\mbox{and}~ 1 \ge a \ge \alpha}.
\end{align*}
In particular, $\mc{P}^{\rm cvar}_{\alpha}$ gives the popular Conditional
Value-at-Risk (CVaR) parameterized by the percentile $\alpha$.  As we see in
Figure~\ref{fig:queue-cvar-plot}, compared to
Figure~\ref{fig:queue-stability-plot}, our stability measure is better at
differentiating the risk for different policies while CVaR outputs different
rankings in risk with different choices of $\alpha$.  Similarly, in
Figure~\ref{fig:queue-worst-case-expectation-plot}, we estimate the worst-case
expectation under $\mc{P}^{\rm kl}_{\rho}$ using the dual
reformulation~\eqref{eqn:worst-case-exp-kl-duality} and we observe that this
metric applying on the loss gives different robustness rankings for different
choices of $\rho$.  }

\subsection{Health utilization prediction}
\label{section:nhis}
 
As our second example, we analyze a prediction setting where we are interested
in whether an individual is utilizing healthcare services based on a rich set
of demographic features and medical information. Individual-level predictions
can be used to inform a variety of critical operational decisions. Allowing
low-income individuals to utilize healthcare resources is a major policy
goal~\cite{CurrieGr96, LiptonDe15, CurrieFa05, KahendeMaEnZhMoXuSeRo17}, and
individual-level predictions based on reliable ML models can inform targeted
policy-making~\cite{KleinbergLuMuOb15,Athey17}.  Insurance firms can also
decide whether or not to enter a new market based on estimates of the new
market's average utilization~\cite{WeiRaVa15}.

Using the National Health Interview Survey (NHIS), we collate a supervised
dataset focusing on low-income individuals. We are interested in a binary
outcome $Y$ indicating whether the individual visited the doctor's office in
the two weeks before the survey date. We use a rich set of covariates
$X \in \R^d$ ($d=396$) consisting of demographics, earnings, enrollment in
government insurance programs (e.g., Medicaid), medical history, and
pre-existing conditions. We take the viewpoint of an analyst in 2015 who
wishes to train and deploy a prediction model to be used in subsequent
years. The NHIS data contains temporal data up to 2018, allowing us to
simulate distribution shifts occurring in ``future years.''  \emph{We use this
  example to illustrate that stability is a useful complementary measure to
  the usual average-case prediction evaluations, and highlight the
  limitations of our approach.}

The distribution of demographic and socioeconomic factors can significantly
change over time, leading to large covariate shifts~\cite{JeongNa20}.  To
illustrate the flexibility of our stability approach, we consider a scenario
where distribution shift is a major concern only for a subset $Z$ of the
covariates $X$. In contrast to our original framework defined over changes in
the joint distribution of $(X, Y)$, we define the stability measure only over
marginal shifts in $Z$, assuming $(X, Y) | Z$ is fixed.  To start, we focus on
the following list of core variables: gender, years of education, number of
working hours, number of overnight hospital visits, health status, no. home
visits by health professionals in the past 2 weeks, whether the main reason
for not working last week is due to health, whether the individual received
care 10+ times in the past 12 months, months worked last year, limitation
regarding all conditions, whether the individual has no coverage of any type,
family's spending on medical care, months without coverage and number of
nights in the hospital in the past 12 months, and reported health status.  See
Table~\ref{table:NHIS-core-variable-new} in for a detailed description of
these variables.

Formally, let $f(X)$ be a classifier and let $\loss(f(X); Y)$ denote the 0-1
loss. We define our stability measure over the conditional risk
\begin{equation}
  R(Z;f) := \E[\ell(f(X);Y)|Z], \label{eqn:cond-risk}
\end{equation}
considering shifts in the marginal distribution of $Z$. Since $R(Z; f)$ is
never observed, we estimate it using an \emph{auxiliary} black-box machine
learning model for each candidate prediction model $f(\cdot)$ we wish to
evaluate. Then, we compute our stability estimator $\what{I}_n$ based on the
approximate data $\{R(Z_i; f)\}_{i=1}^n$.  We evaluate the stability of
several different prediction modeling frameworks (model classes $\mc{F}$),
with a particular focus on tree-based ensemble methods~\cite{HastieTiFr09} as
they provide state-of-the-art performance on tabular
data~\cite{GrinsztajnOyVa22}. We train logistic regression and random forest
classifiers using the $\mathsf{scikitlearn}$ package and train a gradient-boosted decision tree using the $\mathsf{LightGBM}$ package; for all model
classes, we use $0.5$ as the classification threshold and use standard
five-fold cross-validation to select hyperparameters. See
Section~\ref{section:nhis-table} for implementation details.

\begin{figure}[t]
\centering  
\includegraphics[height=0.18\textheight]{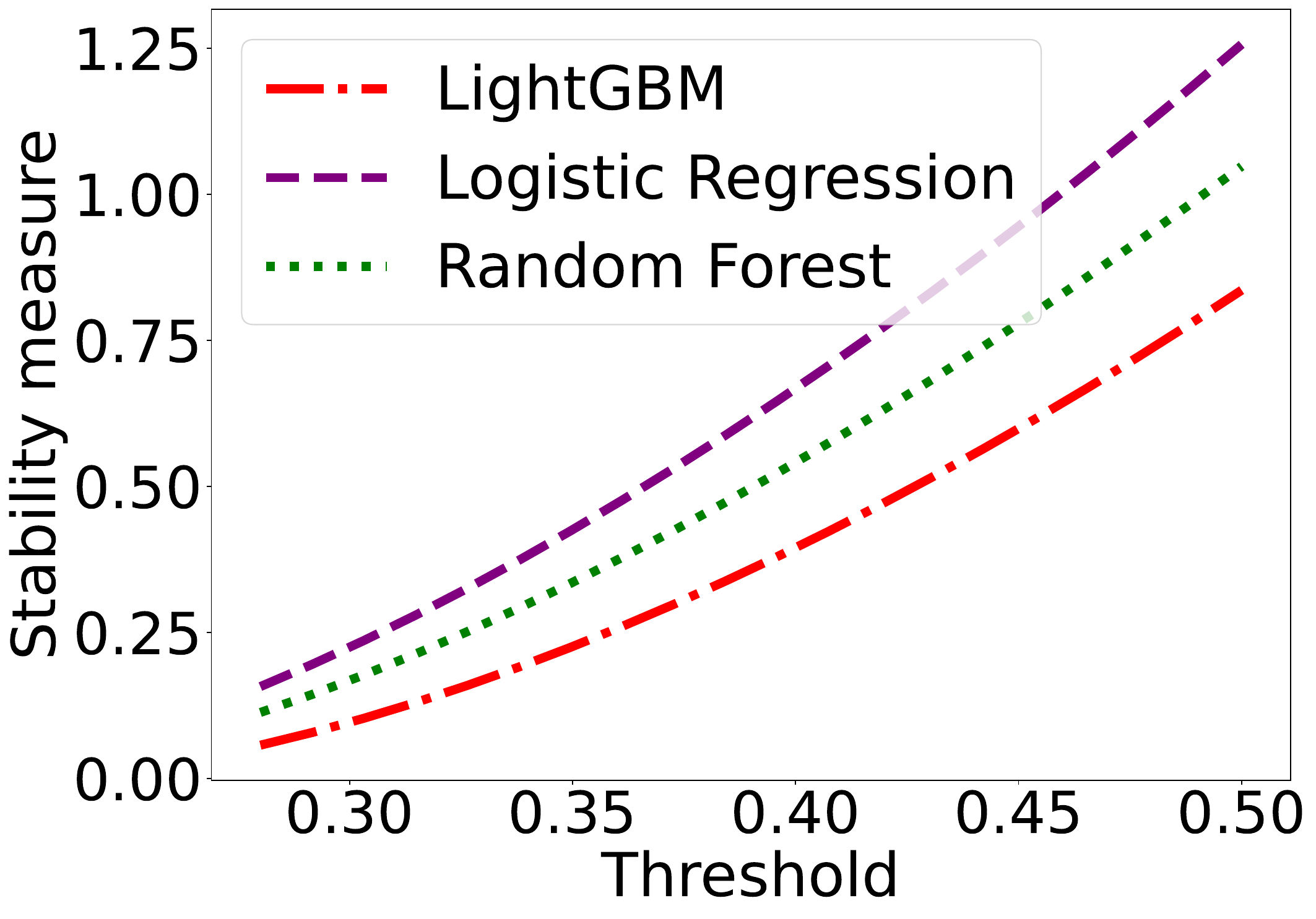}
\caption{Stability $\what{I}_n$ for different models over a set of thresholds}
\label{fig:nhis-stability-plot-new-Z}
\end{figure}

We see that models perform similarly according to average test error:
LightGBM: 14.87\%, Logistic regression: 15.03\%, Random forest: 15.29\%.  The
LightGBM model performs marginally better than the other two models. Going
beyond this standard average-case comparison, we evaluate the stability of
each model under distribution shift, i.e., assess how much the distribution of
$R(Z;f)$ must change until the error rate degrades above the threshold
$y$. For simplicity, we focus on a threshold value $y$ that is approximately
thrice the test error of LightGBM; this again unfairly favors the LightGBM
model, which has the best average error rate. In a stark reversal compared to
the average-case evaluations, we observe that the stability of the LightGBM
model is significantly worse than that of the other two models.  Indeed, in
Figure~\ref{fig:nhis-stability-plot-new-Z}, we observe that LightGBM remains
the least stable model for different thresholds.  In contrast, if this plot
outputs that LightGBM is only stable for certain thresholds, the analyst can
document this information, and stakeholders can take it into account for
decision-making.  Even if the analyst has an exact threshold for model
comparison, computing stability around this threshold and making a plot
similar to Figure~\ref{fig:nhis-stability-plot-new-Z} can provide further context.
\highlight{
Similar to the queueing example, we compare the stability measure with other risk measures
in Section~\ref{section:nhis-table}.
}

\begin{figure}[t]
\begin{subfigure}[t]{0.48\textwidth}
\centering
\includegraphics[height=0.21\textheight]{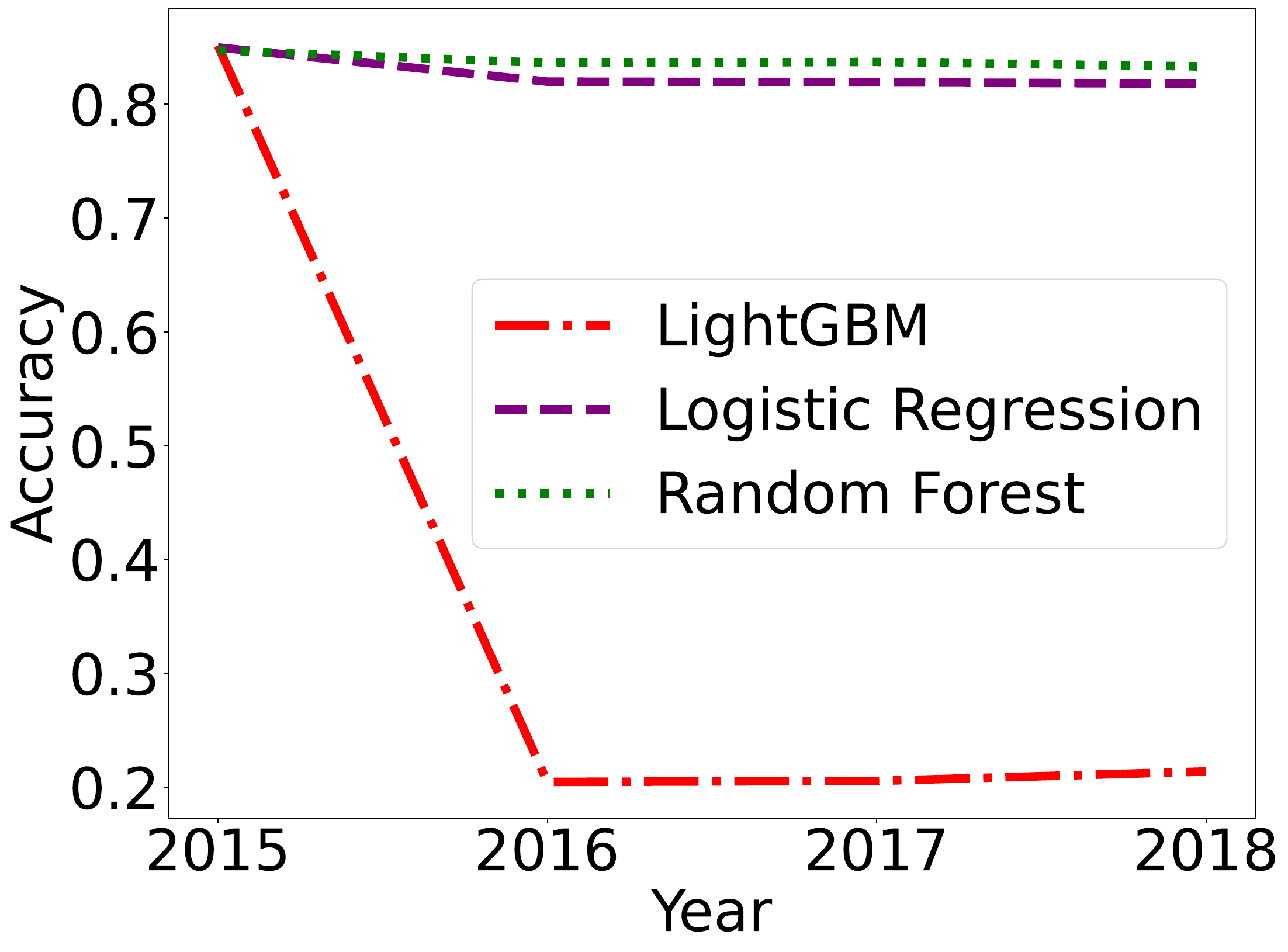}
\caption{Accuracy across time} 
\end{subfigure}  
\begin{subfigure}[t]{0.48\textwidth}
\centering
\includegraphics[height=0.21\textheight]{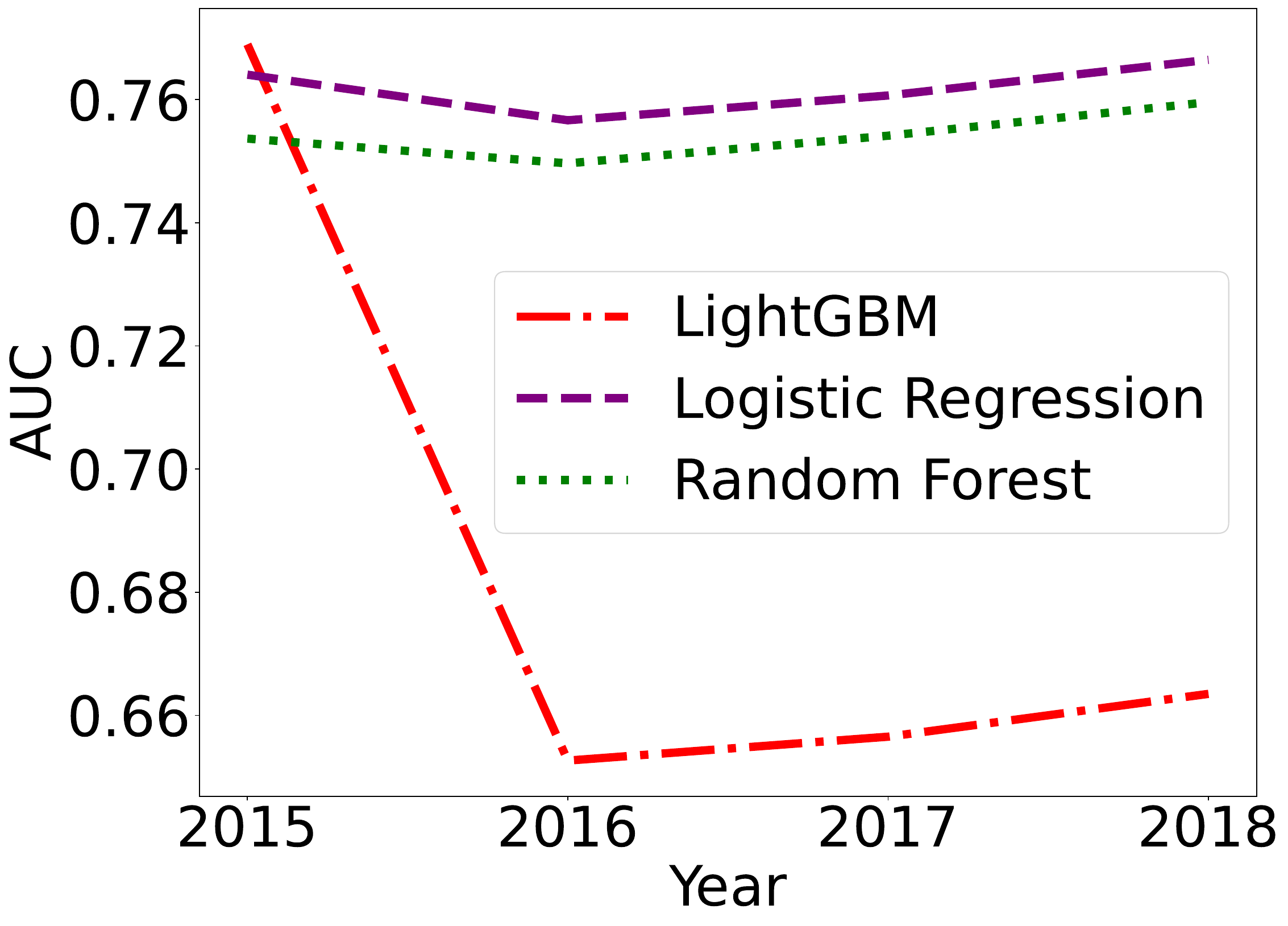}
\caption{AUC across time} 
\end{subfigure}
\caption{Model performance in ``future'' years}
\label{fig:NHIS-auc-and-accuracy}
\end{figure}

\begin{figure}[H]
\centering
\begin{subfigure}[b]{0.6\textwidth}
\centering
\includegraphics[height=0.25\textheight]{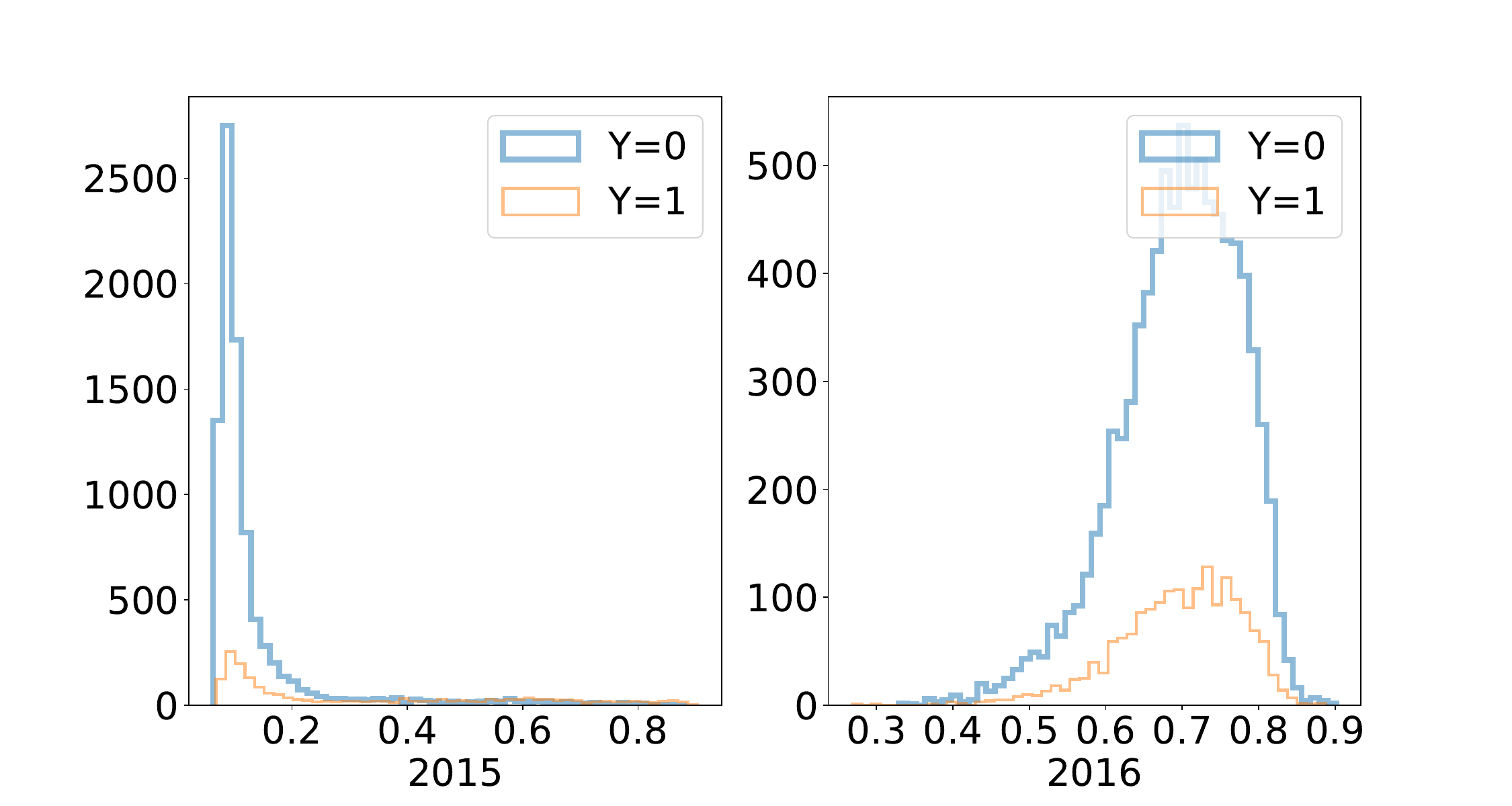}
\caption{LightGBM} 
\end{subfigure}
\hfill
\begin{subfigure}[b]{0.6\textwidth}
\centering
\includegraphics[height=0.25\textheight]{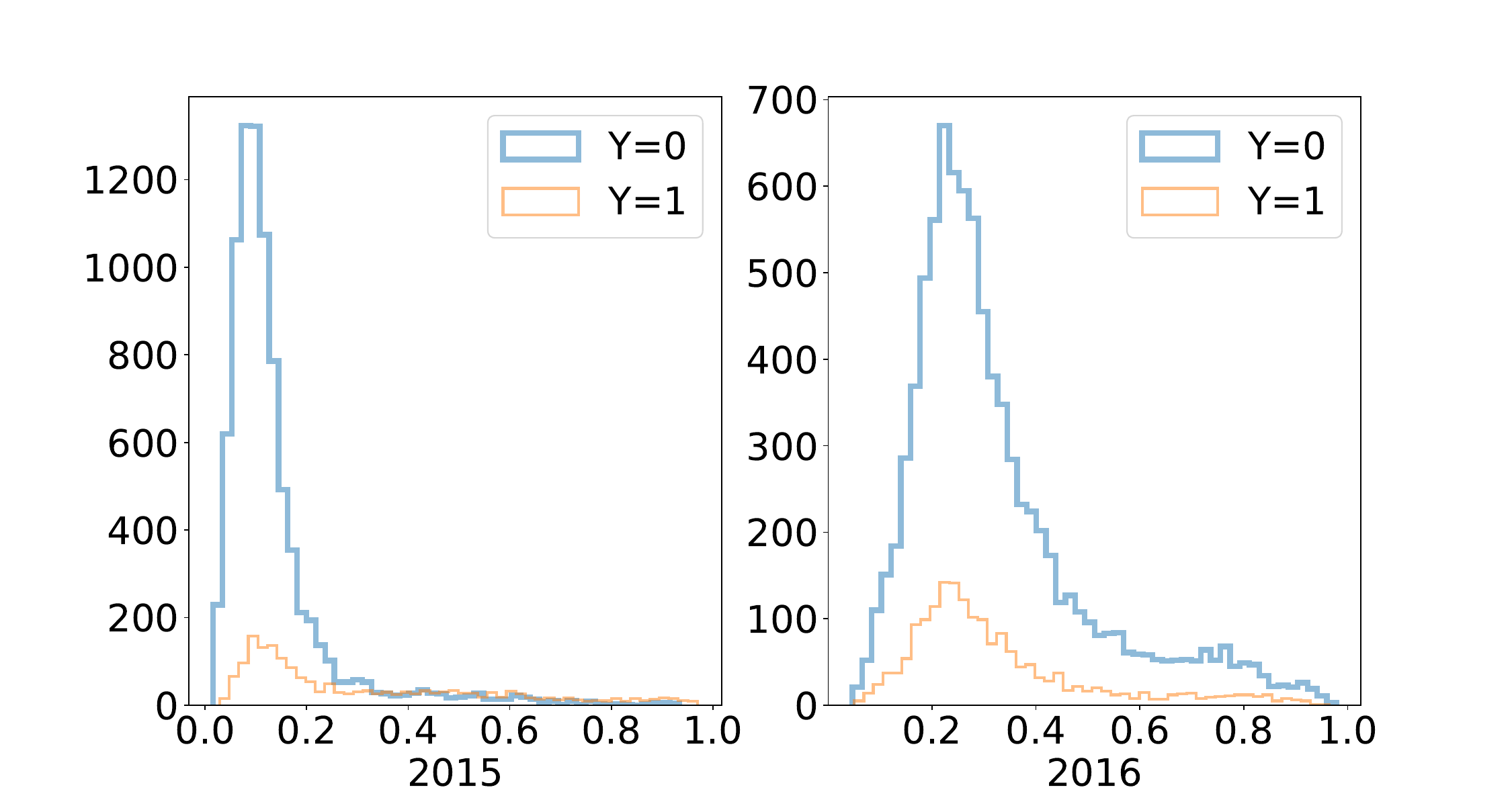}
\caption{Logistic Regression} 
\end{subfigure}
\hfill
\begin{subfigure}[b]{0.6\textwidth}
\centering
\includegraphics[height=0.25\textheight]{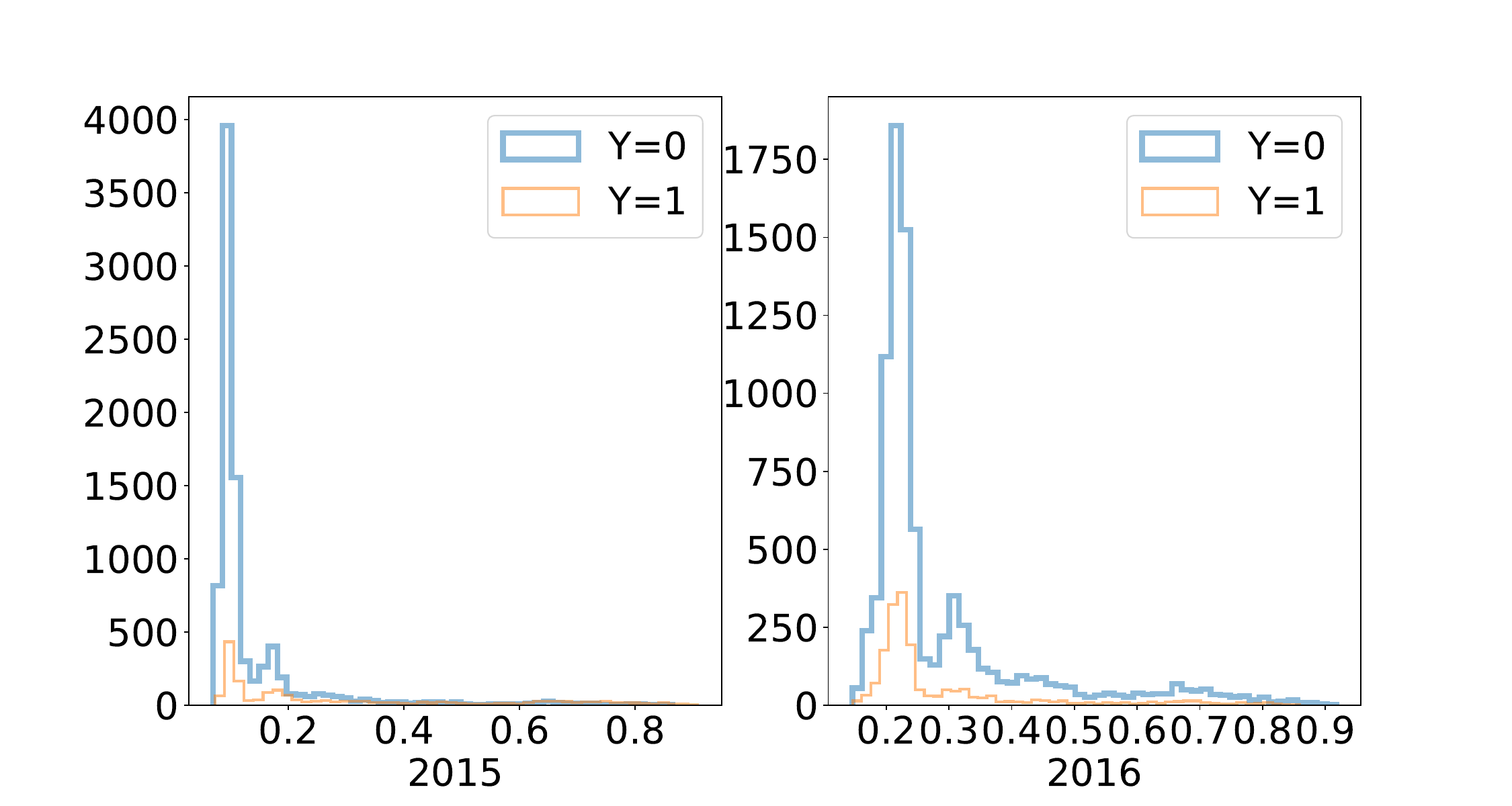}
\caption{Random Forest} 
\end{subfigure}
\caption{Histogram of calibrated probabilities for three models across time}
\label{fig:NHIS-predicted-probabilities}
\end{figure}

To complement our stability evaluations based on worst-case distribution
shifts, we evaluate the performance of each prediction model on ``future
data.''  In Figure~\ref{fig:NHIS-auc-and-accuracy}a), we observe that while
the performance of the random forest and logistic regression models remain
stable, the performance of the LightGBM model substantially degrades over
time. The large performance gap is particularly unexpected given
near-identical test accuracy on 2015 data. To assess whether the issue is
simply due to the classification threshold---recall that we had simply set it
to 0.5 for all models---we evaluate models using the AUC in
Figure~\ref{fig:NHIS-auc-and-accuracy}b) and find similar trends: although
adjusting classification thresholds can mitigate instability, the LightGBM
model still performs poorly under threshold-invariant AUC evaluations.  To
understand the severe performance degradation of the LightGBM model, we
perform further qualitative analysis. We find that the distribution of
(calibrated) classification probabilities changes drastically for the LightGBM
model across different years in Figure~\ref{fig:NHIS-predicted-probabilities},
whereas the shift for the two other models is milder. (We see similar trends
persisting across other years.) We observe a large distribution shift in the
number of home visits by healthcare professionals over time, which the
LightGBM model overly relies on for prediction.

\begin{figure}[t]
\centering  
\includegraphics[height=0.21\textheight]{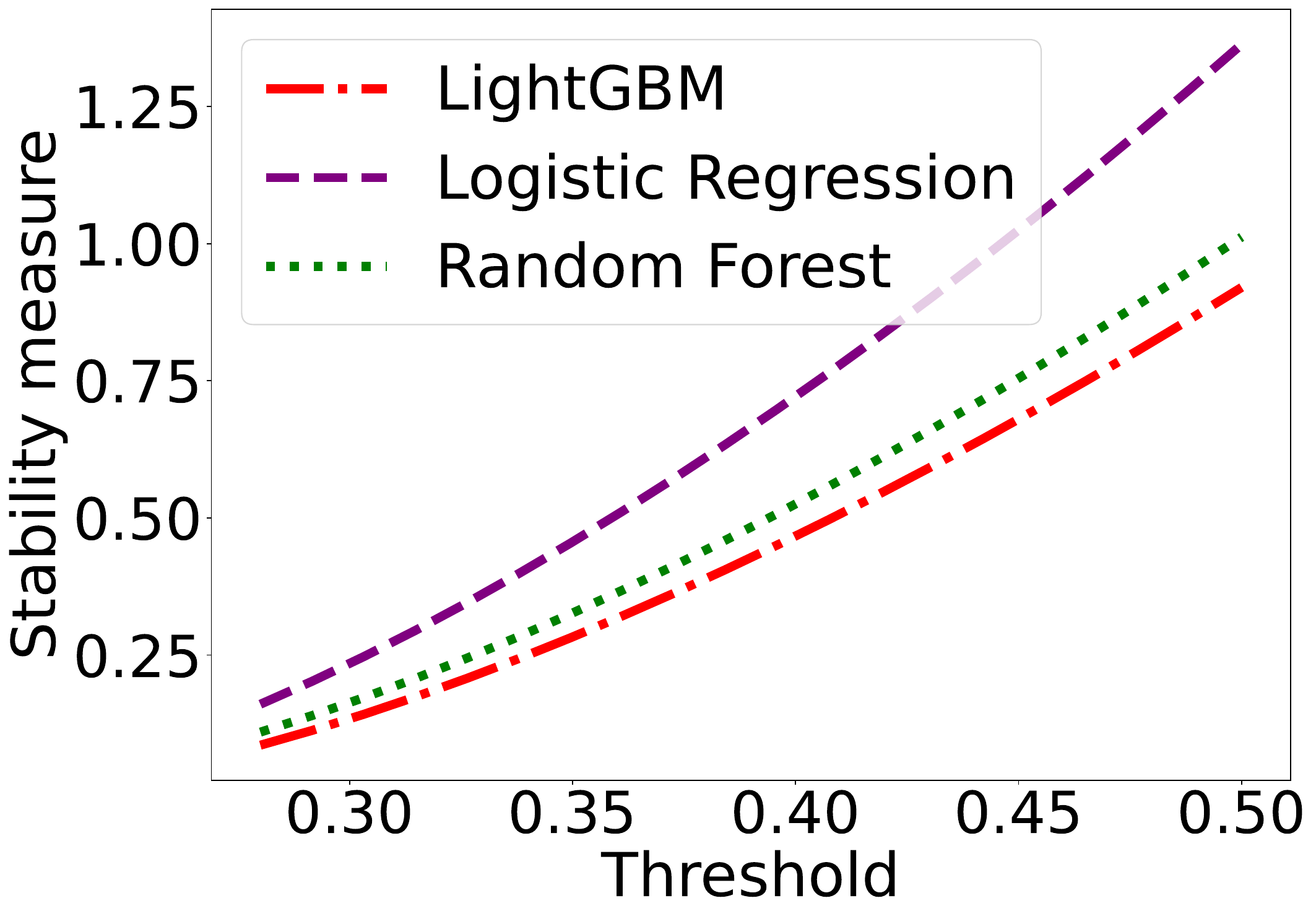}
\caption{Stability $\what{I}_n$ for different models over a set of thresholds}  
\label{fig:nhis-stability-plot}
\end{figure}

Overall, our stability analysis allows surfacing the flaws of the LightGBM
model; it is difficult for the analyst to realize these flaws merely using the
average accuracy and AUC on held-out test data from 2015.  \highlight{ 
Our empirical analysis also highlights the limitations of our approach.  
  We  carry out the same procedure to compute the stability measure described in
  Algorithm~\ref{algorithm:nhis} but with a different set of core variables
  $Z$.  For concreteness, we focus on the following list of ten core
  variables: gender, Medicaid enrollment status, years of education, health
  status, number of working hours, number of overnight hospital visits, number
  of home visits by healthcare professionals, employment, and retirement
  status. In Figure~\ref{fig:nhis-stability-plot}, we observe that although
  the stability measure indicates logistic regression is more stable than the
  random forest classifier, both perform similarly in future years. This
  suggests that stability measure with these new variables $Z$ may be too
  conservative for such fine-grained comparisons. A principled selection
  mechanism for $Z$ is a promising future research topic.}


\section{Discussion}
\label{section:discussion}

We study a stability framework for assessing system performance against
distribution shifts. Our approach is interpretable as it is parameterized by a
level of permissible performance degradation in the cost scale and allows
relative comparisons between different system designs as we detail in
Section~\ref{section:experiment}. In contrast, a principled choice of the
magnitude of distribution shift is a longstanding open problem in the
distributional robustness literature; this is often the most important
modeling choice that governs subsequent performance evaluations. As humans are
typically unable to reason on the scale of distribution shifts, our framework
circumvents this problem by considering a dual-like
quantity~\eqref{eqn:primal} defined with respect to the performance threshold
$y$.

Theoretically, the finite sample minimax rates we present in
Sections~\ref{section:convergence} and~\ref{section:hardness} are one of the
first hardness results in estimating system performance under distribution
shifts.~\citet{DuchiNa21, LevyCaDuSi20} are the only prior works we know of,
that showed minimax rates for distributionally robust optimization. Given the
connections to large deviations theory delineated in
Section~\ref{section:approach}, we hope the analytic tools developed in this
work can spur further formal studies on the statistical estimation of rare
event probabilities and large deviation rates~\cite{DemboZe98, DeuschelSt89,
  AsmussenGl07, Ellis07}.

We focus on the formulation~\eqref{eqn:primal} considering all distribution
shifts, but this may often be overly conservative in practice. In many
application scenarios, the modeler is concerned with a particular type of
distribution shift as discussed in Section~\ref{section:nhis}. In prediction,
spatial and subpopulation shifts often lead to changes in certain demographic
compositions~\cite{DuchiHaNa22, JeongNa22, KohSaEtAl20}. In causal inference,
one is typically concerned with unobserved confounders that impact changes in
the distribution of $R \mid X$, where $X$ is a set of observed
covariates~\cite{Rosenbaum02, Tan06, KallusZh18}.  A stability framework over
a structured set of distribution shifts---constructed using domain
knowledge---is a fruitful future direction.

In many applications, the performance of a system must be evaluated across
multiple dimensions. In drug development, both efficacy and potency are
crucial. For online platforms, consumer satisfaction and long-term revenue are
natural dual objectives.  In each case, stability in one dimension may not
imply stability in the other. Any robustness evaluation must take into
account potential trade-offs, which requires modeling correlation
structures. Developing an understanding of stability in such multidimensional
settings is an important future research direction.




\bibliographystyle{abbrvnat}

\ifdefined\useorstyle
\setlength{\bibsep}{.0em}
\else
\setlength{\bibsep}{.7em}
\fi
\bibliography{bib}

\ifdefined\useorstyle

\ECSwitch


\ECHead{Appendix}

\else
\newpage
\appendix
\fi

\section{Duality results}

\subsection{Proof of Lemma~\ref{lemma:duality-result}}
\label{section:proof-of-duality}
Since $y \ge \E_P[R]$, we have $\lambda\opt \ge 0$ by Jensen's inequality. To prove the first result of the lemma, it suffices to show that
\begin{align*}
 I_y(P) = \sup_{\lambda \ge 0} \set{\lambda  y - \log \E_P[e^{\lambda \rv}]}.
\end{align*} 
Consider the likelihood ratio formulation~\eqref{eqn:likelihood}. 
Introducing Lagrange multipliers
$\lambda \ge 0$ for the constraint $\E_P[L \rv] \ge y$ and $\eta$ for the
constraint $\E_P[L] = 1$, the Lagrangian is given by
\begin{equation*}
  \mc{L}(L, \lambda, \eta) \defeq
  \E_P[L \log L] - \lambda (\E_P [L \rv] - y) - \eta (\E_P[L] - 1).
\end{equation*}

Taking $L(\rv) = \frac{e^{\mu \rv}}{\E_P[e^{\mu \rv}]}$ for $\mu$ large enough so that $\E_P[LR] > y$, extended Slater's condition holds.  Standard functional
duality theory (e.g.,~\citet[Theorem 8.6.1, Problem
8.7]{Luenberger69},~\citet[Theorem 2.165]{BonnansSh00}) shows that there is no
duality gap. We hence arrive at
\begin{align*}
  & \inf_{L \ge 0}  \left\{ \E_P[L \log L]: \E_P [L \rv] \geq y,
  \E_P[L] = 1 \right\} \\
  & = \sup_{\lambda \ge 0, \eta \in \R} \inf_{L \ge 0}
    \left\{ \E_P[L \log L] - \lambda (\E_P [L \rv] - y) - \eta (\E_P[L] - 1)
    \right\},
\end{align*}
where the infimums are taken over measurable functions.

We wish to interchange the inner infimum and the integral. For any fixed
$\lambda \ge 0$ and $\eta \in \R$,
\begin{align}
  & \inf_{L \ge 0}
  \left\{ \E_P[L \log L] - \lambda (\E_P [L \rv] - y) - \eta (\E_P[L] - 1)
  \right\} \noindent \nonumber \\
  & \ge  \E_P\left[ \inf_{l \ge 0}  \left\{
    l \log l - \lambda \rv l  - \eta l \right\} \right]
    + \lambda y  + \eta. \label{eqn:interchange-inf-int}
\end{align}
Using the first-order optimality condition for the infimum over $l \ge 0$,
we see that the minimum is obtained at
\begin{equation}
  \label{eqn:primal-opt}
  l\opt = \exp(\lambda \rv + \eta - 1).
\end{equation}
Since $L\opt \defeq\exp(\lambda \rv + \eta - 1)$ is clearly measurable, we conclude
equality holds in place of the inequality~\eqref{eqn:interchange-inf-int}.

Plugging  the optimal solution~\eqref{eqn:primal-opt} in the
problem~\eqref{eqn:interchange-inf-int}, the dual simplifies to
\begin{equation*}
  \sup_{\lambda \ge 0, \eta \in \R}
  \left\{ \lambda y  + \eta - e^{\eta -1} \E_P\left[ e^{\lambda \rv} \right]
  \right\}.
\end{equation*}
Noting that we can restrict attention to $\lambda$ such that
$\E_P[e^{\lambda \rv}] < \infty$, directly supremizing over $\eta$ gives
$\eta\opt = 1 - \log \E_P[e^{\lambda \rv}]$, which yields the first result.  The
second result follows from strong duality and the
characterization~\eqref{eqn:primal-opt}.

\subsection{Duality results for other metrics}
\label{sec:duality-result-other-metrics}
\highlight{We first extend the duality result to $f$-divergence.
  Let $f: \R \to \R_+ \cup \set{\infty}$ be a convex function satisfying
  $f(1) = 0$ and $f(t) = \infty$ for any $t < 0$. We define the
  $f$-divergence~\cite{AliSi66,Csiszar67} between $Q$ and $P$ as
\begin{align*}
    \fdiv{Q}{P} \defeq \int f\paran{\frac{dQ}{dP}} dP.
\end{align*}
The $f$-divergence is quite general and it covers the following commonly used
metrics:
\begin{enumerate}
\item KL divergence: $f(x) = x \log x - x +1$
\item Total variation: $f(x) = \half |x-1|$
\item $\chi^2$-divergence: $f(x) = (x-1)^2$
\end{enumerate}
We refer the reader to~\cite[Table 1]{RahimianMe19} for more examples of divergence metrics.

The following proposition extends the duality result for KL divergene through the Fenchel conjugate.
\begin{prop}
Let  $f^*$ be the Fenchel conjugate of $f$: $f^*(s) := \sup_t \set{s t - f(t)}$. 
For every $y$ with $\E_P[\rv] \le y < \mathrm{ess~sup}~\rv$,
  \begin{equation*}
       \inf_{Q} \set{\fdiv{Q}{P}:  \E_Q [\rv] \geq y} =   \sup_{\lambda \ge 0, \eta \in \R}
  \set{ \lambda y  + \eta  -  \E_P[ f^*(\lambda \rv + \eta)]}.
  \end{equation*}
  \label{prop:f-div-duality}
\end{prop}

\begin{proof}
Similar to~\eqref{eqn:likelihood}, we have
\begin{align*}
   \inf_{Q} \left\{\fdiv{Q}{P}:  \E_Q [\rv] \geq y\right\}
  = \inf_{L \ge 0}  \left\{ \E_P[f(L)]: \E_P [L \rv] \geq y, \E_P[L] = 1 \right\},
\end{align*}
We follow the same proof as in Lemma~\ref{lemma:duality-result} but replace $L \log L$ with $f(L)$ and $l \log l$ with $f(l)$.
Now,~\eqref{eqn:interchange-inf-int} changes to
\begin{equation*} 
 \lambda y  + \eta  -  \E_P[ f^*(\lambda \rv + \eta)],
\end{equation*} 
the rest is similar
and we complete the proof.
\end{proof}

We can also extend Lemma~\ref{lemma:duality-result} to Wasserstein distances.
Let $\rv \in \mc{X}$, where $\mc{X}$ is a Polish space, and $\mc{P}(\mc{X})$
denote all probability distributions.  
The Wasserstein distance induced by the cost function $c$ is
\begin{equation*}
    W_c(Q,P) \defeq \inf_{\pi \in \Pi(P,Q)} \set{\E_{(\rv, \hat{\rv}) \sim \pi} \Paran{c(\rv, \hat{\rv})}}, ~~ Q,P \in \mc{P}(\mc{X}),
\end{equation*} where $\Pi(P,Q)$ is the set of probability distributions on $\mc{X} \times \mc{X}$ with marginals $P$ and $Q$.
Throughout, we impose the following assumption on the cost function $c$, which is standard in the literature~\cite{BlanchetMu19}.
\begin{assumption}
    The function $c: \mc{X} \times \mc{X} \to [0,\infty)$ is lower semicontinuous and $c(x,x)=0$ for every $x \in \mc{X}$. \label{assumption:c-func}
  \end{assumption}
  We first state a useful technical result.
\begin{lemma} \cite[Proposition 1, Lemma 2, and Example 2]{ZhangYaGa22}
Consider any measurable function $\phi: \mc{X} \times \mc{X} \to \R \cup \set{-\infty}$.
    Then 
    \begin{equation*}
       \E_{\rv \sim P} \Paran{\sup_{x \in \mc{X}} \phi(\rv,x)} = \sup_{\pi \in \Pi_P} \set{\E_{(\rv, \hat{\rv}) \sim \pi} \Paran{\phi(\rv, \hat{\rv})}},
    \end{equation*}
    where $\Pi_P$ is set of probability distributions $\pi$ on $\mc{X} \times \mc{X}$ whose first marginal is $P$ such that $\E_{(\rv, \hat{\rv}) \sim \pi} \Paran{\phi(\rv, \hat{\rv})}$ is well-defined. \label{lemma:duality-change-order}
\end{lemma}

To compute the stability measure defined using the Wasserstein distance,
we present the following result, which reformulates an 
infinite-dimensional  problem to a low-dimensional problem involving computing the term $\psi_\lambda$.

\begin{prop}
Assume assumption~\ref{assumption:c-func} holds. 
Then 
\begin{equation*}
     \inf_{Q} \set{ \wass{Q}{P}:  \E_Q [\rv] \geq y} =  \sup_{\lambda \ge 0} \set{\lambda y -   \E_{P}[\psi_\lambda(\rv)] },
\end{equation*} where
\begin{equation*}
    \psi_\lambda(z) = \sup_{x \in \mc{X}} \set{\lambda x - c(z,x)}.
\end{equation*}
\label{prop:wasserstein-duality}
\end{prop}

\begin{proof}
We use a proof inspired by~\citet{ZhangYaGa22}.
Fix $P$,
let $L(y) = \inf_{Q} \set{ \wass{Q}{P}:  \E_Q [\rv] \geq y}$,   and  define 
\begin{equation*}
G(\lambda) = \inf_{y \in \R} \set{L(y) - \lambda y} = -L^*(\lambda), 
\end{equation*}
where $L^*$ is the Fenchel conjugate of $L$.
Note that $G(\lambda) = -\infty$ for any $\lambda < 0$.
We start by showing that $L(\cdot)$ is a convex function. To see this, we fix $\alpha \in [0,1]$ and  let 
$\mc{Q}_y \defeq \set{Q: \E_Q \ge y}$  for any $y$ and take
$Q_x \in \mc{Q}_x, Q_y \in \mc{Q}_y$.  Fix $x,y$ and let $z = \alpha x+(1-\alpha)y$.  
Then, letting $Q^\star = \alpha Q_x +(1-\alpha) Q_y$, we have $\E_{Q^\star}[\rv] = 
\alpha\E_{Q_x}[\rv] + (1-\alpha) \E_{Q_z}[\rv] \ge \alpha x + (1-\alpha) y =z$ so that $Q^\star \in \mc{Q}_z$. In addition,
\begin{align*}
L(z) \le 
    \wass{Q^\star}{P} \le \alpha    \wass{Q_x}{P} + (1-\alpha)    \wass{Q_y}{P},
\end{align*} taking infimum over $Q_x \in \mc{Q}_x$ and $Q_y \in \mc{Q}_y$, we arrive at
\begin{equation*}
    L(z) \le \alpha L(x) + (1-\alpha) L(y),
\end{equation*} and therefore $L(\cdot)$ is convex. 
Standard arguments from convex analysis~\citep[Theorem 12.2]{Rockafellar70} 
imply that 
\begin{equation}
    L(y)  = \sup_{\lambda \in \R} \set{\lambda y + G(\lambda)} 
     = \sup_{\lambda 
\ge 0 } \set{\lambda y + G(\lambda)}, \label{eqn:duality-fenchel-alternative-form}
\end{equation} 
where the second equality follows from the fact that $G(\lambda) = -\infty$ for any $\lambda < 0$.

Next,  for any  $\lambda \ge 0$,
\begin{align*}
G(\lambda) &= \inf_{y \in \R} \set{ \inf_{Q} \set{ \wass{Q}{P}:  \E_Q [\rv] \geq y} - \lambda y} \\ 
&= \inf_{y \in \R} \set{ \inf_{Q} \set{ \wass{Q}{P} -\lambda y:  \E_Q [\rv] \geq y} } \\ 
&= \inf_{Q}  \set{ \inf_{y \in \R}  \set{ \wass{Q}{P} - \lambda y:  \E_Q [\rv] \geq y} } \\ 
&= \inf_{Q}  \set{  \wass{Q}{P} -   \lambda   \E_Q [\rv]  }.
\end{align*}
Using the definition 
\begin{align*}
\wass{Q}{P} = \inf_{\pi \in \Pi(P,Q)} \E_{(\rv, \hat{\rv}) \sim \pi} \Paran{c(\rv, \hat{\rv})},
\end{align*}
we have 
\begin{align*}
G(\lambda) &=  \inf_Q \set{\inf_{\pi \in \Pi(P,Q)} \E_{(\rv, \hat{\rv}) \sim \pi} \Paran{c(\rv, \hat{\rv}) - \lambda \hat{\rv}}} \\
&= \inf_{\pi \in \Pi_P} \set{\E_{(\rv, \hat{\rv}) \sim \pi} \Paran{c(\rv, \hat{\rv}) - \lambda \hat{\rv}}} .
\end{align*} 
From~\eqref{eqn:duality-fenchel-alternative-form}, we have
\begin{align*}
L(y)  
&= \sup_{\lambda 
\ge 0 } \set{\lambda y + G(\lambda)} \\
& = \sup_{\lambda \ge 0} \set{\lambda y + \inf_{\pi \in \Pi_P} \E_{(\rv, \hat{\rv}) \sim \pi} \Paran{c(\rv, \hat{\rv}) - \lambda  \hat{\rv}}} \\
& = \sup_{\lambda \ge 0} \set{\lambda y - \sup_{\pi \in \Pi_P} \E_{(\rv, \hat{\rv}) \sim \pi} \Paran{\lambda  \hat{\rv} - c(\rv, \hat{\rv})}} \\
& = \sup_{\lambda \ge 0} \set{\lambda y -   \E_{P}[\psi_\lambda(\rv)] },
\end{align*}  where the last equality comes from Lemma~\ref{lemma:duality-change-order}.
\end{proof}
For completeness, we also include standard duality results in the literature~\citep{HuHong13, Shapiro17, DuchiNa21, GaoKl22} of the worst-case risk $\E_Q[\rv]$ when $Q$ belongs to an ambiguity set parametrized by $\rho > 0$.
\begin{prop}
We have 
\begin{align}
\sup_{Q} \set{ \E_Q [\rv]: \dkl{Q}{P} \le \rho} &= \inf_{\lambda > 0} \set{\lambda \rho +\lambda \log \E_P\Paran{e^{\frac{\rv}{\lambda}}}  }, 
\label{eqn:worst-case-exp-kl-duality}
 \\ 
\sup_{Q} \set{ \E_Q [\rv]: \fdiv{Q}{P} \le \rho} &= \inf_{\lambda > 0, \eta \in \R} \set{\lambda \rho + \eta + \E_P\Paran{\lambda f^*\paran{\frac{\rv - \eta}{\lambda}}}}, \\
\sup_{Q} \set{ \E_Q [\rv]: \wass{Q}{P} \le \rho} &=  \inf_{\lambda \ge 0} \set{\lambda \rho +   \E_{P}[\phi_\lambda(\rv)] },
\end{align} where $\phi_\lambda(z) =  \sup_{x \in \mc{X}} \set{x - \lambda c(z,x)}.$
\end{prop}
}


\section{Proof of convergence results}
\label{section:proof-of-upper}

We begin by proving Theorem~\ref{theorem:convergence} in
Section~\ref{section:proof-of-convergence}.  In
Theorem~\ref{theorem:convergence}, if we set $\bar{\lambda}$ and $\ubletter$
to be arbitrarily close to $\opttilt$ and
$\min\{2, \frac{\absc}{\bar{\lambda}}\}$ respectively, then we can obtain a
convergence rate close to
$O_p(n^{\min\{2, \frac{\absc}{\opttilt}\}^{-1} - 1})$. Roughly speaking, this
is the asymptotic---\emph{distribution-dependent}---rate of convergence for
the dual plug-in estimator shown by~\citet{Feuerverger89}.  Unlike our finite
sample convergence result, an asymptotic result requires
$\rv^{a} e^{\lambda \rv}$ to be regularly varying of order
$\frac{\lambda}{\absc}$~\cite{Feller71, DeHaanFe07} for any
$a, \lambda \ge 0$.

\subsection{Proof of Theorem~\ref{theorem:convergence}}
\label{section:proof-of-convergence}

\paragraph{Preliminaries} Let
$1 < \ubletter \le \min\{2, \frac{\absc}{\bar{\lambda}}\}$, and thus $\ubletter \in (1, 2]$. Throughout
this section, we will abuse notation and use $K$ to denote different universal
constants. Furthermore, we ignore all measurability issues in this proof which
can be dealt with by using outer measures and appropriate versions of Fubini's
theorems as in~\cite[Section~2.3]{VanDerVaartWe96}. Throughout our proof, we
use the inequality $(x +y)^{\ubletter} \le 2^{\ubletter-1} (x^{\ubletter} + y^{\ubletter})$
for $x, y \ge 0$ without justification. The inequality is a consequence of
convexity of $t \mapsto t^{\ubletter}$ on $\R_+$.

Recall the $\psi$-Orlicz norms of a random
variable $\rv$
\begin{equation*}
  \norm{Z}_{\psi} = \inf\left\{ c > 0: \E\left[
      \psi\left(\frac{|Z|}{c} \right) \right] \le 1 \right\},
\end{equation*}
where $\psi$ is a nondecreasing, convex function with $\psi(0) = 0$. For
example, $x \mapsto x^p$ with $p \ge 1$ yields the familiar $L_p$-norm
$\norm{Z}_{p} = (\E|Z|^p)^{1/p}$. If we let $\psi_p(x) = e^{x^p} - 1$ for
$p \ge 1$, we have $\norm{Z}_\ubletter \le \norm{Z}_2 \le \norm{Z}_{\psi_2}$
since $\ubletter \le 2$ and $x^2 \le \psi_2(x)$ for $x \ge 0$.

\paragraph{Symmetrization}  In the following, we show
that
\begin{equation}
  \label{eqn:moment-bound}
  \E|I(P) - \what{I}_n|^\ubletter \le  K \bar{\lambda}^{\ubletter}  n^{-\ubletter + 1}
  \E[|\rv|^\ubletter e^{\bar{\lambda} \ubletter \rv}].
\end{equation}

Using Markov's inequality, we 
plug $t = 
(\delta)^{-\frac{1}{\ubletter}}
K^{\frac{1}{\ubletter}} \bar{\lambda}
n^{\frac{1}{\ubletter} - 1}
\left(
  \E[|\rv|^\ubletter e^{\bar{\lambda} \ubletter \rv}]
  \right)^{\frac{1}{\ubletter}}
$ into
$\P\left( |I(P) - \what{I}_n| \ge t \right) \le 
\frac{\E|I(P) - \what{I}_n|^\ubletter}{t^\ubletter}$
and obtain the desired
result~\eqref{eqn:upper-general}.

To show the bound~\eqref{eqn:moment-bound}, define
$\opttilt \defeq \argmax_{\lambda} \left\{ \lambda y - \log \E e^{\lambda \rv}
\right\}$. Since $\E[\rv] \le y$ and $\opttilt \le \bar{\lambda}$, we have
$0 \le \opttilt \le \bar{\lambda}$ and
$I = \sup_{\lambda \in [0, \bar{\lambda}]} \left\{ \lambda y - \log
  \E[e^{\lambda \rv}] \right\}$. Recalling $|\log x - \log y| \le |x-y|$ for
$x, y \ge 1$, note $\E_{\emp}[e^{\lambda \rv}],~\E[e^{\lambda \rv}] \ge 1$ for
$\lambda \ge 0$ to arrive at the bound
\begin{align*}
  |I - \what{I}_n |
  \le \sup_{\lambda \in [0, \bar{\lambda}]}
  \left| \log \frac{\E [e^{\lambda \rv}]}{\E_{\emp} [e^{\lambda \rv}]} \right|
  \le \sup_{\lambda \in [0, \bar{\lambda}]}
  \left| \E [e^{\lambda \rv}] - \E_{\emp} [e^{\lambda \rv}] \right|.
\end{align*}

 We proceed via a symmetrization
argument~\cite[Chapter 2.3]{VanDerVaartWe96}.  Letting $\varepsilon_i$ be
i.i.d. random signs taking values in $\{-1, +1\}$ uniformly, our goal is to
replace the preceding supremum with a supremum over the \emph{symmetrized
  process}
$\sup_{\lambda \in [0, \bar{\lambda}]} \left| \frac{1}{n} \sum_{i=1}^n
  \varepsilon_i e^{\lambda \rv_i} \right|^{\ubletter}$. We will shortly see that
this latter supremum can be controlled by using the modulus of continuity of
the symmetrized process
$\lambda \mapsto \frac{1}{n} \sum_{i=1}^n \varepsilon_i e^{\lambda \rv_i}$.
To this end, define $\rv_i'$ to be i.i.d. copies of $\rv_i$ and denote by
$\E_{\rv'}$ expectation with respect to only $\rv'$. Rewrite the preceding
display with this new representation
\begin{align*}
  \E|I - \what{I}_n|^{\ubletter}
   & \le \E_{\rv} \sup_{\lambda \in [0, \bar{\lambda}]}
     \left| \E [e^{\lambda \rv}] - \E_{\emp} [e^{\lambda \rv}] \right|^{\ubletter} 
     = n^{-\frac{\ubletter}{2}} \E_{\rv} \sup_{\lambda \in [0, \bar{\lambda}]}
     \left| \E_{\rv'} \frac{1}{\sqrt{n}} \sum_{i=1}^n (e^{\lambda \rv_i}
     - e^{\lambda \rv_i'}) \right|^{\ubletter} \\
   &  \le n^{-\frac{\ubletter}{2}}  \E \sup_{\lambda \in [0, \bar{\lambda}]}
     \left|  \frac{1}{\sqrt{n}} \sum_{i=1}^n (e^{\lambda \rv_i}
     - e^{\lambda \rv_i'}) \right|^{\ubletter}.
\end{align*}
For any set of signs $\varepsilon_i$, the expectation remains unchanged if we
replace $e^{\lambda \rv_i} - e^{\lambda \rv_i'}$ with
$\varepsilon_i (e^{\lambda \rv_i} - e^{\lambda \rv_i'})$.  Using the
independence of $\varepsilon_i$'s, conclude
\begin{align}
  \E
  \sup_{\lambda \in [0, \bar{\lambda}]}
  \left| \frac{1}{\sqrt{n}} \sum_{i=1}^n (e^{\lambda \rv_i} - e^{\lambda \rv_i'}) \right|^{\ubletter}
  & = \E \sup_{\lambda \in [0, \bar{\lambda}]}
    \left| \frac{1}{\sqrt{n}}
    \sum_{i=1}^n \varepsilon_i (e^{\lambda \rv_i} - e^{\lambda \rv_i'})
    \right|^{\ubletter}  \nonumber \\
  & \le 2^{\ubletter} \E \sup_{\lambda \in [0, \bar{\lambda}]}
    \left| \frac{1}{\sqrt{n}} \sum_{i=1}^n \epsilon_i e^{\lambda \rv_i} \right|^{\ubletter}.
    \label{eqn:upper-bound-E-sup}
\end{align}
For a rigorous treatment of measurability issues arising in a symmetrization
argument, see~\citet[Section 2.3]{VanDerVaartWe96}.

We now bound~\eqref{eqn:upper-bound-E-sup} with a $\psi_2$-Orlicz norm over
the suprema of the symmetrized process.  From convexity of
$t \mapsto t^{\ubletter}$ for $t \ge 0$, 
\begin{align*}
  \E_{\epsilon} \sup_{\lambda \in [0, \bar{\lambda}]}
  \left| \frac{1}{\sqrt{n}} \sum_{i=1}^n \epsilon_i e^{\lambda \rv_i} \right|^{\ubletter}
  & \le 2^{\ubletter - 1} \left( \E_{\epsilon} \sup_{\lambda \in [0, \bar{\lambda}]}
  \left| \frac{1}{\sqrt{n}} \sum_{i=1}^n \epsilon_i (e^{\lambda \rv_i} - 1)
  \right|^{\ubletter}
  + \E_{\epsilon} \left| \frac{1}{\sqrt{n}} \sum_{i=1}^n \epsilon_i \right|^{\ubletter}
    \right)  \\
  & \le 2^{\ubletter - 1} \left( \E_{\epsilon} \sup_{\lambda, \lambda' \in [0, \bar{\lambda}]}
  \left| \frac{1}{\sqrt{n}} \sum_{i=1}^n \epsilon_i (e^{\lambda \rv_i} - e^{\lambda' \rv_i})
  \right|^{\ubletter}
  + \E_{\epsilon} \left| \frac{1}{\sqrt{n}} \sum_{i=1}^n \epsilon_i \right|^{\ubletter}
    \right).
\end{align*}
Recalling the Orlicz norm discussed in the preliminaries, the preceding
display can be bounded by the $\psi_2$-Orlicz norm (defined conditional on
$\rv_i$)
\begin{align}
  2^{\ubletter - 1} \left( \norm{\sup_{\lambda, \lambda' \in [0, \bar{\lambda}]}
  \left| \frac{1}{\sqrt{n}} \sum_{i=1}^n \epsilon_i (e^{\lambda \rv_i} - e^{\lambda' \rv_i})
  \right| }_{\psi_2}^{\ubletter}
  + \E_{\epsilon} \left| \frac{1}{\sqrt{n}} \sum_{i=1}^n \epsilon_i \right|^{\ubletter}
  \right).
    \label{eqn:orlicz-bound}
\end{align}
We proceed by bounding the two terms separately.

\paragraph{Bounding the second term in Eq.~\eqref{eqn:orlicz-bound}} To control
$\E_{\epsilon} \left| \frac{1}{\sqrt{n}} \sum_{i=1}^n \epsilon_i
\right|^{\ubletter}$, use Hoeffding's bound to get
\begin{align*}
  \E_{\epsilon}\left| \sum_{i=1}^n \epsilon_i x_i \right|^{\ubletter}
  = \int_0^{\infty} \P\left( \left| \sum_{i=1}^n \epsilon_i x_i \right|^{\ubletter}
  \ge t \right) dt
  \le \int_0^{\infty}
  2 \exp\left( -\frac{t^{\frac{2}{\ubletter}}}{2\norm{x}_2^2} \right) dt = \ubletter 2^{\frac{\ubletter}{2}} \norm{x}_2^{\ubletter}
  \int_0^{\infty} u^{\frac{\ubletter}{2} - 1} e^{-u} du,
\end{align*}
where we used a change of variables
$u = \frac{t^{\frac{2}{\ubletter}}}{2\norm{x}_2^2}$. Applying this bound with
$x \equiv 1$,
\begin{align}
  \E_{\epsilon} \left| \frac{1}{\sqrt{n}} \sum_{i=1}^n \epsilon_i \right|^{\ubletter}
  \le \ubletter 2^{\frac{\ubletter}{2}} 
  \Gamma \left( \frac{\ubletter}{2} \right).
  \label{eqn:second-term-bound}
\end{align}

\paragraph{Bounding the first term in Eq.~\eqref{eqn:orlicz-bound} through
  chaining}
We use a standard chaining argument (see, for
example,~\cite[Ch. 2.2]{VanDerVaartWe96}) to control the first term in the
bound~\eqref{eqn:orlicz-bound}. Our argument relies on the notion of 
covering numbers. Let $V \subset \mc{V}$ be any subset of a vector space
$\mc{V}$.  For a (semi) norm $d$ on $\mc{V}$, we say a collection of
$v_1, \ldots, v_N \subset \mc{V}$ is an $\epsilon$-cover of $\mc{V}$ if for
each $v \in \mc{V}$, there exists $v_i$ such that $d(v, v_i) \le
\epsilon$. Then \textit{covering number} of $V$ with respect to $d$ is given
by
\begin{equation*}
  N(V, \epsilon, d) \defeq \inf\left\{ N \in \N : \text{there is an}~ \epsilon\text{-cover over}~ V ~\text{of size}~ N ~\text{with respect to}~ d \right\}.
\end{equation*}
The following classical chaining inequality bounds the $\psi_2$-Orlicz norm of
the supremum of a stochastic process.
\begin{lemma}
  \cite[Theorem 2.2.4]{VanDerVaartWe96}
  \label{lemma:chaining}
  For some (semi) metric $d$ on $V$, let $\{ Z_v: v \in V \}$ be a separable
  stochastic process with
  \begin{equation*}
    \norm{Z_v - Z_w}_{\psi_2} \le Cd(v, w) ~~\text{  for every   }~~~ v, w \in V
  \end{equation*}
  for a constant $C$. Then, for a constant $K$ depending only on $C$
  \begin{equation*}
    \norm{\sup_{v, w \in V} |Z_v - Z_w|}_{\psi_2}
    \le K \int_0^{\diam V} \sqrt{\log N(V, \epsilon, d)} d\epsilon.
  \end{equation*}
\end{lemma}

To apply Lemma~\ref{lemma:chaining}, we verify
$\lambda \mapsto \frac{1}{\sqrt{n}} \sum_{i=1}^n \varepsilon_i e^{\lambda \rv_i}$
is Lipschitz with respect to the $L_2(\emp)$-norm
$\norm{Z}_n \defeq \left( \frac{1}{n} \sum_{i=1}^n Z_i^2 \right)^{\half}$.
Conditional on $\rv_i$'s,  Hoeffding's inequality gives
\begin{equation*}
  \P_{\varepsilon} \left(\left|\frac{1}{\sqrt{n}} \sum_{i=1}^n \varepsilon_i e^{\lambda \rv_i} -
      \frac{1}{\sqrt{n}} \sum_{i=1}^n \varepsilon_i e^{\lambda' \rv_i} \right|
    \ge t \right)
  \le 2 \exp\left( -\frac{t^2}{2\norm{e^{\lambda \rv} - e^{\lambda'\rv}}_n^2} \right).
\end{equation*}
\begin{claim}
The preceding display implies
\begin{equation*}
  \norm{\frac{1}{\sqrt{n}} \sum_{i=1}^n \varepsilon_i e^{\lambda \rv_i} -
    \frac{1}{\sqrt{n}} \sum_{i=1}^n \varepsilon_i e^{\lambda' \rv_i}}_{\psi_2}
  \le \sqrt{6} \norm{e^{\lambda \rv} - e^{\lambda'\rv}}_n.
\end{equation*}
\end{claim}
\begin{proof-of-claim}
For any random variable $Z$ such that $\P(|Z| > u) \le 2\exp(-v^2 u^2)$, 
Fubini gives
\begin{equation*}
  \E[e^{(Z / c)^2} - 1] = \int_0^{\infty} \P\left( e^{(Z/c)^2} - 1 \ge t \right) dt
  = \int_0^\infty \P\left( |Z| \ge \sqrt{u} \right) \frac{1}{c^2} e^{u / c^2} du
  \le \frac{2}{c^2v - 1}
\end{equation*}
where we used the change of variables $u = c^2 \log (t+1)$.  Since
$2 / (c^2v - 1) \le 1$ for $c \ge \sqrt{3 / v}$, we obtain the claim.
\end{proof-of-claim}

Letting
$\mc{F} \defeq \left\{ x \mapsto e^{\lambda x}: \lambda \in [0, \bar{\lambda}]
\right\}$, use the claim to apply Lemma~\ref{lemma:chaining}, which
immediately yields
\begin{align*}
  \norm{\sup_{\lambda, \lambda' \in [0, \bar{\lambda}]}
  \left| \frac{1}{\sqrt{n}} \sum_{i=1}^n \epsilon_i
  (e^{\lambda \rv_i} - e^{\lambda' \rv_i})
  \right| }_{\psi_2}
  & \le K \int_0^{\bar{\lambda} \norm{F}_n} \sqrt{\log N(\epsilon,
    \mc{F}, \norm{\cdot}_n)}
    d\epsilon  \\
  & = K \int_0^{\bar{\lambda}}
    \sqrt{\log N(\epsilon \norm{F}_{n}, \mc{F}, \norm{\cdot}_n)}
    d\epsilon  \cdot  \norm{F}_n.
\end{align*}
To bound the covering number of $\mc{F}$, we use the Lipschitzness of
$\lambda \mapsto e^{\lambda x}$. 
\begin{lemma}
  \label{lemma:lipschitz-covering}
  \cite[Theorem 2.7.11]{VanDerVaartWe96} If
  $\mc{F} = \{f_\lambda: \lambda \in \Lambda \}$ is a class of functions such
  that $|f_{\lambda}(x) - f_{\lambda'}(x)| \le F(x) d(\lambda, \lambda')$ for
  some metric $d$ on $\Lambda$ and function $F$ on the sample space, we have
  \begin{equation*}
    N(\epsilon \norm{F}, \mc{F}, \norm{\cdot}_n)
    \le N(\epsilon, \Lambda, d).
  \end{equation*}
\end{lemma}
\noindent
Note that $|e^{\lambda x} - e^{\lambda' x}| \le F(x) |\lambda - \lambda'|$
with $F(x) = xe^{\bar{\lambda}x}$. Applying
Lemma~\ref{lemma:lipschitz-covering} with $\Lambda := [0,\bar{\lambda}]$,
\begin{align*}
 \int_0^{\bar{\lambda}}
    \sqrt{\log N(\epsilon \norm{F}_{n}, \mc{F}, L_2(\emp))}
    d\epsilon  \cdot  \norm{F}_n  
   \le  \int_0^{\bar{\lambda}} \sqrt{\log \left( 1+ \frac{\bar{\lambda}}{\epsilon}\right)}
    d\epsilon \cdot \norm{F}_n \le K \bar{\lambda} \norm{F}_n.
\end{align*}

\paragraph{Final bound} Collecting these bounds, we arrive at
\begin{align*}
  \E | I - \what{I}_n |^{\ubletter}
  \le 2^{2\ubletter-1} n^{-\frac{\ubletter}{2}}
  \E_{\rv, \epsilon} \sup_{\lambda \in [0, \bar{\lambda}]}
  \left| \frac{1}{\sqrt{n}} \sum_{i=1}^n \epsilon_i e^{\lambda \rv_i} \right|^{\ubletter}
  & \le 2^{2\ubletter-1} n^{-\frac{\ubletter}{2}}
    \left( K \bar{\lambda}^{\ubletter} \E_{\rv}[\norm{F}_n^{\ubletter}] +
    \ubletter 2^{\frac{\ubletter}{2}} \Gamma\left(\frac{\ubletter}{2}\right) \right).
\end{align*}
Noting $\norm{x}_1 \le \left(\sum_{i=1}^n |x_i|^{\ubletter/2}\right)^{2/\ubletter}$
for $\ubletter \in (1, 2]$, the final result~\eqref{eqn:moment-bound} follows from
\begin{equation*}
  \E_{\rv} \norm{F}_n^{\ubletter}
  = n^{-\frac{\ubletter}{2}} \E_\rv \left( \sum_{i=1}^n \rv_i^2 e^{2\bar{\lambda}\rv_i}
  \right)^{\frac{\ubletter}{2}}
  \le n^{-\frac{\ubletter}{2} + 1} 
  \E[|\rv|^\ubletter e^{\bar{\lambda} \ubletter \rv}].
\end{equation*}

Summarizing, we have shown
\begin{align} 
  \E|I - \what{I}_n|^\ubletter
  & \le 2^{2\ubletter-1} n^{-\frac{\ubletter}{2}} 
    \E_{\rv, \epsilon} \sup_{\lambda \in [0, \bar{\lambda}]}
    \left| \frac{1}{\sqrt{n}} \sum_{i=1}^n \epsilon_i e^{\lambda \rv_i} 
    \right|^{\ubletter} \nonumber  \\
  & \le  2^{2\ubletter-1}  n^{-\frac{\ubletter}{2}}
    \left(
    K \bar{\lambda}^{\ubletter} \E_{\rv}[\norm{F}_n^{\ubletter}] +
    \ubletter 2^{\frac{\ubletter}{2}} \Gamma\left(\frac{\ubletter}{2}\right)
    \right) \nonumber \\
  & \le K \bar{\lambda}^{\ubletter}  n^{-\ubletter + 1}
  \E[|\rv|^\ubletter e^{\bar{\lambda} \ubletter \rv}].
  \label{eqn:bound-on-I-I-n}
\end{align}
for an appropriately large universal constant $K$, which is exactly~\eqref{eqn:moment-bound}.

\subsection{Proof of Theorem~\ref{theorem:upper-ld-rate}}
\label{section:proof-of-upper-ld-rate}

For an $\epsilon > 0$ to be chosen later, let
$\bar{\lambda} = \absc - \frac{\mcpgamma}{y}$ and
$\ubletter \defeq \frac{1}{1+2\epsilon} \frac{\absc}{\absc - \mcpgamma/y}  \in (1,2)$. 
Applying
Theorem~\ref{theorem:convergence} with this definition of $\bar{\lambda}$ and
$\ubletter$, we proceed to bound $\E_P[\rv^\ubletter e^{\bar{\lambda}\ubletter \rv}]$
uniformly over $P \in \mc{P}_{\absc, y, \mcpgamma}$.

Toward this goal, we will use the following useful algebraic inequality
\begin{equation}
  \label{eqn:moment-bound-exponential-function}
  x^a \le e^{tx} \left(\frac{a}{et} \right)^a
  ~~~\mbox{for all}~~x \geq 0, t > 0, a > 0.
\end{equation}
To see this, note that $1 + y \le e^y$ with $y = -1 + \frac{tx}{a}$ gives
$\frac{tx}{a} \le e^{ \frac{tx}{a} -1}$. 
Applying the inequality~\eqref{eqn:moment-bound-exponential-function} with
$t = \frac{\epsilon}{1+2\epsilon} \absc, a = \ubletter$, we deduce
\begin{equation*}
\rv^\ubletter e^{\bar{\lambda}\ubletter \rv}
\le e^{\frac{1}{1+2\epsilon} \absc \rv} e^{\frac{\epsilon}{1+2\epsilon} \absc \rv} \left(\frac{\ubletter}{e} \right)^\ubletter
\left(\frac{1+2\epsilon}{\epsilon  \absc} \right)^\ubletter.
\end{equation*}
Taking expectation on both sides, use Condition~\ref{item:mgf-bound} and
$\ubletter < e$ to arrive at
\begin{equation}
 \E_P[\rv^\ubletter e^{\bar{\lambda}\ubletter \rv}]
 \le \left(\frac{\ubletter}{e} \right)^\ubletter
  \left(\frac{1+2\epsilon}{\epsilon  \absc} \right)^\ubletter
  \E_P\left[e^{\frac{1+\epsilon}{1+2\epsilon} \absc \rv}\right]
  \le \left(\frac{1+2\epsilon}{\epsilon} \right)^{\ubletter+1}
\left(\frac{1}{  \absc}\right)^\ubletter. \label{eqn:upper-ld-proof-modify}
\end{equation}

Using the preceding bound on $\E_P[\rv^\ubletter e^{\bar{\lambda}\ubletter \rv}]$ in Theorem~\ref{theorem:convergence}, conclude
\begin{equation*}
  \sup_{P \in \mc{P}_{\absc,y,\mcpgamma}}
  \P\left(\left|\what{I}_n - I(P)\right| \ge
    K \left(
 1- \frac{\mcpgamma}{\absc y}
    \right) 
    \left(\frac{1}{\delta}\right)^{(1+2\epsilon) (1 - \frac{\mcpgamma}{\absc y})}
    \left( 2+ \frac{1}{\epsilon} \right)^{ 1+ (1+2\epsilon) (1 - \frac{\mcpgamma}{\absc y})}
    n^{- \frac{\mcpgamma}{\absc y} + 2\epsilon (1 - \frac{\mcpgamma}{\absc y})}
  \right) 
  \le \delta.
\end{equation*}
Choose $\epsilon \defeq \frac{1}{2 \log n} \frac{1}{1 - \frac{\mcpgamma}{\absc y}}$. Noting that
$\epsilon \le \half$ since $n 
\ge 8 > e^2
\ge \exp\left(
\frac{1}{1- \frac{\mcpgamma}{\absc y}}
\right)$ and $n^{1/\log n} = e$,
we obtain the desired result.

\subsection{Proof of Theorem~\ref{theorem:upper-md-rate}}
\label{section:proof-of-upper-md-rate}

Applying Theorem~\ref{theorem:convergence} with
$\bar{\lambda} \defeq \absc - \frac{\mcpgamma}{y}$ and $\ubletter =2$, it remains to
bound $\E_P[\rv^\ubletter e^{\bar{\lambda} \ubletter \rv}]$ uniformly over
$P \in \mc{P}_{\absc, y, \mcpgamma}$.  Use the
bound~\eqref{eqn:moment-bound-exponential-function} with
$t = \frac{\mcpgamma}{y} - \frac{\absc}{2} > 0$ and note $a = 2 < e$ to get
\begin{equation*}
\rv^\ubletter e^{\bar{\lambda}\ubletter \rv}
\le
e^{2(\absc - \frac{\mcpgamma}{y}) \rv}
e^{(\frac{\mcpgamma}{y} - \frac{\absc}{2}) \rv}
\left(\frac{1}{\frac{\mcpgamma}{y} - \frac{\absc}{2}} \right)^2.
\end{equation*}
Taking expectation on both sides and using Condition~\ref{item:mgf-bound},
\begin{equation}
\E_P[\rv^\ubletter e^{\bar{\lambda} \ubletter \rv}]
\le  \left(\frac{y}{\mcpgamma -\frac{\absc y}{2}} \right)^2 
\frac{\absc y}{\mcpgamma -  \frac{\absc y}{2}}.
\label{eqn:upper-md-proof-modify}
\end{equation}

\highlight{
\subsection{Proof of Theorem~\ref{theorem:upper-ld-rate-relaxed}}
\label{sec:proof-of-upper-ld-rate-relaxed}
    We use the same proof as in Theorem~\ref{theorem:upper-ld-rate}, except that we change \eqref{eqn:upper-ld-proof-modify}
    to
    \begin{equation*}
         \E_P[\rv^\ubletter e^{\bar{\lambda}\ubletter \rv}]
 \le \left(\frac{\ubletter}{e} \right)^\ubletter
  \left(\frac{1+2\epsilon}{\epsilon  \absc} \right)^\ubletter
  \E_P\left[e^{\frac{1+\epsilon}{1+2\epsilon} \absc \rv}\right]
  \le \left(\frac{1+2\epsilon}{\epsilon} \right)^{\ubletter+m}
\left(\frac{1}{  \absc}\right)^\ubletter.
    \end{equation*} and then the rest follows similarly, and we complete proving~\eqref{eqn:upper-ld-rate-relaxed}.

\subsection{Proof of properties of $ \mc{Q}_{\bar{\lambda},y}$}
\label{section:proof-property-Q-class}

The statement for Gaussian random variables follows directly from Example~\ref{example:Gamma}.

We now consider $P = \mathsf{Ber}(p)$, then $F(\lambda) \defeq \lambda y - \log \E_P[e^{\lambda \rv}] = 
\lambda y - \log((1-p) + p e^{\lambda} )
$ and $F'(\lambda) = y - \frac{p e^{\lambda}}{1-p + p e^{\lambda}}$, so we obtain $\lambda\opt = 
\log\left(\frac{(1-p)y}{p(1-y)}\right)
$ by setting $F'(\lambda) = 0$. 
Solving $\lambda\opt \le \bar{\lambda}$ yields $p \ge \frac{y}{e^{\bar{\lambda}} - y   ( e^{\bar{\lambda}}- 1)}$.

Next, consider $P = \mathsf{Uni}(0,1)$, we have 
$F(\lambda) = \lambda y - \log \paran{\frac{e^{ \lambda} - 1}{\lambda}}$, then 
\begin{align*}
    F'(\lambda) = \frac{e^{\lambda} ((y-1) \lambda + 1) - y \lambda - 1}{\lambda e^{\lambda}-1}.
\end{align*}
Noting that $F'( \frac{1}{1-y}) \le 0$, we have
  $\opttilt(\mathsf{Uni}(0,1)) \le \frac{1}{1-y}$. 
  Similarly, we can show that $\opttilt(\mathsf{Uni}(0,b)) \le \frac{1}{1-y/b}$ and
  and $\opttilt(\mathsf{Uni}(a,b)) \le \frac{1}{1-(y-a)/(b-a)} =  \frac{b-a}{b-y}$.
  }

\newcommand{\GammaConst}{\frac{\absc^{1-\eta}}{\Gamma(1-\eta)}}
\newcommand{\GammaConstRec}{\frac{\Gamma(1-\eta)}{\absc^{1-\eta}}}
\newcommand{\GammaLambdaOptUpper}{\left( \absc - \frac{\mcpgamma}{y} \right)} 
\newcommand{\GammaLambdaOptSep}{\left( \absc - \frac{1-\eta}{y} \right)}
\newcommand{\CUpperBound}{\GammaConst    \frac{\absc x_0}{\absc x_0 - 1}}
\newcommand{\expfactorGammaEta}{e^{- (1-\eta)x_0/y}}
\newcommand{\expfactorGammaEtaTwo}{e^{- \mcpgamma x_0/y}}

\section{Proof of fundamental hardness results}
\label{section:proof-of-lower}

\subsection{Proof of Theorem~\ref{theorem:lower-ld-rate}}
\label{section:proof-of-lower-ld-rate}
We proceed by using Le Cam's method by reducing the estimation problem to a
binary hypothesis test. %
We denote by $\mc{P}$ the set of probability distributions over which the
worst-case is taken over.
\begin{lemma}[Lemma 1,~\citet{Yu97}]
  \label{lemma:lecam}
  Let $(\Theta, d)$ be a pseudo-metric space and let
  $\theta: \mc{P} \to \Theta$ be the parameter of interest. If there exists
  $P_1, P_2 \in \mc{P}$ such that $d(\theta(P_1), \theta(P_2)) \ge 2\epsilon$, we
  have
  \begin{equation*}
    \inf_{\what{\theta}(\rv_1^n)} \sup_{P \in \mc{P}}
    \P \left( d(\what{\theta}(\rv_1^n), \theta(P)) \ge \epsilon \right)
    \ge \half (1 - \tvnorm{P_1^n - P_2^n}).  
  \end{equation*}
\end{lemma}
\noindent We take $\theta(P) = I(P)$,
$d(I(P_1), I(P_2)) = |I(P_1) - I(P_2)|$ and
$\mc{P} = \mc{P}_{\absc,y,\mcpgamma}$. To apply the lemma, we construct two
distributions $P_1$ and $P_2$ that are well-separated in terms of $I(P_i)$
but are close in total variation distance.

Our example is inspired by~\citet[Theorem 3]{HallTeVa92}. Let $1 < x_0$ be some
fixed number to be chosen later.  Throughout the section, we set
$\eta \defeq 1 - \mcpgamma - \frac{1}{\absc x_0}$ and define
$\rv_1 \sim P_1 = \mathsf{Gamma}(1-\eta,\absc) = \mathsf{Gamma}(\mcpgamma +
\frac{1}{\absc x_0},\absc)$ with density
$f_1(x) = \GammaConst x^{-\eta} e^{-\absc x} \indic{x \ge 0}$ and
$\rv_2 \sim P_2$ with density
\begin{equation*}
  f_2(x) = \begin{cases}
    0 & ~~\mbox{if}~~  x < 0 \\   
    \GammaConst x^{-\eta} e^{-\absc x}
    & ~~\mbox{if}~~ 0 \le x \le x_0 \\
    C x^{-1}  e^{-\absc x} & ~~\mbox{if}~~ x > x_0
  \end{cases}
  ~~\mbox{where}~~C \defeq \frac{ \GammaConst \int_{x_0}^\infty x^{-\eta} e^{-\absc x} dx
  }{\int_{x_0}^\infty x^{-1} e^{-\absc x} dx}.
\end{equation*} 
Throughout this section, we assume $y \geq \E[\rv_1]$ and $y \geq \E[\rv_2]$.

Since $\absc x_0 \ge \frac{2}{1-\mcpgamma} \ge \frac{2}{\mcpgamma}$,
$\half(1-\mcpgamma) \le \eta \le 1-\mcpgamma$ and
$\mcpgamma \le 1 - \eta \le \frac{3}{2} \mcpgamma$.  We also have
\begin{align}
\frac{1}{\Gamma(\mcpgamma)} \le 
    \frac{1}{\Gamma(1-\eta)} \le 1, ~~~
 \mcpgamma - 1 \le    -\eta \le \half(\mcpgamma - 1).  \label{eqb:ld-gamma-eta-bound}
\end{align}
We will use the following estimate repeatedly. From integration by parts, 
\begin{align*}
  \int_{x_0}^\infty \frac{1}{x^\eta} e^{-x} dx
  = \frac{1}{x_0^\eta} e^{-x_0} - \eta \int_{x_0}^\infty \frac{1}{x^{\eta+1}} e^{-x} dx
  = \frac{1}{x_0^\eta} e^{-x_0} - \eta \frac{1}{x_0^{\eta+1}} e^{-x_0}
  + \eta(\eta+1) \int_{x_0}^\infty \frac{1}{x^{\eta+2}}  e^{-x} dx,
\end{align*}
so that for any $w > 0$
\begin{equation} 
  \label{eqn:int-estimate}
  e^{-w x_0} \frac{1}{w x_0^\eta} \left(1-\frac{\eta}{wx_0}\right)
  \le \int_{x_0}^\infty \frac{1}{x^\eta} e^{-w x} dx
  \le 
  e^{-w x_0} \frac{1}{w x_0^\eta}
  \left( 1 \wedge \left(1- \eta \frac{1}{ w x_0} + \eta(\eta+1) \frac{1}{w^2 x_0^2}\right) \right).
\end{equation}
In particular, we have 
\begin{equation}
  \label{eqn:c-estimate}
\GammaConst x_0^{1-\eta}     \left( 1 -  \frac{\eta}{\absc x_0} \right)
\le C \le 
\GammaConst x_0^{1-\eta}      \left(  1 +  \frac{1}{\absc x_0 - 1} \right).
\end{equation}

We first confirm $P_1$ and $P_2$ are included in the set under
consideration. See Section~\ref{subsec:proof-of-class-inclusion} for the
proof.
\begin{lemma}
  \label{lemma:class-inclusion}
  $P_1, P_2 \in \mc{P}_{\absc,y,\mcpgamma}$ if
   \begin{align*}
     x_0  \ge 
     \frac{y}{\mcpgamma} \left( 2 \log x_0
     + \log \left( \frac{4 \absc \left(  \frac{\mcpgamma}{y} \vee 1 \right) }{1-\mcpgamma}  \right)
     \right) \vee 2.
 \end{align*}  
\end{lemma}
\noindent Next, we show $I_1 = I(P_1)$ and $I_2 = I(P_2)$ are well-separated.
\begin{lemma}
  \label{lemma:rate-separated}
  Let $\absc y \ge 2\mcpgamma$. 
  If $x_0 
\ge \paran{\frac{2}{\mcpgamma^2} \frac{\absc y + 1}{\absc y - 1} \vee \frac{2}{\mcpgamma}
\vee \frac{4}{\absc y -  1 } 
} y
  $, 
  we have
  \begin{equation*}
    |I_1 - I_2| = I_2 - I_1
    \ge      
 \frac{1}{\Gamma(\mcpgamma)} 
         \left(\frac{y}{x_0}\right)^{1-\mcpgamma}
 \expfactorGammaEta \frac{\absc y - 1}{2(\absc y + 1)(\absc x_0 - 1)}. 
  \end{equation*} 
\end{lemma}
\noindent We defer the proof of the lemma to Section~\ref{section:proof-of-rate-separated}.

Finally, we wish to show $P_1$ and $P_2$ are close in total variation distance
despite being well-separated in $I(P)$. From Pinsker's inequality, we have
$\tvnorms{P_1^n - P_2^n}^2 \le \half \dkls{P_1^n}{P_2^n} = \frac{n}{2}
\dkl{P_1}{P_2}$. See Section~\ref{section:proof-ld-kl-bound} for a proof of
the following bound on the KL divergence.
\begin{lemma}
  \label{lemma:ld-kl-bound}
If $x_0 \ge \frac{\log 3(1 \vee \absc^{-2})}{\half(1 - \mcpgamma)}$ and $x_0 \ge 1+ \frac{1}{\absc}$, then
$\dkl{P_1}{P_2} \le  e^{-\absc x_0}$.
\end{lemma} 
\noindent Let $x_0 = \frac{1}{\absc} \log cn$ for
$c = \frac{1}{2(1-2\delta)^2}$. By Lemma~\ref{lemma:ld-kl-bound}, we have
$\dkl{P_1}{P_2} \le e^{-\absc x_0} = \frac{1}{cn}$ and
$\tvnorms{P_1^n - P_2^n} \le 1-2\delta$. Applying
Lemmas~\ref{lemma:lecam},~\ref{lemma:rate-separated}, and using the identity
\begin{align*}
 \expfactorGammaEta  = \paran{\frac{1}{c n}}^{\frac{1-\eta}{\absc y}}
 = \paran{\frac{1}{c n}}^{\frac{\mcpgamma + \frac{1}{\log cn}}{\absc y}}
  = \paran{\frac{1}{c n}}^{\frac{\mcpgamma}{\absc y}} e^{-\frac{1}{\absc y}},
\end{align*} we obtain $\mathfrak{M}_n \ge \delta$ as claimed in Theorem~\ref{theorem:lower-ld-rate}.

\subsubsection{Proof of Lemma~\ref{lemma:class-inclusion}}
\label{subsec:proof-of-class-inclusion}

First, we show that $P_1 \in \mc{P}_{\absc,y,\mcpgamma}$. Since
$\E e^{\lambda \rv_1} = \left(\expmgf \right)^{1-\eta} \le \expmgf, \E[\rv_1] = \frac{1-\eta}{\absc} < \frac{1}{\absc}$, and
$\lambda\opt(P_1) = \absc - (1-\eta)/y \le \absc - \frac{\mcpgamma}{y}$,
all conditions are satisfied.

Clearly, $\rv_2$ satisfies Condition~\ref{item:nonneg-heavy-tail}. 
To confirm that
Condition~\ref{item:mgf-bound}  is satisfied,   we first present 
an integral inequality that is similar to the well-known covariance inequality 
(Proposition 7.2.1,~\cite{Ross96}):
$\E[f(\rv)g(\rv)] \ge \E[f(\rv)] \E[g(\rv)]$ for any increasing 
functions $f,g$ and random variable $\rv$.
\begin{lemma}[Theorem 1,~\cite{Wijsman85}]
\label{lemma:ineq-int-ratio}
Consider four non-negative functions $f_1,f_2,g_1,g_2$,
an interval $A \subset \R$,
and assume $\int_A f_i g_j d\mu > 0$
for $i,j=1,2$. If $\frac{f_1}{f_2}$ and $\frac{g_1}{g_2}$ are monotonic in the same direction,
then
\begin{equation*}
  \frac{\int_A f_1 g_1 d\mu}{\int_A f_1 g_2 d\mu}
  \ge 
  \frac{\int_A f_2 g_1 d\mu}{\int_A f_2 g_2 d\mu}.
\end{equation*}
\end{lemma}
\begin{proof}
Define $F(x,y) = f_1(x) f_2(y) - f_1(y) f_2(x), G(x,y) = g_1(x) g_2(y) - g_1(y) g_2(x)$.
Since  $\frac{f_1}{f_2}$ and $\frac{g_1}{g_2}$ are monotonic in the same direction, we deduce $F(x,y) G(x,y) \ge 0$ and
\begin{align*}
\int_A f_1 g_1 d\mu \int_A f_2 g_2 d\mu
 -  \int_A f_1 g_2 d\mu \int_A f_2 g_1 d\mu
 = \half \int_A\int_A F(x,y) G(x,y) \mu(dx) \mu(dy) \ge 0.
\end{align*}
\end{proof}
\noindent By Lemma~\ref{lemma:ineq-int-ratio} with
$f_1(x) = e^{-(\absc - \lambda) x}, f_2(x) = e^{-\absc x}, g_1(x) =x^{-\eta},
g_2(x) = x^{-1}$, we have
\begin{align*}
    \frac{\int_{x_0}^\infty x^{-\eta} e^{-\absc x} dx
}{\int_{x_0}^\infty x^{-1} e^{-\absc x} dx} 
\le 
    \frac{\int_{x_0}^\infty x^{-\eta} e^{-(\absc-\lambda) x} dx
}{\int_{x_0}^\infty x^{-1}e^{-(\absc-\lambda) x}  dx},
\end{align*}
so that 
\begin{align}
    \E[e^{\lambda \rv_2}]
    &= 
    \int_0^{x_0} \GammaConst
x^{-\eta} e^{-(\absc - \lambda) x} dx
+ \int_{x_0}^\infty 
\frac{
 \GammaConst \int_{x_0}^\infty x^{-\eta} e^{-\absc x} dx
}{\int_{x_0}^\infty x^{-1} e^{-\absc x} dx}  x^{-1} e^{-(\absc - \lambda) x}
dx  \nonumber \\
&\le  \int_0^{x_0} \GammaConst
x^{-\eta} e^{-(\absc - \lambda) x} dx
+ \int_{x_0}^\infty \GammaConst
x^{-\eta} e^{-(\absc - \lambda) x} dx \nonumber\\
 &=     \E[e^{\lambda \rv_1}]
    =\left(\expmgf \right)^{1-\eta} \le \expmgf.
    \label{eqn:mgf-bound-ld}
\end{align}
Similarly, Lemma~\ref{lemma:ineq-int-ratio} implies $\E[\rv_2] \le \E[\rv_1] < \frac{1}{\absc}$.

It now remains to show that $\rv_2$ satisfies
Condition~\ref{item:argmax-bound}.   
By~\eqref{eqn:first-order-optimality}, it suffices to show
$\E[\rv_2 e^{(\absc - \mcpgamma/y) \rv_2}] \ge y \E[e^{(\absc - \mcpgamma/y)
  \rv_2}]$.  Before proving this, it is useful to have the following technical
result.
\begin{lemma}
  \label{lemma:incomplete-gamma-bound}
  For any $\eta >0, x_0 \ge \eta +1, w > 0$,
  \begin{align}
    \int_{x_0}^\infty e^{-wx} x^\eta dx \le 
    \frac{1}{w} (\eta+1) e^{-wx_0} x_0^\eta. \label{eqn:incomplete-gamma-bound}
  \end{align}
\end{lemma}
\begin{proof} 
  Let $f(x) = \eta \log x - x$ so that $e^{-x}x^{\eta} = e^{f(x)}$. Since $f$
  is concave for $\eta >0$, $f(x) \le f(x_0)+ f'(x_0)(x-x_0)$ or equivalently,
  $\eta \log x - x \le \eta \log x_0 - x_0 + (\eta/x_0 - 1)(x-x_0)$. Hence,
  \begin{align*}
    \int_{x_0}^\infty e^{-x} x^\eta dx
    \le 
    \int_{x_0}^\infty e^{-x_0} x_0^\eta 
    e^{-(1 - \eta/x_0) (x-x_0)}  
    dx
    =e^{-x_0} x_0^\eta  \frac{1}{1-\eta/x_0}
    \le  e^{-x_0} x_0^\eta (\eta+1),
  \end{align*}
  where we used $\frac{1}{1-\eta/x_0} \le \eta+1$ for all $x_0 \ge 1+\eta$ in
  the final inequality.  We have the result since
  $\int_{x_0}^\infty e^{-wx} x^\eta dx = \frac{1}{w^{1+\eta}}
  \int_{wx_0}^\infty e^{-x} x^\eta dx$.\end{proof}

  Applying Lemma~\ref{lemma:incomplete-gamma-bound}, we have
\begin{align*}
 \E[\rv_2 e^{\lambda \rv_2}]
 &\ge \GammaConst \int_0^{x_0} x^{1-\eta} e^{-(\absc-\lambda) x} dx \\
 &= \GammaConst 
 \left(
 \Gamma(2-\eta) \frac{1}{(\absc - \lambda)^{2-\eta}} -\int_{x_0}^\infty  x^{1-\eta} e^{-(\absc-\lambda) x} dx \right) \\
 &\ge \GammaConst 
 \left(
 \Gamma(2-\eta) \frac{1}{(\absc-\lambda)^{2-\eta}} -  (2-\eta) \frac{1}{\absc - \lambda} \expfactor x_0^{1-\eta} 
 \right) \\
 &= \GammaConst
 \frac{1}{(\absc-\lambda)^{2-\eta}} 
 \left(
 \Gamma(2-\eta) - (2-\eta) (\absc-\lambda)^{-\eta+1} \expfactor x_0^{1-\eta} 
 \right). \\
  &\ge \left(\expmgf \right)^{1-\eta} \frac{1}{\Gamma(1-\eta) (\absc-\lambda)}  \left(  \Gamma(2-\eta) - (2-\eta) (\absc-\lambda)^{-\eta+1} \expfactor x_0^{1-\eta}
    \right)
\end{align*}
Plugging in $\lambda = \absc - \frac{\mcpgamma}{y}$, we wish to bound the
final term inside the parentheses. Recall from the line above
Eq.~\eqref{eqb:ld-gamma-eta-bound} that
$\frac{4}{1-\mcpgamma} \ge \frac{2}{\eta} \ge \frac{2-\eta}{\eta}$ and
$1 \ge \frac{1}{\Gamma(1-\eta)}$. Using these inequalities and the hypothesis,
\begin{align*}
  x_0  \ge  \frac{y}{\mcpgamma} \left( 2 \log x_0 +
  \log \left( \frac{4 \absc\left(  \frac{\mcpgamma}{y} \vee 1
  \right) }{1-\mcpgamma} \right) \right)  
  \ge  \frac{y}{\mcpgamma} \left( (2-\eta) \log x_0 
  + \log \left(   \frac{(2-\eta)  \absc
  \left(  \frac{\mcpgamma}{y} \right)^{-\eta+1}  }{\Gamma(1-\eta) \eta} 
  \right) \right).
\end{align*}
Rewriting the bound, we have
\begin{align}
\frac{(2-\eta) 
\left( 
\frac{\mcpgamma}{y}
\right)
^{-\eta+1} \expfactorGammaEtaTwo x_0^{1-\eta} }{\Gamma(1-\eta)} \le  \frac{1}{\absc x_0}. \label{eqn:Gamma-ratio-bound-x0-term}
\end{align}
Noting $x_0 \ge 2 > 2-\eta$, we arrive at
\begin{align*}
   \E[\rv_2 e^{\lambda \rv_2}]
  \ge \left(\expmgf \right)^{1-\eta} \frac{1}{\absc-\lambda}  \paran{1 - \eta - \frac{1}{\absc x_0} } 
  \ge  \left(\expmgf \right)^{1-\eta} \frac{1}{\absc-\lambda}  \mcpgamma,
\end{align*}
where we used $\frac{\Gamma(2-\eta)}{\Gamma(1-\eta)} = 1- \eta$ and the
bound~\eqref{eqn:Gamma-ratio-bound-x0-term} in the first inequality, and
$\eta = 1- \mcpgamma - \frac{1}{\absc x_0}$ in the second inequality.  We
conclude
\begin{align*}
 \E\left[\rv_2 e^{\GammaLambdaOptUpper \rv_2}\right]
  \ge \left(\frac{\absc y}{\mcpgamma} \right)^{1-\eta} \frac{y}{\mcpgamma}   \mcpgamma
  = y \E\left[e^{\GammaLambdaOptUpper \rv_1}\right]
  \ge y \E\left[e^{\GammaLambdaOptUpper \rv_2}\right],
\end{align*}
where the final inequality follows from the bound~\eqref{eqn:mgf-bound-ld}.

\subsubsection{Proof of Lemma~\ref{lemma:rate-separated}}
\label{section:proof-of-rate-separated}
Plugging  $\E[e^{\lambda \rv_1}] = \left(
\expmgf\right)^{1-\eta}$ into the
expression for $I_1$, we get $\lambda\opt(P_1) = \absc - \frac{1-\eta}{y}$ and
\begin{equation*}
  I_1 = \sup_{\lambda} \left\{ \lambda y - \kappa_1(\lambda) \right\}
  = y \GammaLambdaOptSep - \kappa_1\GammaLambdaOptSep.
\end{equation*}
Letting $\kappa_j(\lambda) = \log \E e^{\lambda R_j}$, we have the
following lower bound on the separation
\begin{align}
  |I_1 - I_2|
  & \ge I_2 - I_1
    = \sup_{\lambda} \left\{ \lambda y - \kappa_2(\lambda) \right\}
    -  \sup_{\lambda} \left\{ \lambda y - \kappa_1(\lambda) \right\}
  \nonumber \\
  & \ge y \GammaLambdaOptSep - \kappa_2\GammaLambdaOptSep
    - \left(y \GammaLambdaOptSep - \kappa_1\GammaLambdaOptSep \right) \nonumber \\
  & = \kappa_1\GammaLambdaOptSep - \kappa_2\GammaLambdaOptSep
    = \log \frac{\E e^{\GammaLambdaOptSep \rv_1}}{\E e^{\GammaLambdaOptSep \rv_2}} \nonumber \\
  & \ge \frac{\E e^{\GammaLambdaOptSep \rv_1} - \E e^{\GammaLambdaOptSep \rv_2}}{\E e^{\GammaLambdaOptSep \rv_1}}
    = 1 -  \left(\frac{\absc y}{1-\eta} \right)^{\eta-1} \E e^{\GammaLambdaOptSep \rv_2},
    \label{eqn:ld-sep-lower-bound}
\end{align}
where we used $\log x \ge 1 - \frac{1}{x}$ in the last inequality.

We proceed by upper bounding $\E e^{\GammaLambdaOptSep \rv_2}$.  Using the
bounds~\eqref{eqn:int-estimate} and~\eqref{eqn:c-estimate}, we have
\begin{align*}
\int_{x_0}^\infty 
C x^{-1} e^{-(\absc - \lambda) x} dx  
&\le  C \frac{1}{\absc - \lambda} e^{-(\absc - \lambda) x_0} \frac{1}{x_0}
\left(1-  \frac{1}{ (\absc - \lambda) x_0} + 2\frac{1}{(\absc - \lambda)^2 x_0^2}\right)  \\
& \le 
\CUpperBound 
 \frac{1}{\absc - \lambda} 
 \expfactor
x_0^{-\eta}
\left(1-  \frac{1}{ (\absc - \lambda) x_0} + 2\frac{1}{(\absc - \lambda)^2 x_0^2}\right) , 
\end{align*}
and
\begin{align*} 
    \int_0^{x_0} 
\GammaConst x^{-\eta } e^{-(\absc - \lambda) x} dx
&= \left( \expmgf \right)^{1-\eta}
-   \int_{x_0}^\infty   \GammaConst x^{-\eta } \expfactor dx \\
& \le  \left( \expmgf \right)^{1-\eta}
-  \GammaConst (\absc - \lambda)^{- 1}
  \expfactor  x_0^{-\eta} \left(1-\frac{\eta}{(\absc - \lambda)x_0}\right).
\end{align*}
Combining these bounds, 
\begin{align*}
  \E[e^{\lambda \rv_2}]
  & \le \left(\expmgf \right)^{1-\eta} +\GammaConst x_0^{-\eta}
    \frac{1}{\absc - \lambda} \expfactor  \\
  & \qquad \qquad \qquad \qquad \qquad \times \Biggl\{
    \left( 1 + \frac{1}{\absc x_0 - 1} \right)
    \left(1-  \frac{1}{ (\absc - \lambda) x_0} + 2 \frac{1}{(\absc - \lambda)^2 x_0^2}\right)  
    - 1 + \frac{\eta}{(\absc - \lambda)x_0} \Biggl\}.
\end{align*}
Plugging   $\lambda = \absc - \frac{1-\eta}{y}$ into the final term in the
preceding display, we use
$x_0 \ge \paran{\frac{2}{\mcpgamma^2} \frac{\absc y + 1}{\absc y - 1} \vee
  \frac{2}{\mcpgamma}} y \ge \paran{\frac{2}{(1-\eta)^2} \frac{\absc y +
    1}{\absc y - 1} \vee \frac{2}{1-\eta}} y$ to get
\begin{align*}
  & \left( 1 + \frac{1}{\absc x_0 - 1} \right)
    \left(1-  \frac{1}{ (\absc - \lambda) x_0} + 2 \frac{1}{(\absc - \lambda)^2 x_0^2}\right)  
    - 1 + \frac{\eta}{(\absc - \lambda)x_0} \\
 & = \frac{1}{\absc x_0 - 1}\left(1-  \frac{1}{ (\absc - \lambda) x_0} + 2 \frac{1}{(\absc - \lambda)^2 x_0^2}\right)   - \frac{1}{(\absc - \lambda) x_0}\left( 1- \eta - \frac{2}{(\absc - \lambda) x_0} \right)\\
& \stackrel{(a)}{\le}
\frac{1}{\absc x_0 - 1} - \frac{1}{x_0/y (1-\eta)}
\paran{1-\eta - \frac{2}{(1-\eta)x_0/y}}\\
& \stackrel{(b)}{\le}
\frac{1}{\absc x_0 - 1} - \frac{1}{x_0/y }
\frac{2}{\absc y + 1} \\
  & \stackrel{(c)}{\le}  
    \frac{2y}{(\absc y + 1)x_0(\absc x_0 - 1)}
    \left(-\frac{\absc y - 1}{4} \frac{x_0}{y} \right)
    =  -\frac{\absc y - 1}{2(\absc y + 1)(\absc x_0 - 1)},
\end{align*}
where in steps $(a)$ and $(b)$ we used
$\frac{1}{ (\absc - \lambda) x_0} - 2 \frac{1}{(\absc - \lambda)^2 x_0^2} \le
0, \frac{2}{(1-\eta)^2x_0/y} \le \frac{\absc y - 1}{\absc y +1}$, and in step
$(c)$ we used $x_0 \ge \frac{4}{\absc y - 1} y$.

Collecting previous bounds, we arrive at 
\begin{align*}
  \E[e^{\GammaLambdaOptSep \rv_1}] -
  \E[e^{\GammaLambdaOptSep \rv_2}] 
  & \ge  \GammaConst x_0^{-\eta}
    \frac{1}{(1-\eta)/y} \expfactorGammaEta \frac{\absc y - 1}{2(\absc y + 1)(\absc x_0 - 1)}.
\end{align*}
Recalling $\Gamma(\alpha) \ge \Gamma(1-\eta)$ and
$x_0 \ge \frac{2}{\mcpgamma} y \ge y$, and plugging the preceding display into the
separation inequality~\eqref{eqn:ld-sep-lower-bound}, we have
\begin{align*}
  |I_1 - I_2|
  &  \ge \frac{1}{\Gamma(\mcpgamma)} 
    \left(\frac{y}{x_0 (1-\eta)}\right)^{\eta}
    \expfactorGammaEta \frac{\absc y - 1}{2(\absc y + 1)(\absc x_0 - 1)} \\
  & \ge
    \frac{1}{\Gamma(\mcpgamma)} 
    \left(\frac{y}{x_0}\right)^{1-\mcpgamma}
    \expfactorGammaEta \frac{\absc y - 1}{2(\absc y + 1)(\absc x_0 - 1)},
\end{align*}
where we used $(1-\eta)^{-\eta} \ge 1$ in the final inequality.

\subsection{Proof of Lemma~\ref{lemma:ld-kl-bound}}
\label{section:proof-ld-kl-bound}

Using the bound~\eqref{eqn:c-estimate}, we have 
\begin{align*}
  \dkl{P_1}{P_2}
  & = \int_{x_0}^\infty  \GammaConst
    x^{-\eta} e^{-\absc x} \log 
    \frac{\GammaConst x^{1-\eta}}{C}  dx   \\  
  & \le \int_{x_0}^\infty  \GammaConst
    x^{-\eta} e^{-\absc x} \log 
    \frac{x^{1-\eta}}{x_0^{1-\eta} \left( 1 -  \frac{\eta}{\absc x_0} \right)}  dx   \\  
  & \le \int_{x_0}^\infty  \GammaConst
    x^{-\eta} e^{-\absc x} \log 
    \frac{x}{x_0 \left( 1 -  \frac{\eta}{\absc x_0} \right)}  dx   \\  
  & = \GammaConst 
    \Paran{-\log\paran{x_0 - \frac{\eta}{\absc}} \int_{x_0}^\infty x^{-\eta} e^{-\absc x} dx
    + \int_{x_0}^\infty x^{-\eta} e^{-\absc x} \log x dx}.
\end{align*}
Integration by parts gives us
\begin{align*}
  \int_{x_0}^\infty x^{-\eta} e^{-\absc x} \log x dx    
    &= \frac{x_0^{-\eta}}{\absc} e^{-\absc x_0} \log x_0 + \frac{1}{\absc} \int_{x_0}^\infty e^{- \absc x} x^{-\eta - 1} (1- \eta \log x) dx \\
    &\le \frac{x_0^{-\eta}}{\absc} e^{-\absc x_0} \log x_0 + \frac{1}{\absc} \int_{x_0}^\infty e^{- \absc x} x^{-\eta - 1}   dx
      \le \frac{x_0^{-\eta}}{\absc} e^{-\absc x_0} \paran{\log x_0 + \frac{1}{\absc x_0}}.
\end{align*}
Using $\log x \le x-1$ and the assumption $x_0 \ge 1 + \frac{1}{\absc} \ge 1+ \frac{\eta}{\absc}$, 
\begin{align*}
  \GammaConstRec     \dkl{P_1}{P_2}
  & \le -\frac{x_0^{-\eta}}{\absc} e^{-\absc x_0} 
    \paran{1 - \frac{\eta}{\absc x_0}} \log \paran{x_0 - \frac{\eta}{\absc}} + \frac{x_0^{-\eta}}{\absc} e^{-\absc x_0} \paran{\log x_0 + \frac{1}{\absc x_0}} \\ 
  &  = \frac{x_0^{-\eta}}{\absc} e^{-\absc x_0} \log \frac{x_0}{x_0 - \eta/\absc} + 
    \frac{x_0^{-\eta}}{\absc^2 x_0} e^{-\absc x_0}  \paran{
    \eta \log\paran{x_0 - \frac{\eta}{\absc}} + 1
    } \\
  &\le  \frac{x_0^{-\eta}}{\absc} e^{-\absc x_0} \frac{\eta/\absc}{x_0 - \eta/\absc} + 
    \frac{x_0^{-\eta}}{\absc^2 x_0} e^{-\absc x_0}  \paran{
    \eta x_0 - \frac{\eta^2}{\absc} -\eta + 1
    } \\
  &\le \frac{x_0^{-\eta}}{\absc} e^{-\absc x_0} \paran{\frac{2\eta}{\absc} + \frac{1}{\absc x_0}}  
  \le  3 \frac{x_0^{\half(\mcpgamma -1)}}{\absc^2} e^{-\absc x_0}
\end{align*}
where in the last step, we used the inequalities~\eqref{eqb:ld-gamma-eta-bound}
and $\eta \in [0,1]$.  Using the inequalities~\eqref{eqb:ld-gamma-eta-bound}
again and $\absc^{-\eta -1} \le 1 \vee \absc^{-2}$, it follows that
\begin{equation*}
  \dkl{P_1}{P_2} \le 3 \frac{\absc^{-\eta - 1}}{\Gamma(1-\eta)}
  x_0^{\half(\mcpgamma-1)} e^{- \absc x_0} \le 3 (1 \vee \absc^{-2})
  x_0^{\half(\mcpgamma-1)} e^{- \absc x_0}.
\end{equation*}
Since $x_0 \ge \frac{\log 3(1 \vee \absc^{-2})}{\half(1 - \mcpgamma)}$, we
have $\dkl{P_1}{P_2} \le e^{- \absc x_0}$.


\subsection{Proof of Theorem~\ref{theorem:lower-md-rate}}
\label{section:proof-of-lower-md-rate}

As in Section~\ref{section:proof-of-lower-ld-rate}, we use Le Cam's method.
Let $\rv_1 \sim P_1 = \mathsf{Exp}(\absc)$ with density
$f_1(x) = \absc e^{-\absc x} \indic{x \ge 0}$ and $\rv_2 \sim P_2$ with density
\begin{equation*}
  f_2(x) = \begin{cases}
    0 & ~~\mbox{if}~~  x < 0 \\
    \absc (1+\omega) e^{-\absc(1+\omega) x} & ~~\mbox{if}~~ 0 \le x \le x_0 \\
    \absc e^{-\absc \omega x_0} e^{-\absc x} & ~~\mbox{if}~~ x > x_0
  \end{cases}
\end{equation*}
for some $x_0, \omega > 0$ to be chosen later. First, we show that both $P_1$
and $P_2$ are in $\mc{P}_{\absc, y, \mcpgamma}$.
\begin{lemma}
  \label{lemma:md-class-inclusion}
  For $\omega \le \frac{1-\mcpgamma}{\absc y} \wedge \frac{2 - \absc y}{\absc y}$, we have
  $P_1, P_2 \in \mc{P}_{\absc,y,\mcpgamma}$.
\end{lemma}
\noindent See Section~\ref{section:proof-of-md-class-inclusion} for the proof.
Next, we prove that $I_1$ and $I_2$ are $\Omega(\omega)$ apart. 
\begin{lemma}
  \label{lemma:md-rate-separated}
  \begin{equation*}
    |I_1 - I_2| = I_2 - I_1
    \ge \frac{\absc y -1}{\omega \absc y +1}
    (1- e^{-\left(\absc \omega + 1/y \right) x_0}) \omega.
  \end{equation*}
\end{lemma}

Despite the rate functions being $\Omega(\omega)$-separated, we next show that
$P_1$ and $P_2$ are $O(\omega^2)$-close in total variation distance. From Pinsker's
inequality, we have
$\tvnorms{P_1^n - P_2^n}^2 \le \half \dkls{P_1^n}{P_2^n} = \frac{n}{2}
\dkl{P_1}{P_2}$. We have
\begin{align*}
  \dkl{P_1}{P_2}
  & = \int_0^{x_0} \absc e^{-\absc x}
    \log \frac{\absc e^{-\absc x}}{\absc (1+\omega) e^{-\absc (1+\omega) x}} dx
    + \int_{x_0}^\infty \absc e^{-\absc x}
    \log \frac{\absc e^{-\absc x}}{\absc e^{-\absc \omega x_0} e^{-\absc x}} dx \\
  & = -\absc \omega x_0 e^{-\absc x_0} + \frac{1}{\omega} (1-e^{-\absc x_0})
    - (1-e^{-\absc x_0}) \log (1+\omega) + \absc \omega x_0 e^{-\absc x_0} \\
  & = (\omega - \log (1+\omega)) (1-e^{-\absc x_0})
    \le \frac{1}{2} \omega^2 (1-e^{-\absc x_0}),
\end{align*} where we used $x - \log(1+x) \le \frac{1}{2} x^2$ for all $x \ge 0$.
Now, let $x_0 = \absc^{-1} \log n$ and $\omega = \frac{c}{\sqrt{n}}$ for
$c = 2(1-2\delta)$. Then, $\dkl{P_1}{P_2} \le \half \omega^2$ so that
$\tvnorms{P_1^n - P_2^n} \le \frac{c}{2}$. Since $x_0 \ge y \log 2$ and
$\omega > 0$, we have
$e^{-(\absc \omega + 1/y)x_0} \le \half$ and
$\frac{1}{\omega \absc y + 1} \ge \half$.  From
Lemma~\ref{lemma:md-rate-separated}, we get
$|I_1 - I_2| \ge \frac{\absc y - 1}{4} \omega$.  Using these two bounds in
Lemma~\ref{lemma:lecam}, we obtain $\mathfrak{M}_n \ge \delta$ as desired.

\subsubsection{Proof of Lemma~\ref{lemma:md-class-inclusion}}
\label{section:proof-of-md-class-inclusion}
Evidently, $P_1 \in \mc{P}_{\absc,y,\mcpgamma}$ since
$\E e^{\lambda \rv_1} = \expmgf, \E[\rv_1] = \frac{1}{\absc}$, and $\lambda\opt(P_1) = \absc - 1/y$.   To show $P_2 \in \mc{P}_{\absc, y, \mcpgamma}$, we first note that $P_2$ clearly satisfies Condition~\ref{item:nonneg-heavy-tail}. Then we check
Condition~\ref{item:mgf-bound}, which follows from
\begin{align}
  \E e^{\lambda \rv_2}
  & = \frac{\absc(1+\omega)}{\absc(1+\omega) - \lambda}
  (1-e^{-(\absc(1+\omega)-\lambda) x_0}) + \frac{\absc}{\absc - \lambda}
  e^{-(\absc(1+\omega)-\lambda)x_0} \nonumber \\
  & = \frac{\absc}{\absc - \lambda}
    \left(1 - \frac{\omega \lambda}{\absc(1+\omega) - \lambda}
    (1-e^{-(\absc(1+\omega) - \lambda)x_0})\right) \le \expmgf,
    \label{eqn:md-p2-mgf}
\end{align}
and
\begin{align*}
   \E[\rv_2] = \int_0^{x_0}  \absc(1+\omega)x e^{-\absc(1+\omega) x} dx  + \int_{x_0}^\infty \absc x e^{-\absc \omega x_0} e^{-\absc x} dx =  \frac{1}{\absc(1+\omega)}(1+\omega e^{-\absc(1+\omega) x_0}) \le \frac{1}{\absc}.
\end{align*}
It remains to show Condition~\ref{item:argmax-bound}. Since
$\omega \le \frac{1 - \mcpgamma}{\absc y}$,
$\absc(1+\omega) - \frac{1}{y} \le \absc - \frac{\mcpgamma}{y} $,  it
suffices to show $\lambda\opt(P_2) \le \absc(1+\omega) - 1 / y$.
Similar to inequality~\eqref{eqn:first-order-optimality}, we need to show
$\E[\rv_2 e^{\lambda \rv_2}] \ge y \E[e^{\lambda \rv_2}]$ when
$\lambda = \absc (1+\omega) - \frac{1}{y}$.  Plugging 
$\lambda = \absc (1+\omega) - \frac{1}{y}$ into the preceding
display~\eqref{eqn:md-p2-mgf}, we obtain
\begin{align*}
  \E[e^{\left( \absc (1+\omega) - \frac{1}{y} \right) \rv_2}]
  = \absc(1+\omega)y\left(1  -e^{-x_0/y}\right)
  + \frac{\absc}{\frac{1}{y} - \omega\absc} e^{-x_0/y}.
\end{align*}
Next, integration by parts gives
\begin{align*}
  \E[\rv_2 e^{\lambda \rv_2}]
  & = \int_0^{x_0} \absc (1+\omega) x e^{-(\absc (1+\omega) - \lambda) x} dx
    + \int_{x_0}^\infty \absc e^{-\absc \omega x_0} x e^{- (\absc - \lambda) x} dx \\
  & = \frac{\absc(1+\omega)}{\absc(1+\omega) - \lambda}
    \left(-x_0 e^{-(\absc(1+\omega) - \lambda) x_0}
    + \frac{1}{\absc(1+\omega) - \lambda}
    (1-  e^{-(\absc(1+\omega) - \lambda) x_0})\right) \\
  & \qquad + \frac{\absc}{\absc - \lambda} e^{-\absc \omega x_0}
    \left(
    x_0 e^{-(\absc - \lambda)x_0} + \frac{1}{\absc - \lambda} e^{-(\absc - \lambda)x_0}
    \right).
\end{align*}
Plugging in $\lambda = \absc (1+\omega) - \frac{1}{y}$ again, we have
\begin{align*} 
  \E[\rv_2e^{\left(\absc (1+\omega) - \frac{1}{y} \right)\rv_2}] 
  & =  \absc(1+\omega)y 
    \left(-x_0 e^{-x_0/y}
    + y
    (1-  e^{-x_0/y})\right) \\
  & \qquad + \frac{\absc}{\frac{1}{y} - \omega\absc} e^{-\absc \omega x_0}
    \left(
    x_0 e^{-(\frac{1}{y} - \omega\absc) x_0} + \frac{1}{\frac{1}{y} - \omega\absc} e^{-(\frac{1}{y} - \omega\absc) x_0}
    \right).
\end{align*}
Taking differences, we get 
\begin{align*}
  & \E[\rv_2e^{\left(\absc (1+\omega) - \frac{1}{y} \right)\rv_2}]
  - y \E[e^{\left(\absc (1+\omega) - \frac{1}{y} \right)\rv_2}] \\
  &= \frac{e^{-x_0/y} y \omega \absc}{(1-\absc y \omega)^2}\left[
  x_0 \Big(-1 - y^2\omega(1+\omega) \absc^2 + y\absc (1+2\omega) \Big) + \absc y^2
  \right]
  \\ 
  & \stackrel{(a)}{\ge} \frac{e^{-x_0/y} y \omega \absc}{(1-\absc y \omega)^2}\left[
  x_0 \Big(-1 - 2\omega \absc y + y\absc (1+2\omega) \Big) + \absc y^2  
  \right] \\
   & \stackrel{(b)}{\ge} \frac{e^{-x_0/y} y \omega \absc}{(1-\absc y \omega)^2}   \absc y^2   \ge 0
\end{align*}
where we used $
1 + \omega \le \frac{2}{\absc y}
$ in step $(a)$ and $\absc y \ge 1$ in step $(b)$. We conclude that
$P_2$ satisfies Condition~\ref{item:argmax-bound}.

\subsubsection{Proof of Lemma~\ref{lemma:md-rate-separated}}
\label{section:proof-of-md-rate-separated}
Plugging  $\E[e^{\lambda \rv_1}] = \frac{\absc}{\absc - \lambda}$ into the
expression for $I_1$, we get $\lambda\opt(P_1) = \absc - \frac{1}{y}$ and
\begin{equation*}
  I_1 = \sup_{\lambda} \left\{ \lambda y - \kappa_1(\lambda) \right\}
  = y (\absc - 1/y) - \kappa_1(\absc - 1 / y).
\end{equation*}
Then, we have the following lower bound on the separation:
\begin{align}
  |I_1 - I_2|
  & \ge I_2 - I_1
    = \sup_{\lambda} \left\{ \lambda y - \kappa_2(\lambda) \right\}
    -  \sup_{\lambda} \left\{ \lambda y - \kappa_1(\lambda) \right\}
  \nonumber \\
  & \ge y (\absc - 1/y) - \kappa_2(\absc - 1 / y)
    - (y (\absc - 1/y) - \kappa_1(\absc - 1 / y)) \nonumber \\
  & = \kappa_1(\absc - 1 / y) - \kappa_2(\absc - 1 /y)
    = \log \frac{\E e^{(\absc - 1/y) \rv_1}}{\E e^{(\absc - 1/y) \rv_2}} \nonumber \\
  & \ge \frac{\E e^{(\absc - 1/y) \rv_1} - \E e^{(\absc - 1/y) \rv_2}}{\E e^{(\absc - 1/y) \rv_1}}
    = \frac{1}{\absc y}
    \left( \E [e^{(\absc - 1/y) \rv_1}] - \E e^{(\absc - 1/y) \rv_2} \right)
    \label{eqn:md-sep-lower-bound}
\end{align}
where we have used $\log x \ge 1 - \frac{1}{x}$ in the last inequality.  Using
previous calculations~\eqref{eqn:md-p2-mgf},
\begin{equation*}
  \E[e^{\lambda \rv_1}] - \E[e^{\lambda \rv_2}]
  = \frac{\absc \omega \lambda}{(\absc - \lambda) (\absc (1+\omega) - \lambda)}
  (1-e^{-(\absc(1+\omega) - \lambda)x_0}).
\end{equation*}
Plugging in $\lambda = \absc - 1/y$, we obtain
\begin{equation*}
  \E [e^{(\absc - 1/y) \rv_1}] - \E[e^{(\absc - 1/y) \rv_2}]
  = \absc y \frac{\absc y - 1}{\omega \absc y +1}
  (1-e^{-(\absc \omega + 1/y)x_0}) \omega.
\end{equation*}
Plugging this in~\eqref{eqn:md-sep-lower-bound}, we get the desired result.

\highlight{
\subsection{Proof of Corollary~\ref{cor:gaussian-lb}}
\label{section:proof-gaussian-lb}

Let $P_1 = \mathsf{N}(y - \bar{\lambda},1)$ and $P_2 = \mathsf{N}(y - \bar{\lambda} + \frac{c}{\sqrt{n}}, 1 )$ with $c = 2(1-2\delta)$.
From Example~\ref{example:Gaussian},
we have
\begin{align*}
\opttilt_y(P_1) & = \bar{\lambda} \\
\opttilt_y(P_2) & = \bar{\lambda}- \frac{c}{\sqrt{n}}  <  \bar{\lambda},
\end{align*} so that 
$P_1,P_2 \in \mc{Q}_{\bar{\lambda}, y}$.  \\
It is straightforward to verify that 
\begin{align*}
    \dkl{P_1}{P_2}  = \frac{c^2}{2n},
\end{align*}
and 
\begin{align*}
    I_1 & = \frac{1}{2} (\bar{\lambda})^2, \\ 
    I_2 &= \frac{1}{2} \paran{\bar{\lambda}- \frac{c}{\sqrt{n}} }^2, \\ 
    |I_1 - I_2| &=  
    \frac{c}{\sqrt{n}} \bar{\lambda}- \half \frac{c^2}{n} 
    \ge \frac{c}{2\sqrt{n}} \bar{\lambda},
\end{align*} where the last inequality follows from the assumption $n \ge \paran{\frac{c}{\bar{\lambda}}}^2$ so that $\frac{c^2}{n}  \le    \frac{c}{\sqrt{n}} \bar{\lambda}$.
 From Pinsker's inequality, we have
$\tvnorms{P_1^n - P_2^n}^2 \le \half \dkls{P_1^n}{P_2^n} = \frac{n}{2}
\dkl{P_1}{P_2} = \frac{c^2}{4} = (1-2\delta)^2$. Applying
Lemma~\ref{lemma:lecam}, we obtain $\mathfrak{M}_n \ge \delta$  as claimed.

}


 \section{Experimental details}

\subsection{Details of the DQN model for queueing} \label{section:detail-DQN}

The DQN model is implemented in OpenAI Gym~\cite{BrockmanChPeScScTaZa16}, a toolkit for benchmarking reinforcement learning algorithms, 
to build our discrete-event simulator. 
Using our custom  simulator,
the queueing process can be modeled as a Markov Decision Process $\mathcal{M}$  as follows.
\begin{itemize} 
\setlength{\itemsep}{1pt}
\setlength{\parskip}{0pt}
\setlength{\parsep}{0pt} 
\item Time steps $t$ are  time points when the scheduler needs to make decisions.
\item  The state $s_t$ is  history information up to step $t$. As we allow non-exponential arrival distributions, we need all history information as the 
state. Even though  one can use a recurrent neural network~\cite{HausknechtSt15} to parametrize the $Q$ function and use $s_t$ directly,
here we opt  for a simpler approach similar to~\cite{MnihKaSiGrAnWiRi13}: use a preprocessing function $\Phi(\cdot)$ to compress state 
information  into a low-dimensional vector $\phi_t$ and update $Q$ through $\phi_t$ instead of  $s_t$. 
Based on domain knowledge of queueing systems, we choose $\phi_t$ as 
a vector concatenating queue length and age of the oldest job for each type of job at step $t$. Thus, $\phi_t \in \mathbb{N}^3 \times \R_+^3$.
Using $x_t$ to denote information between step $t-1$ and step $t$, we have $s_t=s_{t-1},x_t$. This is consistent with notations in~\cite{MnihKaSiGrAnWiRi13}.
\item  The action $a_t \in \set{1,2,3}$ represents the queue that the server chooses to serve.  
\item  We assume initially there is one job in each queue, so the system needs to make a decision initially.
\item  The transition probability is specified by the simulator.
\item  The reward function is $r_a(\phi_t,\phi_{t+1}) =  - (1,3,6)^\intercal \cdot \text{age}(\phi_{t+1})
 - 1,000,000 \cdot \\ \mathbb{I}(\text{if the queue for type $a$ is empty at step } t)$   
where $\text{age}(\phi_{t+1})$ is the age part of vector $\phi_{t+1}$.   
The large cost associated with the selection of empty queues is introduced to dissuade the model from making such suboptimal decisions.
We  choose this reward function based on experimental results and inspired by~\cite{LiuXiMo19}.
\item To deal with the unboundedness of the state space, we implement   policy $\pi_0$: FIFO when one of the queue lengths is larger 
than the threshold $G = 10$. In the algorithm, we use $g(\phi_t)$ to represent the largest queue length at step $t$.
\item We apply a pure exploration strategy and do not update the $Q$ function in the first half of episodes; this is observed to help 
stabilize training in our experiment.
\item We also use experience replay~\cite{Lin92}, a buffer with fixed size storing most recent transitions and rewards information, 
which is a central tool in deep reinforcement learning to improve sample efficiency.
\end{itemize}

To approximate the action-value function $Q$, we use Pytorch~\cite{PaszkeGrMaLeJa19} to 
construct a fully connected feedforward neural network with two hidden layers 
with ReLU activation functions.
The numbers of units for each layer are $6,4,4,3$, respectively. Here, we use $\theta$ to denote the weights and biases of the neural network.
The parameters $\theta$ are initialized by   standard Kaiming initialization~\cite{HeZhReSu15}.
More details of our method are given in Algorithm~\ref{algorithm:DQN}, and   hyperparameters are summarized in Table~\ref{table:hyperparameters-DQN}. 
 
\begin{algorithm} 
\caption{Deep Q-learning with Experience Replay for Queueing Systems Control} \label{algorithm:DQN}
\begin{algorithmic}[1] 
\State Initialize replay memory $\mathcal{D}$ to an empty double-ended queue with size $N$ 
\State  Initialize the action-value function $Q$ 
\For{episode = 1 to $M$}
\State Initialize $s_1=x_1$, and $\phi_1 = \Phi(s_1)$ 
\For{$t$ = 1 to $T$}
\State Sample $U \sim \text{Uniform}[0,1]$ 
\If{$U \le \epsilon$ or episode $ \le M/2$}
\State Randomly select $a_t$
\ElsIf{$g(\phi_t) \geq G$} 
\State Select $a_t$ based on   $\pi_0$ 
\Else 
\State Select $a_t = \argmax_a Q(\phi_t, a; \theta)$
\EndIf
\State Execute action $a_t$, observe $x_{t+1}$
\State Update $s_{t+1} = s_t,  x_{t+1}$ and $\phi_{t+1} = \Phi(s_{t+1})$
\State Compute $r_t = r_{a_t}(\phi_t, \phi_{t+1})$
\State Append $(\phi_t,a_t,r_t,\phi_{t+1})$ to $\mathcal{D}$
\If{episode $ > M/2$} 
\State Sample a batch $\mathcal{B}$ with size $B$ from $\mathcal{D}$
\State For every $j \in \mathcal{B}$, set $q_j = r_j + \gamma \max_a Q(\phi_{j+1},a;\theta)$
\State Use the   optimizer on $\sum_{j \in \mathcal{B}}\ell(q_j; Q(\phi_j, a_j; \theta))$ to update $\theta$
\EndIf
\EndFor
\EndFor 
\end{algorithmic}
\end{algorithm}

\begin{table}[H]
    \centering
\begin{tabular}{ll} 
\hline
Hyperparameter                     & Value                                       \\ \hline
Number of episodes  $M$            & 40                                          \\
Number of steps per episode $T$             & 30,000                       \\
Replay memory capacity $N$    & 1,000                                       \\
Replay memory sample size $B$ & 100                                         \\
Exploration probability $\epsilon$ & 0.05 \\
Discounting factor $\gamma$   & 0.9                                         \\
Loss function $\ell(\cdot;\cdot)$          & Huber loss with parameter $\delta = 1$      \\
Optimizer                     &    \begin{tabular}{@{}l@{}} Adam optimizer~\cite{KingmaBa15} with $(\beta_1,\beta_2) = (0.9,0.999)$,
\\ learning rate $0.01$, $\epsilon = 10^{-8}$, and weight decay $0.001$  \end{tabular}
\\ 
\hline
\end{tabular} \captionof{table}{Hyperparameters for DQN}  \label{table:hyperparameters-DQN}
\end{table}

\subsection{Details of the Health utilization prediction models}
\label{section:nhis-table}

 \begin{table}[H]
\centering
\begin{tabular}{ll}
\hline
Variable                   & Meaning                                                 \\ \hline
sex\_female                & If the individual is female or not                      \\  
educ\_yrs                  & Years of education                                      \\  
health\_status\_excellent  & If the health status is excellent                       \\  
wrkhrs                     & Working hours per week                                  \\  
overnight\_hospital\_times & Times of visiting hospital overnight                    \\ 
care\_athome\_2wks\_times  & Number of home visits by health professionals in 2 weeks \\  
looking\_for\_work         & If the individual is looking for work                   \\  
 not\_work\_health              &            Main reason for not working last week: due to health         \\  
 care\_10more\_12mo              & Received care 10+ times in the last 12 months                      \\   
 wrk\_mo\_lastyr & Months worked last year      \\
 limited\_any\_lt5           & Any limitation - all persons, all conditions                  \\   
 overnight\_hospital\_nights & Nights of visiting hospital overnight  \\ 
 hikind\_nocov &  No coverage of any type  \\ 
 care\_spent\_zero & Amount family spent for medical care     \\ 
hino\_months & Months without coverage in the past 12 months       \\ 
 health\_status\_fair  & Reported health status as fair                      \\ \hline
\end{tabular}
\captionof{table}{Core variables}  \label{table:NHIS-core-variable-new}
\end{table}

Since typical ML training algorithms output classification probabilities
(estimates of $\P(Y = 1 | X)$) that are poorly calibrated~\cite{HastieTiFr09},
we explicitly adjust these estimators using a constant finetuned on a small
held-out data, a standard post-processing process known as ``model
calibration''~\cite{Niculescu-MizilCa05}.  In Algorithm~\ref{algorithm:nhis},
we construct the partition $S_1, \ldots, S_4$ by randomly splitting the entire
2015 data into 80\%, 10\%, 5\%, 5\% subsets respectively. For simplicity, we
always use a random forest classifier for the auxiliary model class $\mc{H}$.
Since all classifiers are trained on $S_1$, we use $S_2 \cup S_3 \cup S_4$ to
evaluate the average accuracy of each model.
\begin{algorithm} 
\caption{Estimate stability for prediction models} \label{algorithm:nhis}
\begin{algorithmic}[1] 
  \State \textsc{Input:} Partition of samples $S_1,S_2,S_3,S_4$, prediction
  model class $\mathcal{F}$, auxiliary class $\mathcal{H}$
  \State Fit a classifier $f(X) \in \mathcal{F}$ to predict $Y$ using data
  $S_1$.
  \State Calibrate $f$ using data $S_2$.
  \State Fit 
  $h(Z) \in \mathcal{H}$ to predict $\ell(f(X);Y)$ based on  core variables $Z$ using data $S_3$.
  \State Compute the plug-in stability estimator $\what{I}_n$ of $h(Z)$ over
  data in $S_4$.
\end{algorithmic}
\end{algorithm}

\begin{table}[H]
\centering
\begin{tabular}{ll}
\hline
Variable                   & Meaning                                                 \\ \hline
sex\_female                & If the individual is female or not                      \\  
medicaid\_                 & If uses Medicaid or not                                 \\  
educ\_yrs                  & Years of education                                      \\  
health\_status\_excellent  & If the health status is excellent                       \\  
wrkhrs                     & Working hours per week                                  \\  
overnight\_hospital\_times & Times of visiting hospital overnight                    \\ 
care\_athome\_2wks\_times  & Number of home visits by health professionals in 2 weeks \\  
looking\_for\_work         & If the individual is looking for work                   \\  
not\_work\_retired         & If the individual retired                               \\ \hline
\end{tabular}
\captionof{table}{New core variables}  \label{table:NHIS-core-variable}
\end{table}

\highlight{
To compare the stability measure with other risk measures, we apply CVaR on the cross-entropy loss of three models.
In Figure~\ref{fig:nhis-cvar-plot}, CVaR of the cross-entropy loss 
predicts that LightGBM is much more stable than Logistic Regression, 
which is quite the opposite of what we see in Section~\ref{section:nhis}.
Similarly, in Figure~\ref{fig:nhis-worst-case-expectation-plot}, 
we estimate the worst-case expecation using~\eqref{eqn:worst-case-exp-kl-duality} and we observe 
that this metric applying on the cross-entropy loss does not predict the stability ranking well.
In addition, in Figures~\ref{fig:nhis-cvar-cond-risk-plot} and~\ref{fig:nhis-worst-case-expectation-cond-risk-plot}, we estimate
the worst-case expectation and CVaR for the conditional risk~\eqref{eqn:cond-risk}.
We observe that when applying on the conditional risk, traditional risk measures also 
successfully predict the robustness of each model.
This again underscores the significance of of choosing a proper performance measure $\rv$.

\begin{figure}
\centering
\begin{minipage}{.44\textwidth}
  \centering 
\includegraphics[height=0.18\textheight]{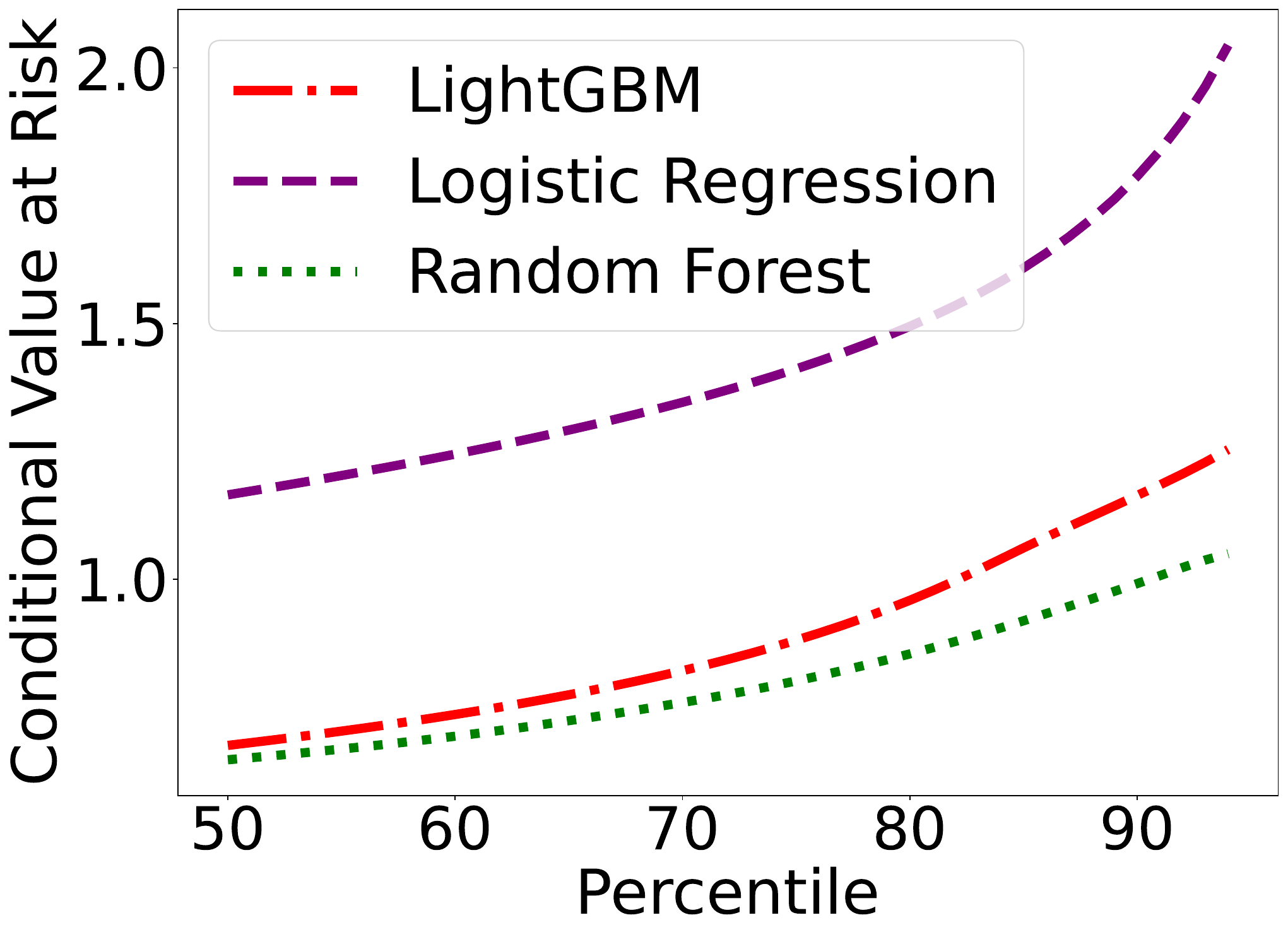}
\caption{CVaR with quantile level $\alpha$ for the cross-entropy loss for different models}  
\label{fig:nhis-cvar-plot}
\end{minipage}%
\begin{minipage}{.44\textwidth}
  \centering 
\includegraphics[height=0.18\textheight]{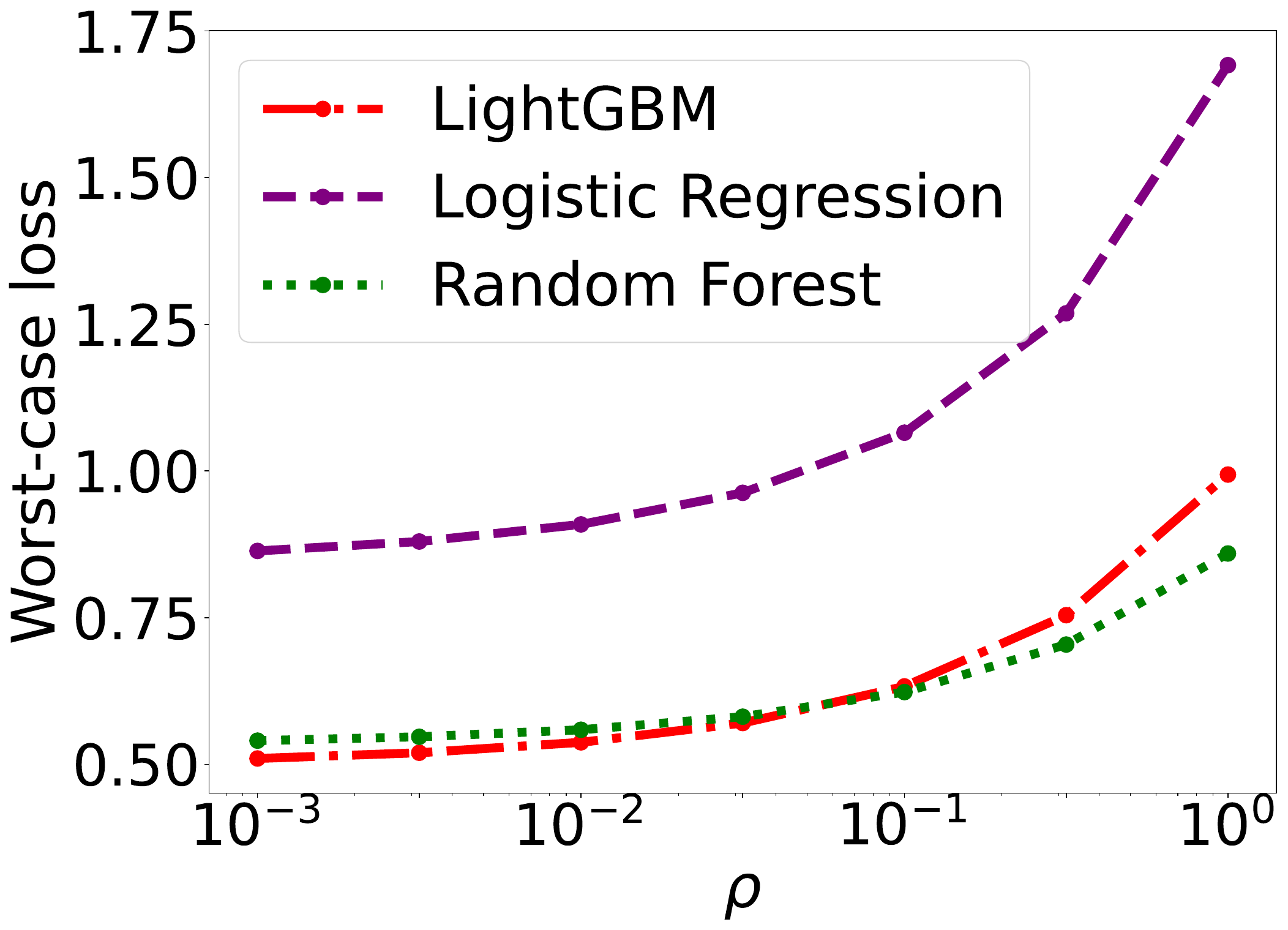}
\caption{Worst-case expectation over a KL neighborhood of radius $\rho$ for the cross-entropy loss for different models}  
\label{fig:nhis-worst-case-expectation-plot}
\end{minipage}%
\end{figure}

\begin{figure}
\centering
\begin{minipage}{.44\textwidth}
  \centering 
\includegraphics[height=0.18\textheight]{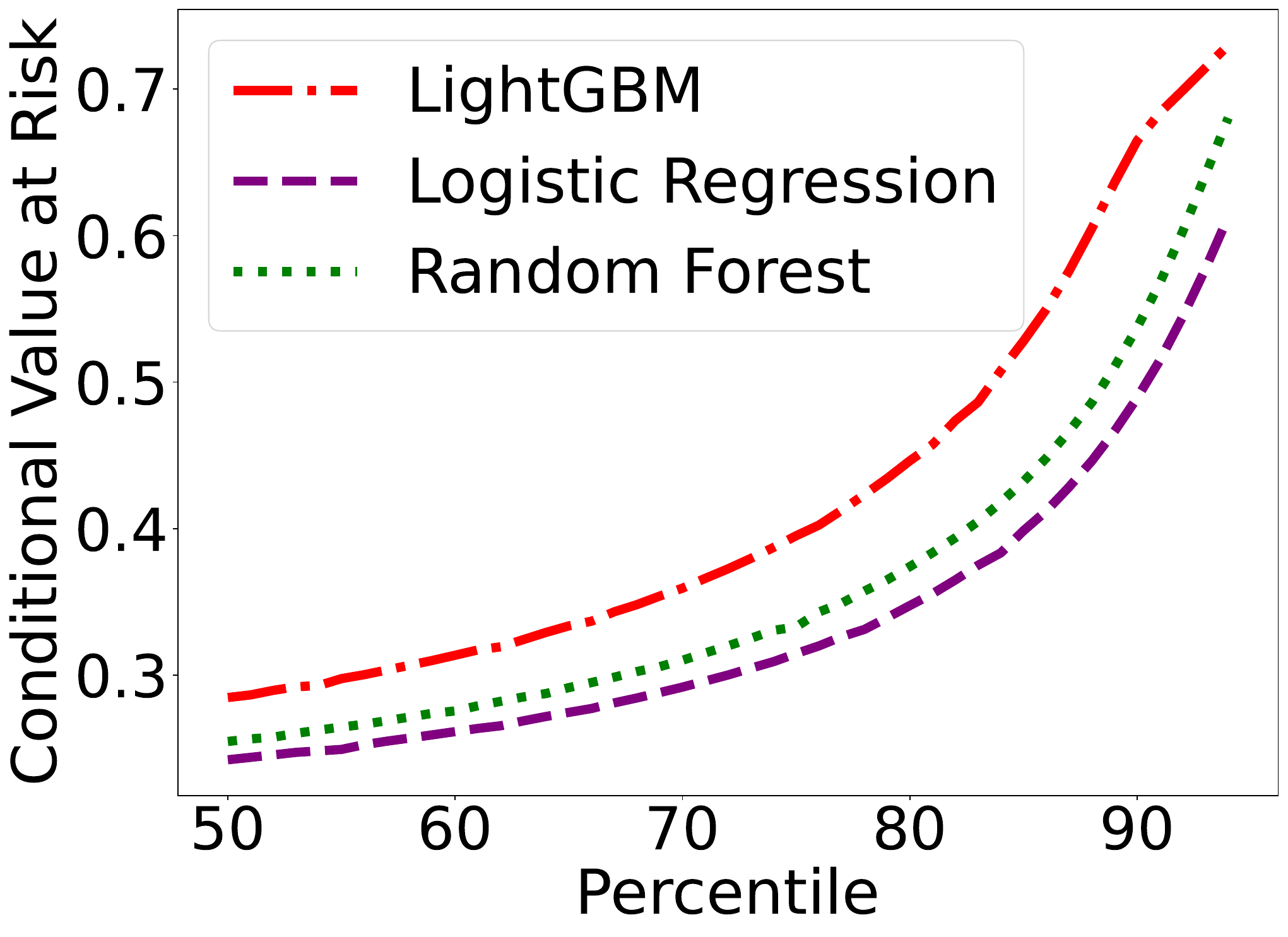}
\caption{CVaR with quantile level $\alpha$ for the conditional loss~\eqref{eqn:cond-risk} for different models}  
\label{fig:nhis-cvar-cond-risk-plot}
\end{minipage}%
\begin{minipage}{.44\textwidth}
  \centering 
\includegraphics[height=0.18\textheight]{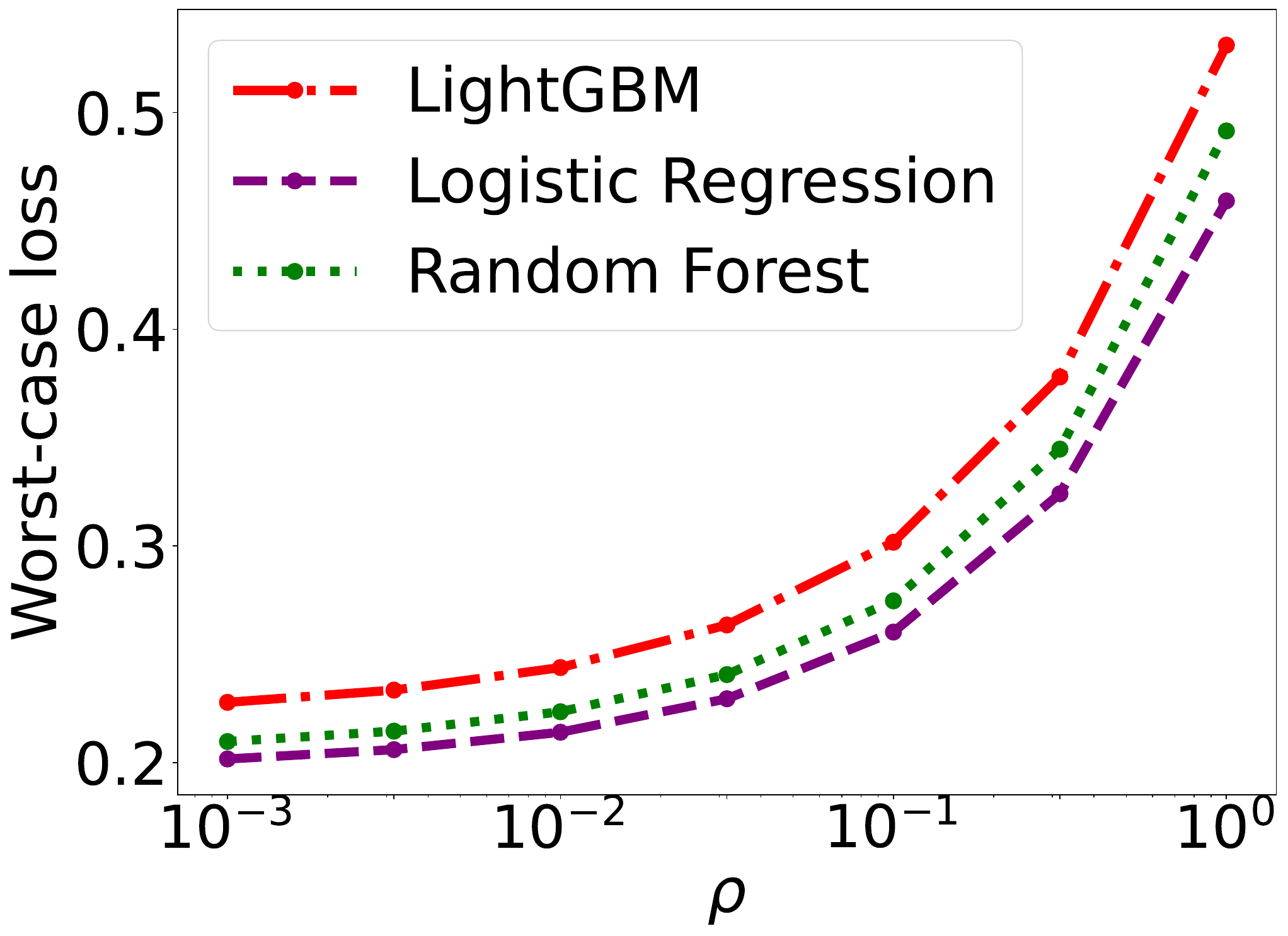}
\caption{Worst-case expectation over a KL neighborhood of radius $\rho$ for the conditional loss~\eqref{eqn:cond-risk} for different models}  
\label{fig:nhis-worst-case-expectation-cond-risk-plot}
\end{minipage}%
\end{figure}
}
 

\end{document}